\documentclass{article}

\usepackage[main, final]{neurips_2026}

\usepackage[utf8]{inputenc}
\usepackage[T1]{fontenc}
\usepackage{hyperref}
\usepackage{url}
\usepackage{booktabs}
\usepackage{amsfonts}
\usepackage{amsmath}
\usepackage{amssymb}
\usepackage{amsthm}
\usepackage{mathtools}
\usepackage{nicefrac}
\usepackage{microtype}
\usepackage{xcolor}
\usepackage{graphicx}
\usepackage{listings}
\usepackage{caption}
\usepackage{enumitem}
\usepackage{longtable}
\usepackage{array}
\usepackage{calc}
\usepackage{etoolbox}
\usepackage{tikz-cd}
\usepackage{fancyvrb}
\usepackage{textcomp}

\graphicspath{{figures/}{../figures/}}
\setlist[itemize]{topsep=3pt,itemsep=1pt,parsep=0pt}
\setlist[enumerate]{topsep=3pt,itemsep=1pt,parsep=0pt}

\theoremstyle{definition}
\newtheorem{definition}{Definition}[section]
\newtheorem{proposition}[definition]{Proposition}
\newtheorem{theorem}[definition]{Theorem}

\newtheorem{corollary}[definition]{Corollary}
\newtheorem{remark}[definition]{Remark}
\newtheorem{example}[definition]{Example}

\providecommand{\tightlist}{\setlength{\itemsep}{0pt}\setlength{\parskip}{0pt}}
\providecommand{\passthrough}[1]{#1}
\providecommand{\pandocbounded}[1]{#1}

\definecolor{pandocCommentColor}{rgb}{0.38,0.63,0.69}
\definecolor{pandocStringColor}{rgb}{0.25,0.44,0.63}
\definecolor{pandocKeywordColor}{rgb}{0.00,0.44,0.13}

\providecommand{\NormalTok}[1]{#1}

\DefineVerbatimEnvironment{Highlighting}{Verbatim}{commandchars=\\\{\}}
\newenvironment{Shaded}{}{}

\newcommand{\enc}{\operatorname{enc}}
\newcommand{\Reach}{\operatorname{Reach}}

\newcommand{\Spell}{\textsc{Spell}}
\newcommand{\eval}{\operatorname{eval}}
\newcommand{\lm}{\operatorname{lm}}

\lstdefinelanguage{spell}{
  morekeywords={quine,do,eval,fn,defn,let,if,when,prune,rethink,!llm-self,
    !call-now,!print,!extend,!peek,!peek-now,persist,future,await,await-all,
    pmap,completion-promise,send-await,spawn,!ask},
  sensitive=true,
  morecomment=[l]{;},
  morestring=[b]",
}

\title{Self-Programmed Execution for Language Model Agents}

\author{%
  Luke J.\ O'Connor\\
  Harvard Medical School\\
  Boston, MA 02115\\
  \texttt{loconnor@hms.harvard.edu}%
}

\begin{document}

\maketitle

\begin{abstract}
	At the heart of a language model agent is an orchestrator program which is responsible for the state transition between consecutive turns. In existing systems, this program is implemented in the harness. This paper introduces \emph{self-programmed execution} (SPE), in which the model completion is itself the orchestrator program, and the harness evaluates this program but does not impose any fixed orchestration policy. Formally, an SPE state of an agentic machine is one from which a model completion can load any successor state up to an embedding, and an instance of an SPE state is one which executes the model completion as a program. In practice, the program requires a language with unusual features, as its source code doubles as the content of the model's context window. I therefore introduce \textsc{Spell}, a Lisp dialect in which programs can edit and re-evaluate themselves, and effectful expressions like model invocations are structured such that re-evaluating an edited program does not replay its side effects. Experiments with existing models, not trained for SPE or \textsc{Spell}, show that frontier models can learn \textsc{Spell} in-context and accomplish challenging agentic tasks. These results demonstrate how an LM can act as an agent without any fixed orchestration policy, and they raise the question of what self-orchestration strategies might be learned by a model trained for self-programmed execution. Code is available at \url{https://github.com/lukejoconnor/spell}.

\end{abstract}

\section{Introduction}
\label{sec:intro}

The harness of a language model (LM) agent acts both as an execution layer, allowing the LM to interact with an environment, and as an orchestrator program. This orchestrator records previous interactions and provides context to the model each turn; more generally, it maintains program state across turns, giving the model the appearance of statefulness. As models have become more capable, the harness and all aspects of orchestration policy have become a major target of engineering effort, particularly the policy for what context is provided to the LM \citep{horthy2025twelvefactoragents}. (In this paper, I use \emph{orchestration} to describe multi-turn, not necessarily multi-agent, logic or policy.) Practitioner accounts emphasize the significance of orchestration policy \citep{anthropic2025contextengineering,openai2026harnessengineering}, and commercial frameworks have been developed for bespoke agentic workflows \citep{langchain2026harnessanatomy,hong2024metagpt}. 

All of this human effort naturally suggests automation. One recent trend is to engineer self-evolving harnesses, for example by adding an outer reflection loop; another is to enable partial self-orchestration, for example by providing tools for subagent delegation (see Section~\ref{sec:related}). This paper asks whether orchestration policy must belong to the harness at all. What if the harness were only an execution layer, and not also an orchestrator?

In \emph{self-programmed execution} (SPE), the harness evaluates a model-written program which is solely responsible for orchestration. This program might write a prompt, append to it the result of a tool call, and then make a recursive self-call. The self-call repeats the process: it inputs a prefix, invokes the LM to complete it, and evaluates the completion:
\begin{equation}
	\texttt{self-call}(\text{prefix}) \;\coloneqq\; \eval(\text{complete}(\text{prefix})).
	\label{eq:self-call}
\end{equation}

The first turn is produced by evaluating an initial program with a prompt and a self-call. Subsequently, every turn-to-turn transition is specified by a model-written self-call expression, and context passes through the turn boundary as the argument to this expression. The harness still exists as an execution layer, but the program it executes is model-written.

Existing agent architectures alternate between a step which is controlled by the model and one which is controlled by the fixed orchestration policy of the harness (Figure~\ref{fig:architectures}). The archetypal example is the ReAct loop \citep{yao2023react}. The closest existing system is the Recursive Language Model (RLM) \citep{zhang2026rlm}, in which arbitrary model-written code is evaluated within a REPL (read--eval--print loop), and this code is able to spawn subagents recursively. However, after the model-written program runs, the harness issues the next main-agent turn, while maintaining the conversation history and REPL state; the transition between main-agent turns is subject to a fixed orchestration policy. SPE shares with RLMs that a model-written program is evaluated by an external runtime. The difference is that this program is fully responsible for orchestration policy, not only for subagent delegation, and does not run inside a fixed agent loop.

\begin{figure}[tbp]
	\centering
	\includegraphics[width=0.9\linewidth]{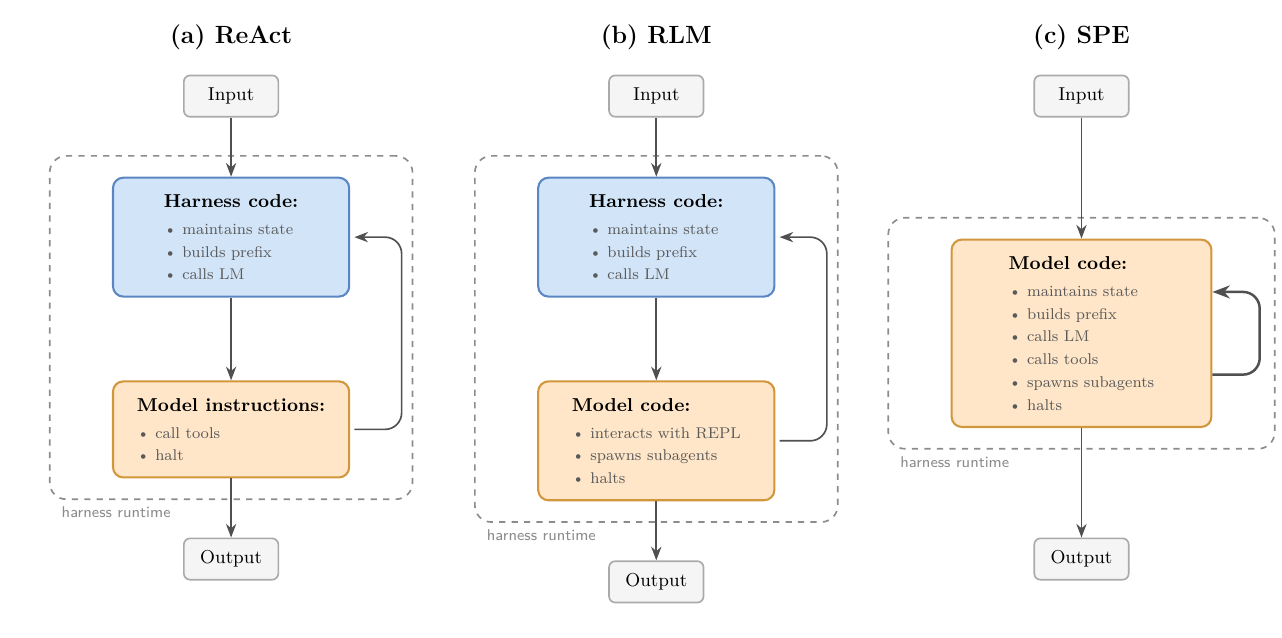}
	\caption{\textbf{Three agent architectures.} Colored boxes distinguish program logic that is implemented in the harness (blue) vs.\ written by the model (yellow). In all cases, these programs are executed by a harness runtime which is external to the model. (a) In ReAct \citep{yao2023react}, the model selects from a prescribed action space. An orchestrator program runs the agent loop, maintaining state (e.g., conversation history) across turns. (b) Recursive language models (RLMs) \citep{zhang2026rlm} execute model-written orchestration code inside a harness-managed main-agent loop, which performs handoff between the main agent's own turns. (c) In SPE, a model-written program specifies the orchestration policy, free of any outer agent loop. The model-written program is still executed by an external runtime.}
	\label{fig:architectures}
\end{figure}

Below, I introduce formal abstractions for orchestration policy and SPE, defining an {\it SPE state} of an agentic machine and proving that the proposed architecture suffices. I propose Self-Programmed Execution Language for LMs (\Spell{}), a language for SPE which locates and addresses the challenges that arise when implementing SPE in practice. Finally, I study how existing models, without special training, behave and perform when challenged to solve hard agentic tasks using \Spell{}.

\section{Formalizing self-programmed execution}
\label{sec:theory}

What is meant by ``fixed orchestration policy,'' what does it mean to ablate it, and why does doing so lead to the specific procedure---``evaluate the model completion as a program''---which is termed SPE? This section gives formal definitions. {\it Orchestration policy} is the transition function exposed by an agentic machine at a particular machine state, mapping a model completion to a successor state. When the image of this function is restricted, the orchestration policy is {\it fixed}. By contrast, an \emph{SPE state} is one in which the model freely chooses its successor state, up to an embedding. Finally, Theorem~\ref{thm:spe} shows that the procedure of Equation~\ref{eq:self-call} does indeed generate an SPE state. Proofs and additional results are provided in Appendix~\ref{app:theory}.

\begin{definition}
	\label{def:agentic-machine}
	Fix a prompt space $P$ and completion space $C$. An \emph{agentic machine} over $(P,C)$ is a tuple $X=(S,p,h)$, where $S$ is a state space, $p\colon S\to P$ is a prompt function mapping states to prompts, and $h\colon S\times C\to S\cup\{\mathbf{1},\uparrow\}$ is a harness function mapping a state/completion pair to the next state, where $\mathbf{1}$ is a halting state, and $\uparrow$ indicates that the  deterministic computation diverges, never halting or reaching a model call. The machine is paired with a language model which maps prompts in $P$ to completions in $C$, possibly non-deterministically. For a state $x$ of $X$, the map $c\mapsto h(x,c)$ is the {\it orchestration policy} exposed at that state: after the model emits completion $c$, this policy produces the next state. \end{definition}

Roughly, the orchestration policy is fixed at a state $x$ of $X$ when there exist inaccessible successor states, not reachable from $x$ by any model completion. To formalize this, we define an {\it embedding} of one machine in another. An embedding is an injection which preserves model-observable prompts, such that the model cannot distinguish a state from its embedding (formally, see Theorem~\ref{thm:embedding-invisible}).  

\begin{definition}
	\label{def:embedding}
	An \emph{embedding} of an agentic machine $X'=(S',p',h')$ into $X=(S,p,h)$ is an injective map $e\colon S'\to S$, extended by $e(\mathbf{1})=\mathbf{1}$ and $e(\uparrow)=\uparrow$, such that $p(e(x'))=p'(x')$ and
	\[
		h(e(x'),c)=e(h'(x',c))
	\]
	for every $x'\in S'$ and completion $c\in C$.
\end{definition}

Intuitively, an SPE state is one where, by choosing the completion, the model can select the next state. Formally, it is defined by a self-embedding which is completion-generated:
\begin{definition}
	\label{def:spe-state}
	For a state $x$ of an agentic machine $X$, $(X,x)$ \emph{completion-generates} $X'$ if there exists an embedding $e$ of $X'$ into $X$ such that for all $y\in \operatorname{im}(e)$, some completion $c$ generates the transition from $x$ to $y$: $h(x,c)=y$. An \emph{SPE state} of $X$ is a state $x_0$ such that $(X,x_0)$ completion-generates $X$.
\end{definition}

Completion-generation is defined for a possibly different target machine $X'$, but an SPE state applies it to the case $X'=X$: the machine completion-generates itself. Equivalently, imagine an equivalent copy of the agentic machine embedded inside $X$. Starting from $x_0$, for every state $y$ in that copy, there is some completion that makes $y$ the next state. Conversely, at a non-SPE state, the model selects among a fixed menu of successor states.

It remains to show that this definition is satisfied by the architecture this paper terms ``SPE.'' This architecture can be realized in an agentic machine which is able to evaluate programs. We use the call-by-value language $\mathcal{L}$ specified in Appendix~\ref{app:theory:agentic-evaluators}, which provides quotation, \texttt{eval}, and a model-call primitive $\lm$. Its evaluation rules give the following agentic machine:

\begin{definition}\label{def:agentic-evaluator}
	The \emph{agentic evaluator} is the agentic machine whose states are programs in $\mathcal{L}$ whose next execution step is to invoke the model (\emph{boundary states}). We express a boundary-state program as
	\[
		\texttt{let }c=\lm\,q\texttt{ in }K,
	\]
	using syntax which distinguishes the next step of execution, $\lm\,q$, from the continuation, $K\in\mathcal{L}$. The prompt function of this machine returns the prompt, $q\in P$. The harness function binds $c$ to the completion and evaluates $K$, producing the next boundary state reached, or $\mathbf{1}$ or $\uparrow$ if this computation halts or diverges.
\end{definition}

\noindent The procedure ``prompt the model and evaluate its completion'' is encoded in the agentic evaluator by the following program:
\begin{equation}
	x^{*} := \texttt{let }c=\lm\,q\texttt{ in }\eval\,c.
	\label{eq:spe-seed}
\end{equation}

\begin{theorem}
	\label{thm:spe}
	$x^*$ is an SPE state of the agentic evaluator.
\end{theorem}

At a lower level of abstraction, the program $\eval$ which evaluates the model completion could be complex. Its logic counts as orchestration policy only when it constrains what successor states of the agentic machine can be reached by the model.

\begin{corollary}
	\label{cor:universality}
	The state $x^*$ completion-generates every computable agentic machine over the same LM interface.
\end{corollary}

In principle, any fixed orchestration policy (if computable) could be installed by a model-written program. On the other hand, any orchestration program written by the model can just as well be fixed externally. The distinguishing feature of SPE is not what orchestration policy can be expressed, and especially not what policy is actually expressed by a given model, but rather what entity is responsible for expressing it.

\section{A practical language for SPE}
\label{sec:spell}

Making SPE practical requires a language with unusual properties, motivating the development of \Spell{} (see Appendix~\ref{app:language} for language details). In SPE, the same data---a model completion---is both the content of the model's context window and the program specifying what context becomes the prefix or prompt for a subsequent turn. Consider the following attempt to express one iteration of a ReAct-style loop as a model-written program, particularly the appearance of \lstinline[language=spell]!this_entire_completion()!:

\begin{lstlisting}[language=spell]
prompt := "Summarize README.md."

// LM turn 1 begins here
file_contents := read_file("README.md")
next_prefix :=
  concatenate(this_entire_completion(), "observation := ", file_contents)
self_call(next_prefix)
\end{lstlisting}

{\bf Challenge 1: context persistence as code.}
The expression \lstinline[language=spell]!this_entire_completion()! serializes the source code of the running program. That source code is edited in order to produce a new prefix, here by concatenating one additional line. In order to generalize such logic, an SPE language should make it easy for the model to manipulate code as data. \Spell{} therefore uses the syntax of Lisp, specifically Clojure \citep{mccarthy1960lisp,hickey2020clojure}: programs are written as nested expressions whose syntax matches the procedure of evaluating them, and this syntax greatly simplifies the logic of programs which generate or transform source code. In order to support the creation of programs or expressions which reference their own source code specifically, \Spell{} also adds a self-referential \lstinline[language=spell]!quine! form. (A \emph{quine} is a program that reproduces its own source code.) The expression \lstinline[language=spell]!(quine name expr)! binds \lstinline[language=spell]!name! to the source form \lstinline[language=spell]!(quine name expr)! as data before evaluating \lstinline[language=spell]!expr!.

Natural language can be embedded in code, and a \Spell{} program will often define long string literals containing the model's visible reasoning or planning. The motivation for making reasoning traces a part of the program is that it lets the model control whether they are retained as context on subsequent turns.

{\bf Challenge 2: replaying effects.}
A serious problem appears when this program is replayed on turn~2. The turn~2 program contains the turn~1 program as a prefix; evaluating it replays the self-call, producing an inescapable loop, and also replays the IO operation. \Spell{} resolves this problem using a trailing-expression pattern. The body of a \Spell{} program performs ordinary local computation but cannot directly trigger effects, including self-calls. Instead, it computes an effectful expression as data. A wrapper around the program body (see below) evaluates only the last such expression, the \emph{trailing expression}, and this special evaluator has access to effect functions. When an old trailing expression is followed by newly appended source code, it is no longer the trailing expression; it is reconstructed as inert data but not evaluated. This makes it safe to replay prior source as the prefix of a later turn.

{\bf Challenge 3: turn-boundary interference.}
A more subtle problem arises because the turn~2 program is evaluated inside of, not subsequent to, the turn~1 program. If the inner program inherited an opaque runtime environment from the outer program, or could mutate that environment, then the model could not reliably reason about its behavior; for example, the inner program could overwrite a binding defined by the outer program. \Spell{} therefore evaluates self-calls in fresh local environments, such that a child program neither inherits nor mutates the environment of its parent.

	{\bf The \Spell{} wrapper.}
A normal \Spell{} program has the following structure, which combines self-reference with replay-safe effects:

\begin{lstlisting}[language=spell]
(quine completion          ; 1. bind the current completion as data
  (eval                    ; 4. evaluate the trailing expression
    (do
      ...                  ; 2. ordinary local computation
      '(effectful-expression)))) ; 3. return quoted trailing expression
\end{lstlisting}

\noindent The outer \lstinline[language=spell]!quine! form binds the symbol \lstinline[language=spell]!completion! to the source code of the entire program. The inner \lstinline[language=spell]!do! block returns the value of its last expression, which is the quoted trailing expression. Because it is quoted, this expression is not evaluated until it is passed to \lstinline[language=spell]!eval!, at which time effect functions, including LM self-calls, are made available. If the model appends an expression after the current quoted trailing expression, the old expression is no longer passed to \lstinline[language=spell]!eval! and becomes inert data. The trailing expression can make use of bindings defined upstream, including the \lstinline[language=spell]!completion! binding.

An iteration of a ReAct-style loop is replicated in \Spell{} by writing a trailing expression which makes a tool call, concatenates the result with \lstinline[language=spell]!completion! inside of the \lstinline[language=spell]!do! block, and makes a self-call with this prefix (\Spell{} provides a convenience function for this). This pattern is easily generalized. Instead of retaining its entire completion, the model can choose which items to include and which to prune (see Appendix~\ref{app:language:context-management}). Tool calls can be chained together, packaged into reusable functions, and composed with arbitrary control flow. Tool calls can also be composed with self-calls; for example, an agent could list files in a directory and assign each file to a subagent.

No model has yet been trained to use \Spell{}, so an important feature of the language is its built-in prompts. The \Spell{} system prompt explains semantics and gives positive and negative examples (Appendix~\ref{app:language:system-prompts}). Each namespace has a documentation prompt that the model can read via tool calls. Besides prompting, \Spell{} exposes two surfaces for harness engineering. First, the LM's first turn is produced by evaluating an initial program, which defaults to a prompt and a self-call. By customizing this program, a user may embed a \Spell{} agent within a user-specified orchestrator program; for example, this program could compose an ordinary self-call with deterministic logic that runs tests or invokes a review agent. Second, the user specifies the effect namespaces with which the model is provided. \Spell{} functions with possibly-dangerous effects are packaged into namespaces, like \lstinline[language=spell]!io-read/! and \lstinline[language=spell]!io-write/!, the availability of which is configurable (Appendix~\ref{app:language:capability-boundaries}).

\Spell{} is not the only possible solution to the practical challenges of SPE. However, any implementation will have to solve the same challenges, and the approach taken by \Spell{}---in particular, Lisp and the trailing-expression pattern---is carefully considered. Appendix~\ref{app:language:principles} articulates the high-level principles which motivate the implementation of \Spell{}.

\section{Empirical evaluation}
\label{sec:eval}

Using \Spell{} in practice could be cognitively demanding, particularly for a model not trained to do so: the model must understand SPE conceptually, learn \Spell{} itself in-context, transfer agentic behaviors to an unfamiliar action space, and additionally solve the assigned task. I sought to address three questions. First, can existing models, without SPE- or \Spell{}-specific training, use \Spell{} to solve challenging agentic tasks? Second, what kind of \Spell{} programs can these models write, and what programs do they write in practice? Third, to what extent does self-orchestration with \Spell{} handicap (or benefit) existing models on standard benchmarks, versus a traditional harness? These experiments evaluate \Spell{} as a realization of SPE, as opposed to either \Spell{} or SPE independently. Configuration details for harnesses and models are provided in Appendix~\ref{app:bench:config}. Protocol details for each benchmark are provided in Appendix~\ref{app:bench:protocols}.

{\bf Observation 1: strong existing models can use \Spell{} to solve challenging agentic tasks.}
I applied six strong models---GPT-5.4, Claude Opus 4.6, GLM-5.1, Kimi-K2.6, Qwen3.7 Plus, and GPT-5.6 Sol---as \Spell{} agents to two coding benchmarks, Terminal-Bench~1.1 and SWE-bench Lite (Figure~\ref{fig:obs1}; Appendix~\ref{app:bench:obs1}). (GPT-5.6 was added after primary analyses were completed). The Terminal-Bench comparison used 32 tasks; the SWE-bench Lite comparison sampled 32 tasks, of which 31 were evaluable after one shared container-build failure. Every model solved at least 12 Terminal-Bench tasks, and every model except Qwen3.7 Plus solved at least 12 SWE-bench Lite tasks. GPT-5.6 Sol tied GPT-5.4 for the most Terminal-Bench successes and led the plotted models on SWE-bench Lite, while also having the fewest terminal run errors.

These results suggest that \Spell{} itself is a viable substrate with which to solve challenging agentic tasks. Moreover, existing models have the capability to learn \Spell{} in context and generalize agentic behaviors to this action space. Terminal errors nevertheless remained substantial for several models, consistent with the cognitive demands imposed by SPE and \Spell{} being close to current capability thresholds. Subsequent analyses primarily target GPT-5.4.

\begin{figure}[tbp]
	\centering
	\includegraphics[width=0.9\linewidth]{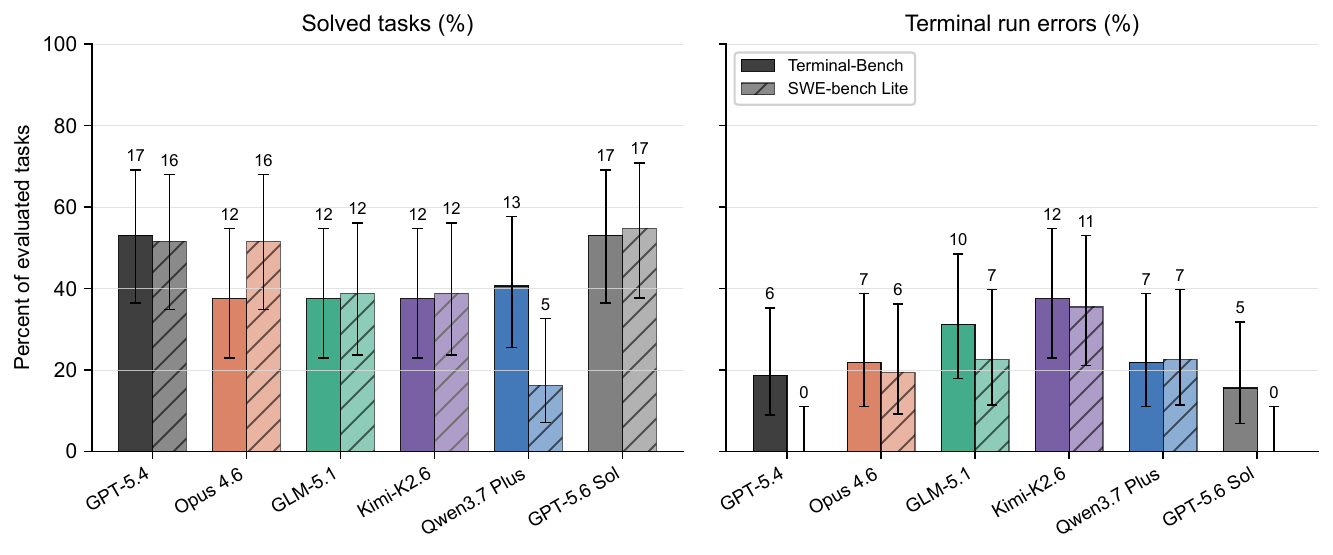}
	\caption{\textbf{Accuracy and terminal run-error rate by model.} Terminal-Bench used the exact 32-task Figure~2 subset; SWE-bench Lite used a seeded 32-task sample with 31 evaluable tasks after one shared container-build failure. Errors remain in the denominator and include timeouts and non-null terminal error records. GPT-5.4, Opus 4.6, and GPT-5.6 Sol used medium reasoning effort; GLM-5.1 and Qwen3.7 Plus used high effort; Kimi-K2.6 used default effort. Error bars show 95\% Wilson intervals; numbers above bars are task counts.}
	\label{fig:obs1}
\end{figure}

{\bf Observation 2: current models use \Spell{} for context management and programmatic tool calling.}
I examined what orchestration policies were expressed in \Spell{} programs produced by GPT-5.4. One important component of orchestration policy is context management, and \Spell{} provides powerful features for this, particularly the \lstinline[language=spell]|!peek| function, which creates ephemeral tool call expressions (Appendix~\ref{app:language:peek}). This behavior was explicitly encouraged by the \Spell{} system prompt, which describes not only how to invoke context-management features but when to use them; prior work on context management found that a prescriptive prompt outperformed a descriptive one \citep{li2025sculptor}.

GPT-5.4 used \lstinline[language=spell]|!peek| heavily in both Terminal-Bench and SWE-bench Lite, and this had a substantial effect on context utilization; compared with Codex CLI (see below), total token count (input and output) was about $4\times$ smaller for \Spell{}, although the difference in API cost was less dramatic (Appendix~\ref{app:bench:cost}). These behaviors intensified with reasoning effort: at high versus low effort, the model used more output tokens but much fewer input tokens, such that total token count actually decreased by about one third (total cost still increased); the same was not observed for Codex CLI (Appendix~\ref{app:bench:context}). Agents also used \Spell{} as a substrate for programmatic tool calling \citep{wang2024codeact}: tool calls were often chained together, conditioned on the success or status of earlier calls, or batched several per turn (Appendix~\ref{app:bench:features}). By contrast, \Spell{} agents rarely used multi-agent features or composed self-calls with control flow. No \Spell{} agent was observed to use multi-agent delegation successfully; two attempts were made, both likely harmful, and instructing the agent to attempt multi-agent orchestration was ineffective (Appendix~\ref{app:bench:features}).

	{\bf Observation 3: orchestration games elicit nontrivial \Spell{} programs.}
I sought to distinguish whether the simplicity of observed \Spell{} programs was due to behavior or capability. I tasked a \Spell{} agent running GPT-5.4 (medium effort) with orchestrating one of three ``orchestration games'' designed to emulate plausibly useful multi-agent orchestration motifs (Appendix~\ref{app:bench:orchestration-games}):
\begin{itemize}[leftmargin=*]
	\item \textbf{Auction}: the orchestrator makes three self-calls asking for sealed bids and selects the highest bid; this emulates a best-of-\(K\) motif.
	\item \textbf{Telephone}: at each iteration of a loop, a self-call instance receives a sentence and is asked to rephrase it, passing the rephrased sentence to the next iteration; this emulates a deterministic workflow.
	\item \textbf{Twenty questions}: the initial agent chooses a secret, and self-call instances ask yes-or-no questions, accumulating a public transcript, until guessing the answer; this emulates a worker-checker motif.
\end{itemize}
GPT-5.4 succeeded according to program trace audits in 8/8 auction trials, 4/8 telephone trials, and 7/8 twenty-questions trials (Appendix~\ref{app:bench:orchestration-games}). 

\noindent\begin{minipage}{\linewidth}
The following was the trailing expression of a successful telephone implementation (comments are added):
\begin{lstlisting}[language=spell,basicstyle=\ttfamily\footnotesize]
'(let [relays
       ; loop variables are k=1, current=initial-wording, acc=[]
       (loop [k 1 current initial-wording acc []]
         (if (> k 8)
           acc  ; value of the loop expression
           ; relay-program packages the prompt for a fresh self-call.
           (let [next-wording (!llm-self (relay-program k current))]
             ; k+=1, current=next-wording, acc+={...}
             (recur (+ k 1) next-wording
                    (conj acc {:relay k :wording next-wording})))))
       final-wording (:wording (last relays))]
   {:initial initial-wording :relays relays :final final-wording})
\end{lstlisting}
\end{minipage}

\Spell{} supports either synchronous delegation, via self-calls, or asynchronous delegation, via the \texttt{agents/} namespace and a more complex API. I re-tested GPT-5.4 on a variant of the orchestration games in which the agent uses asynchronous delegation; it completed valid asynchronous orchestration in only 9/24 trials. However, when GPT-5.6 Sol was tested on the same items (this model was released after main experiments were concluded), it succeeded on 22/24 trials (Appendix~\ref{app:bench:orchestration-games}). 

These results illustrate that \Spell{} supports flexible self-orchestration policies, and they show that strong agentic models are capable of implementing such policies despite not usually choosing to do so. 

	{\bf Observation 4: \Spell{} is competitive on coding benchmarks, although not uniformly.}
I compared \Spell{} against Codex CLI on Terminal-Bench and SWE-bench Lite using the same GPT-5.4 model at low, medium, and high reasoning effort. On Terminal-Bench~1.1 (Figure~\ref{fig:obs2_gpt54_frontiers}, left), Codex at high effort achieved the highest accuracy, resolving 43/80 tasks, but at $\sim 2\times$ higher cost than the default \Spell{} agent at high effort (40/80). Codex at low and medium effort had similar accuracy and cost as the \Spell{} agent at high effort; the \Spell{} agent at low and medium effort had lower accuracy and commensurately lower cost. On SWE-bench Lite, similarly, the \Spell{} agent had lower accuracy but commensurately lower cost compared with Codex. The \Spell{} agent resolved 171/300 tasks at medium effort, nearly matching Codex (172/300) on accuracy at lower total cost; however, high effort benefited the Codex agent (185/300 resolved) but not the \Spell{} agent (171/300). I additionally evaluated mini-swe-agent~\citep{yang2024sweagent}, a strong academic open-source harness, with GPT-5.4 at medium effort (Appendix~\ref{app:bench:mini-swe-agent}). On Terminal-Bench, mini-swe-agent had significantly higher performance than \Spell{} at slightly higher cost (44/80 versus 36/80; \$30.71 versus \$25.72). On SWE-bench Lite, it had slightly lower performance at slightly lower cost (168/300 versus 171/300; \$86.57 versus \$102.12).

I performed two secondary analyses on Terminal-Bench. First, the \Spell{} agent was supplied with a specialized prompt, prescribing an iterative research/plan/implement/verify workflow. This had a positive but statistically ambiguous effect which seemed to vary with reasoning effort: at low effort, accuracy was unchanged, while cost increased by 56\%; at medium effort, the intervention won 8 tasks and lost 4, while cost increased by 6\%; at high effort, the intervention won 5 tasks and lost 2, while cost decreased by 7\% (Appendix~\ref{terminal-bench-1.1-cost-accuracy-frontier}). Second, I compared a \Spell{} agent running Claude Opus~4.6 at medium effort versus Claude Code running the same model; both agents solved 42/80 tasks with similar costs (Appendix~\ref{app:bench:harness}).

Although individual datapoints are noisy, these estimates suggest a genuine cost--accuracy tradeoff. At fixed reasoning effort, the largest quantitative difference between the two harnesses was token utilization. At medium effort, \Spell{} used 6.96M total tokens on Terminal-Bench versus 51.14M for Codex CLI, and 51.37M versus 155.96M on SWE-bench Lite (Appendix~\ref{app:bench:cost}). This reduction is due at least in part to \Spell{}'s context-management features, which sharply reduce repeated input context. The dollar-cost gaps are smaller than the token-count gaps because output and reasoning tokens are more expensive than cached input tokens. The accuracy reduction is more difficult to attribute; possible contributors include overuse of context pruning, prompting differences, and the cognitive overhead of learning \Spell{} itself. 

I evaluated \Spell{} agents on LongBench~v2, a long-context QA benchmark, and AppWorld, a computer-use benchmark. The \Spell{} agent (with GPT-5.4 medium effort) underperformed Codex CLI on LongBench at similar cost (61.0\% vs. 67.5\%) and badly underperformed on AppWorld (42.1\% vs. 63.2\%), at higher cost (Figure~\ref{fig:obs5}). Trace inspection suggested assorted failure modes: query formulation and data-source selection errors, final-submission mistakes, one refusal on a payment task, possible evaluator mismatches, and mismatch between the real run date and fixture dates (Appendix~\ref{app:bench:cross}). It is not clear what about \Spell{} makes it less competitive on these benchmarks. One possibility is that its context management features are more appropriate for coding tasks; another is prompting; a third is that GPT-5.4 is more extensively trained for such tasks, such that effective behaviors are more deeply ingrained and thus generalize more readily.

A limitation of these comparisons is that the comparators differ along axes besides the use of SPE or the core semantics of \Spell{}: Codex CLI is a mature harness likely aligned with GPT-5.4's post-training, whereas \Spell{} is a new implementation of an unfamiliar architecture. Prompts also differ (Appendix~\ref{comparison-with-codex-cli-and-claude-code}). Because of these differences, these results do not necessarily predict the future performance of agents using \Spell{} (see Section~\ref{sec:discussion}).

\begin{figure}[tbp]
	\centering
	\includegraphics[width=0.9\linewidth]{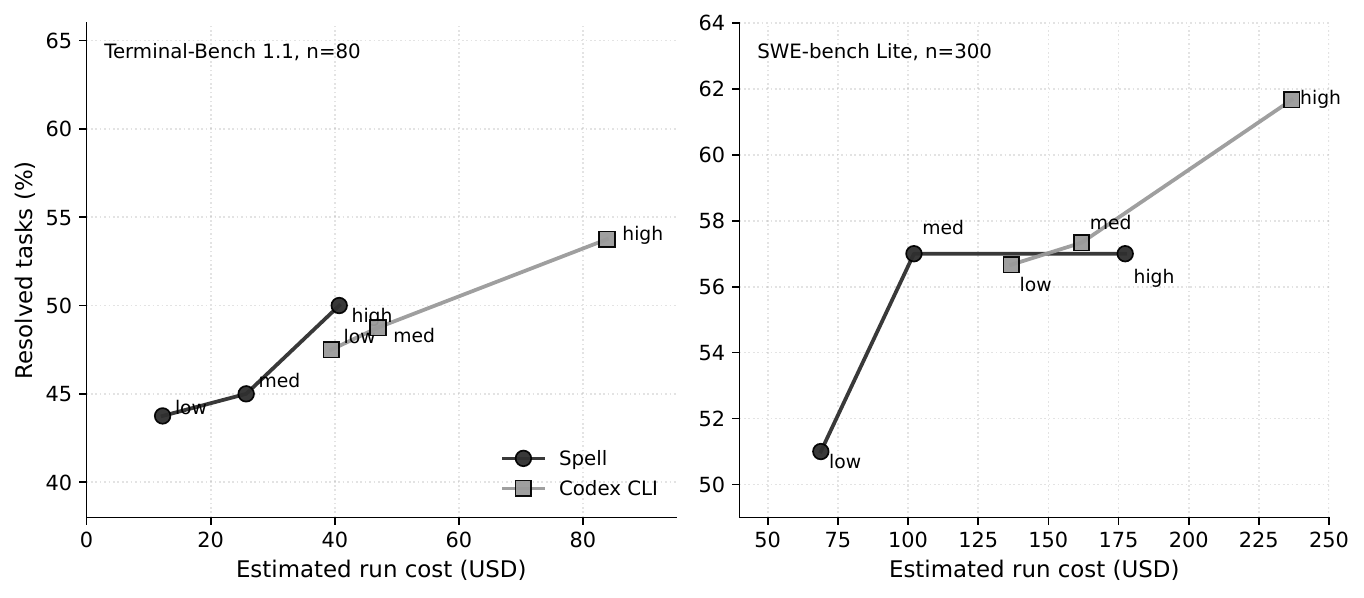}
	\caption{\textbf{Comparison with Codex CLI on coding benchmarks.} Left: Terminal-Bench~1.1. Right: SWE-bench Lite. Each point is one full benchmark run with GPT-5.4 at low, medium, or high reasoning effort. For numerical results, see Appendix~\ref{app:bench:harness}.}
	\label{fig:obs2_gpt54_frontiers}
\end{figure}

\begin{figure}[tbp]
	\centering
	\begin{minipage}[c]{0.62\linewidth}
		\centering
		\includegraphics[width=\linewidth]{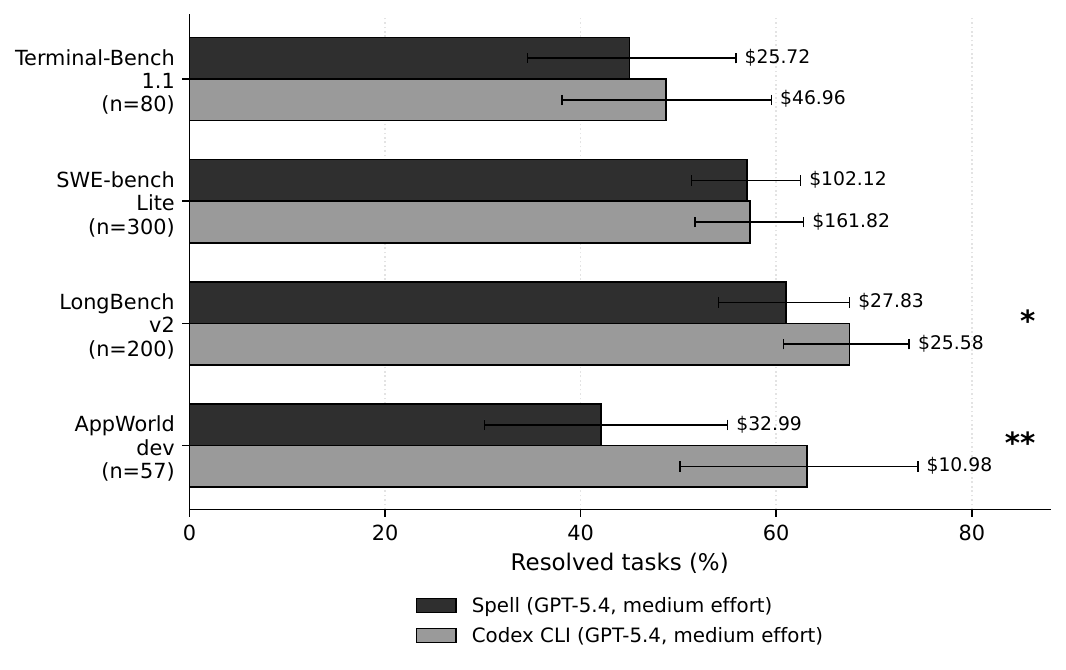}
	\end{minipage}\hfill
	\begin{minipage}[c]{0.34\linewidth}
		\captionsetup{justification=raggedright,singlelinecheck=false}
		\caption{\textbf{Comparison between \Spell{} and Codex CLI across benchmarks.} Error bars show 95\% confidence intervals for each marginal resolved-task rate. Labels give the total API cost of each run. Significance markers report two-sided binomial tests over tasks with discordant outcomes; because the outcomes are correlated, this test can be significant despite overlapping error bars. Numerical results and run-configuration details are reported in Appendix~\ref{app:bench:harness} and Appendix~\ref{app:bench:cross}. *~$p<0.05$. **~$p<0.005$.}
		\label{fig:obs5}
	\end{minipage}
\end{figure}

\section{Related work}
\label{sec:related}

This section reviews recent literature on the design of agentic systems and earlier work on self-referential architectures, and additional language design background is reviewed in Appendix~\ref{app:language:background}.

{\bf Partial self-orchestration.}
Many self-orchestrated architectures, motivated by the limitations of the LM's context window, provide special tools for context management or subagent delegation. MemGPT creates a memory storage system and provides the agent with special tools to manage memory and context \citep{packer2023memgpt}. A series of papers empowers a model to purge, summarize, and more generally manipulate its own context window, including MemAct \citep{zhang2025memact}, FoldGRPO \citep{sun2025contextfolding}, Sculptor \citep{li2025sculptor}, and AgentFold \citep{ye2025agentfold}. These papers are precedent for the context management features of \Spell{}. Another axis of partial self-orchestration is delegation. Conductor \citep{nielson2026conductor} is a language model specifically trained to orchestrate other models on agentic tasks, notably including via self-delegation. RLMs \citep{zhang2026rlm} are especially relevant to SPE because the means of delegation is programmatic. In concurrent work, Context Language Models \citep{shao2026context} place model context into an external file, allowing a model to manage its context using ordinary file-editing tools. Delegation also features in commercial systems, particularly coding agents \citep{anthropic2025agentteams,moonshot2026kimik26}. Whereas these systems add features to an otherwise fixed orchestration policy, SPE simply yields responsibility for this policy to the model.

Several self-orchestrating architectures are paired with LMs trained or tuned for the corresponding interface. FoldGRPO and MemAct train models to manage their own context \citep{sun2025contextfolding,zhang2025memact}; RLMs train models for programmatic delegation \citep{zhang2026rlm}; and Conductor and Kimi K2.6 report gains from models or systems specialized for subagent orchestration \citep{nielson2026conductor,moonshot2026kimik26}. These successes suggest that a parallel approach may be viable for \Spell{} (see Section \ref{discussion}).

{\bf Self-evolving harnesses.}
Another nearby line of work attempts to optimize prompts, programs, or scaffolds through self-improvement. Voyager \citep{wang2023voyager} accumulates LM-authored skills over time, emulating a continual learner. DSPy \citep{khattab2024dspy} optimizes prompts and other LM-program fragments against a metric, effectively automating prompt engineering. EnCompass \citep{li2025encompass} superimposes tree search over a prescribed space of execution paths; STOP \citep{zelikman2024stop} starts with a program-improver program and runs the improver on its own source code; ADAS \citep{hu2025adas} introduces a meta-agent which updates an archive of harness designs; AFlow \citep{zhang2025aflow} augments this meta-agent with tree search. Several other methods vary the means of self-evolution or the space of possible designs \citep{yin2024godelagent,lee2026metaharness,karten2026continualharness}, and a concurrent paper proposes a self-improving RLM \citep{primeintellect2026primeagent}. These methods share with SPE and \Spell{} the motivation that fixed orchestration policy is likely suboptimal, but instead of optimizing this logic, SPE makes it a part of the model's action space. 

{\bf Program execution.}
Many methods involve the execution of a model-written program in a general-purpose programming language. Examples include PAL and Program-of-Thoughts, which are oriented toward numerical calculation, and CodeAct, which uses code as a substrate for tool calling \citep{gao2023pal,chen2023programofthoughts,wang2024codeact}. RLMs also belong to this category \citep{zhang2026rlm}. SPE differs from these systems by making the model-written program responsible for orchestration policy.

{\bf Self-referential systems.}
Outside of the machine learning literature, SPE and \Spell{} belong to a tradition of self-referential languages and systems. Self-referential programs first became practical and idiomatic with Lisp \citep{mccarthy1960lisp}, in which the source code of a program and the data it manipulates are represented the same way (this property is \emph{homoiconicity}). \citet{smith1984reflection} promoted a self-referential motif, namely that of programs which interpret programs, into a system architecture with 3-Lisp: in 3-Lisp, a program is interpreted by an interpreter which is itself interpreted by another interpreter, forming a \emph{reflective tower}. \citet{schmidhuber2007godelmachines} similarly took a self-referential motif, recursive self-improvement, and promoted it into a system architecture with the Gödel machine. This theoretical machine iteratively rewrites its own code, proving at each iteration that the rewrite is beneficial. \Spell{} likewise promotes a self-referential motif, the language model self-call, into a system architecture: an outer program $Q$ invokes a language model to compute an inner program $P$, often by mutating $Q$ itself, and evaluates the inner program. Like 3-Lisp, the outer program causes an inner program to run, but the implementation of this behavior is conventional; the interpreter is external, invoked but not implemented by $Q$. Like the Gödel machine, moreover, \Spell{}'s motif is goal-oriented: the program $Q$ produces $P$ as the means toward some end.

\section{Discussion}
\label{sec:discussion}
\label{discussion}

This paper shows that a stateless language model can act as an agent, reasoning and acting across multiple turns, without any fixed agent loop or other orchestration policy. Instead, one need only execute the model completion as a program. This approach does present various challenges in practice, motivating the development of \Spell{}, but experiments show that \Spell{} addresses these challenges well enough to support capable agents on difficult coding benchmarks.

Models primarily used \Spell{} to implement practical, unexotic orchestration programs. They typically emulated an agent loop by making one recursive self-call per turn, they managed context by pruning tool call results, and they batched tool calls programmatically. They showed the capability to write more interesting orchestration programs for tasks contrived to require it, but they rarely chose to do so on ordinary agentic benchmark tasks.

This behavior could change if an LM were trained to use \Spell{} natively. Production systems increasingly rely on domain-specific harnesses that manage context and trigger certain orchestration motifs deterministically \citep{horthy2025twelvefactoragents}; practitioners describe this work as ``context engineering,''  ``harness engineering,'' and ``loop engineering'' \citep{anthropic2025contextengineering,openai2026harnessengineering,lulla2026loopengineering}. These workflows mostly remix a handful of building blocks in a manner that seems amenable to learned heuristics. For frontier models, it is already within their capability to implement common self-orchestration motifs in \Spell{} (Observation~3), and post-training may elicit such behaviors more reliably. Prior work provides supportive precedent, reporting gains from systems specialized for subagent orchestration \citep{nielson2026conductor,moonshot2026kimik26} and from post-training for context management \citep{sun2025contextfolding}. It remains unclear how far self-orchestration might scale: would long-running tasks accumulate increasing benefit over time? Would novel strategies for multi-agent coordination emerge which outperform human-designed systems? Would an end-to-end approach outperform self-evolving harnesses which freeze the model weights? Despite these uncertainties, there is precedent for end-to-end learning to outscale human effort \citep{sutton2019bitter}; \Spell{}-native training may provide the means by which to scale orchestration with computation.

\begin{ack}
The author thanks Dr. Isaac Kohane, Dr. Marinka Zitnik, Ms. Amber Shen, and the three anonymous reviewers for their valuable feedback. The author acknowledges support from award number R35GM155278 from the National Institute of General Medical Sciences of the National Institutes of Health. The author declares no competing interests.
\end{ack}

\bibliographystyle{plainnat}
\bibliography{references/references}

\clearpage
\appendix


\section{Agentic evaluators and self-programmed execution}
\label{app:theory}

\subsection{Overview}\label{app:theory:overview}

This appendix is structured as follows. Section~\ref{app:theory:agentic-machines} defines an agentic machine. Section~\ref{app:theory:embeddings} defines three notions related to equivalence or simulation of agentic machines (embedding, indistinguishability, and completion-generation), and it proves results relating these notions to each other; the most important is Theorem~\ref{thm:embedding-invisible}, which shows that a machine state and its embedded copy are indistinguishable to the model. Section~\ref{app:theory:agentic-evaluators} defines the agentic evaluator. Finally, Section~\ref{app:theory:spe-states} proves that this evaluator contains an SPE state and establishes the universality results.
 The following table maps main-text definitions to this appendix.

\begin{center}
	\small
	\renewcommand{\arraystretch}{1.25}
	\begin{tabular}{>{\raggedright\arraybackslash}p{0.18\linewidth}>{\raggedright\arraybackslash}p{0.30\linewidth}>{\raggedright\arraybackslash}p{0.42\linewidth}}
		\textbf{Main text item}     & \textbf{Description}                                & \textbf{Appendix A}                                                                                         \\
		Definition~\ref{def:agentic-machine} & Agentic machines                                    & Definition~\ref{def:am}; model semantics in Section~\ref{app:theory:agentic-machines}                       \\
		Definition~\ref{def:embedding}       & Embeddings of agentic machines                      & Definitions~\ref{app:def:embedding} and~\ref{def:lm-perspective}; Theorem~\ref{thm:embedding-invisible}     \\
		Definition~\ref{def:spe-state}       & Reachability, completion-generation, and SPE states & Definitions~\ref{def:completion-generate} and~\ref{app:def:spe-state}                                       \\
		Definition~\ref{def:agentic-evaluator} & Agentic evaluators and boundary states              & Definition~\ref{def:residual-agentic-evaluator} \\
		$x \xleftarrow{\partial} E$ & First boundary state reached by evaluating $E$      & Definition~\ref{def:residual-agentic-evaluator}; Theorem~\ref{thm:main}                                          \\
		Theorem~\ref{thm:spe}       & SPE seed theorem                                    & Theorem~\ref{thm:main}                                                                                      \\
		Corollary~\ref{cor:universality} & Universality for computable/realizable machines     & Corollaries~\ref{cor:computable-realizable} and~\ref{cor:universal-seed}                                    \\
	\end{tabular}
\end{center}

\subsection{Agentic machines}\label{app:theory:agentic-machines}

An agentic machine describes a system that sends a prompt to a model, receives a completion or response, and performs a deterministic transition to a new state that depends on the previous state and the model completion. This setup is related to oracle-parameterized machines \citep{Turing39,Soare09}.

We fix a prompt space $P$ and a completion space $C$. A model semantics for this interface, when one is needed, is any map
\[
	M:P\to\Delta(C),
\]
where $\Delta(A)$ denotes the set of probability measures on $A$ (deterministic models are also permitted). The results below depend only on $P$ and $C$; whether $M$ is deterministic or stochastic does not matter.

\begin{definition}[Agentic machine]\label{def:am}
	An \emph{agentic machine} over $(P,C)$ is a triple
	\[
		X=(S,p,h)
	\]
	consisting of a state space $S$, a prompt function
	\[
		p:S\to P,
	\]
	and a harness function
	\[
		h:S\times C\to S\cup\{\mathbf{1},\uparrow\},
	\]
	where $\mathbf{1}$ is a distinguished halting point and $\uparrow$ is a distinguished divergence point.

	For each model semantics $M:P\to\Delta(C)$, the induced transition kernel
	\[
		T_M^X:S\to\Delta(S\cup\{\mathbf{1},\uparrow\})
	\]
	is defined by
	\[
		T_M^X(s):=h(s,-)_*M(p(s)),
	\]
	where $h(s,-):C\to S\cup\{\mathbf{1},\uparrow\}$ denotes the map $c\mapsto h(s,c)$ and $h(s,-)_*$ denotes pushforward along that map.
\end{definition}

Thus each transition records exactly one model invocation and one returned completion. The machine data $(S,p,h)$ are deterministic; sampling occurs only when a model semantics $M$ is actually queried at the prompt $p(s)$.

Elements of $C$ may be raw model responses, full prompt-response completions, or typed values delivered by an evaluator interface. A real API may expose bytes or text that must first be parsed, adapted, or rejected; such deterministic adapter behavior is part of the harness $h$. In SPE the returned source is naturally treated as extending a prefix already resident in the evaluator, so we use the word \emph{completion}.

The outcome $\uparrow$ records deterministic divergence between one completion and the next boundary or halt.

\subsection{Embeddings and completion generation}\label{app:theory:embeddings}

\begin{definition}[Embedding]\label{app:def:embedding}
	Let
	\[
		X=(S,p,h),\qquad X'=(S',p',h')
	\]
	be agentic machines over the same interface. An \emph{embedding} of $X'$ into $X$ is an injective map
	\[
		e:S'\to S
	\]
	which we extend by convention to
	\[
		e:S'\cup\{\mathbf{1},\uparrow\}\to S\cup\{\mathbf{1},\uparrow\}
	\]
	with $e(\mathbf{1})=\mathbf{1}$ and $e(\uparrow)=\uparrow$, and which satisfies
	\[
		p(e(s'))=p'(s')
	\]
	for every $s'\in S'$ and
	\[
		h(e(s'),c)=e(h'(s',c))
	\]
	for every $s'\in S'$ and $c\in C$.
\end{definition}

This is an injective, pointwise special case of simulation/refinement ideas from the verification literature \citep{AbadiLamport91,LynchVaandrager95,SegalaLynch95}. From the model's perspective an embedding is invisible: the only part of the state that the model sees is the prompt, and the condition $p\circ e=p'$ says that the ambient machine uses exactly the prompt the embedded machine would have used. Because the post-completion dynamics also commute with $e$, the model sees the same prompt process in the embedded run as in the original one. The next definition makes this precise by comparing the prompts and terminal outcomes seen after any fixed finite sequence of future completions.

For an agentic machine $X=(S,p,h)$, a starting state $s\in S$, and a finite completion sequence $\sigma=c_1\cdots c_n$, write
\[
	x_0^X(s;\sigma),\ldots,x_n^X(s;\sigma)\in S\cup\{\mathbf{1},\uparrow\}
\]
for the unique run prefix with $x_0^X(s;\sigma)=s$ and, for $0\le k<n$,
\[
	x_{k+1}^X(s;\sigma)=
	\begin{cases}
		h\bigl(x_k^X(s;\sigma),c_{k+1}\bigr), & x_k^X(s;\sigma)\in S,             \\
		x_k^X(s;\sigma),                      & x_k^X(s;\sigma)\in\{\mathbf{1},\uparrow\}.
	\end{cases}
\]
Thus halt and divergence are absorbing for purposes of comparing later prompts. Also write
\[
	\operatorname{vis}_X(x)=
	\begin{cases}
		p(x),     & x\in S,     \\
		\mathbf{1},        & x=\mathbf{1},        \\
		\uparrow, & x=\uparrow.
	\end{cases}
\]

\begin{definition}[LM-indistinguishability]\label{def:lm-perspective}
	If $(X,s)$ and $(X',s')$ are pointed states over the same interface, we say that they are \emph{LM-indistinguishable}, and write
	\[
		(X,s)\sim (X',s'),
	\]
	if for every finite completion sequence $\sigma=c_1\cdots c_n$ and every prefix length $0\le k\le n$,
	\[
		\operatorname{vis}_X\bigl(x_k^X(s;\sigma)\bigr)
		=
		\operatorname{vis}_{X'}\bigl(x_k^{X'}(s';\sigma)\bigr).
	\]
	Equivalently, after any shared finite sequence of completions, the two runs have both halted, have both diverged, or remain at states exposing the same prompt. In particular, $\sim$ is an equivalence relation on pointed states over a fixed interface.
\end{definition}

\begin{theorem}[Embeddings are invisible to the LM]\label{thm:embedding-invisible}
	Let $e:S'\to S$ be an embedding of $X'=(S',p',h')$ into $X=(S,p,h)$. Then
	\[
		(X',s')\sim (X,e(s'))
	\]
	for every $s'\in S'$.
\end{theorem}

\begin{proof}
	Extend $e$ to outcomes by $e(\mathbf{1})=\mathbf{1}$ and $e(\uparrow)=\uparrow$. Fix a finite completion sequence $\sigma=c_1\cdots c_n$. Let
	\[
		x'_k=x_k^{X'}(s';\sigma),
		\qquad
		x_k=x_k^X(e(s');\sigma)
	\]
	be the two run prefixes. We prove by induction on $k$ that
	\[
		x_k=e(x'_k)
	\]
	for every $0\le k\le n$. The case $k=0$ is immediate. For the induction step, assume the claim holds at some $k<n$. If $x'_k\in\{\mathbf{1},\uparrow\}$, then by the induction hypothesis $x_k$ is the same terminal outcome, and both run prefixes keep that outcome at step $k+1$. If $x'_k\in S'$, then
	\[
		x_{k+1}=h(x_k,c_{k+1})=h(e(x'_k),c_{k+1})=e(h'(x'_k,c_{k+1}))=e(x'_{k+1}),
	\]
	using transition commutation for the embedding.

	It remains to compare what the model sees. If $x'_k\in S'$, then
	\[
		\operatorname{vis}_X(x_k)=p(e(x'_k))=p'(x'_k)=\operatorname{vis}_{X'}(x'_k),
	\]
	and if $x'_k\in\{\mathbf{1},\uparrow\}$ the two visible outcomes are equal because $e$ fixes $\mathbf{1}$ and $\uparrow$. Thus the visible run prefixes agree for every $\sigma$ and every $k$, so $(X',s')\sim (X,e(s'))$.
\end{proof}

\begin{remark}[This theorem does not need injectivity]\label{rem:lm-noinj}
	The proof of Theorem~\ref{thm:embedding-invisible} uses only prompt preservation and transition commutation, not injectivity of $e$.
\end{remark}

\begin{example}[Indistinguishable machines may lack an embedding]\label{ex:indistinguishable-no-embedding}
	Let $S_1=\{s_1\}$ and $S_2=\{s_2\}$ be disjoint singletons, let $P=S_1\sqcup S_2$, and take any completion space $C$. Define
	\[
		T_1=(S_1\times\{0,1\})\sqcup S_2,
		\qquad
		T_2=S_1\sqcup (S_2\times\{0,1\}),
	\]
	and let $X_i=(T_i,p_i,h_i)$ for $i=1,2$, where
	\[
		p_1((s_1,j))=s_1,\quad p_1(s_2)=s_2,
		\qquad
		p_2(s_1)=s_1,\quad p_2((s_2,j))=s_2,
	\]
	and both harnesses stutter: $h_i(t,c)=t$ for every $t\in T_i$ and $c\in C$. Then every state of $X_1$ is LM-indistinguishable from some state of $X_2$, and conversely, by choosing a state with the same prompt; all finite completion sequences leave that prompt unchanged.

	However, neither machine embeds in the other. An embedding of $X_1$ into $X_2$ would have to send the two distinct states $(s_1,0)$ and $(s_1,1)$ to two distinct states of $X_2$ with prompt $s_1$, but $X_2$ has only one such state. Symmetrically, an embedding of $X_2$ into $X_1$ would require two distinct states of $X_1$ with prompt $s_2$, but $X_1$ has only one.
\end{example}

\begin{definition}[Completion-generation by a state]\label{def:completion-generate}
	Let $X=(S,p,h)$ be an agentic machine and let $x_0\in S$. Define the one-step state-reachable set
	\[
		\Reach_X(x_0):=\{t\in S: \exists c\in C\text{ such that }h(x_0,c)=t\}.
	\]
	For another agentic machine $X'=(S',p',h')$ over the same interface, we say that the pair
	\[
		(X,x_0)
	\]
	\emph{completion-generates} $X'$ if there exists an embedding
	\[
		e:S'\to S
	\]
	such that
	\[
		\operatorname{im}(e)\subseteq \Reach_X(x_0).
	\]
	If there exists $x_0\in S$ such that $(X,x_0)$ completion-generates $X'$, we say that $X$ \emph{completion-generates} $X'$.
\end{definition}

The state-only convention in $\Reach_X(x_0)$ is not a separate mathematical assumption. Equivalently, one could take the full one-step image in $S\cup\{\mathbf{1},\uparrow\}$ and then require the embedding image to lie in its state part. We use the state-only form because completion-generation asks whether target states, not terminal outcomes, can be loaded into the ambient state space.

This one-hop requirement says that the embedded state is not merely present somewhere in the ambient machine, but reachable from one fixed state $x_0$ by a single completion-mediated step. That is the formal analogue of loading a state through one completion.

\begin{example}[Completion-generation need not hold]\label{ex:seeds}
	Let $P=C=S=\mathbb{F}_2$, and let the prompt map be the identity $p:S\to P$. For each harness below, consider the agentic machine
	\[
		X=(S,p,h).
	\]
	Then:
	\begin{enumerate}[label=(\arabic*)]
		\item If $h(a,b)=a+b$, then $(X,x_0)$ completion-generates $X$ for every $x_0\in S$.
		\item If $h(a,b)=ab$, then $(X,1)$ completion-generates $X$, but $(X,0)$ does not.
		\item If $h(a,b)=a$, then $(X,x_0)$ does not completion-generate $X$ for either $x_0\in S$.
	\end{enumerate}
	Indeed, in case (1) the image of $h(x_0,-)$ is all of $\mathbb{F}_2$ for every $x_0$; in case (2) the image of $h(1,-)$ is all of $\mathbb{F}_2$ but the image of $h(0,-)$ is $\{0\}$; and in case (3) the image of $h(x_0,-)$ is always $\{x_0\}$. Because $p$ is the identity, any self-embedding of $X$ must be the identity on states. Thus the reachability claims above exactly determine when $X$ completion-generates itself, and case (3) shows that embedding alone does not imply completion-generation.
\end{example}

\begin{proposition}[Downward closure under embedded targets]\label{prop:cg-downward}
	Let
	\[
		X=(S_X,p_X,h_X),\qquad Y=(S_Y,p_Y,h_Y),\qquad Z=(S_Z,p_Z,h_Z)
	\]
	be agentic machines over the same interface, and let $x_0\in S_X$. If $(X,x_0)$ completion-generates $Y$ and there exists an embedding $j:S_Z\to S_Y$ of $Z$ into $Y$, then $(X,x_0)$ completion-generates $Z$.
\end{proposition}

\begin{proof}
	Let $e:S_Y\to S_X$ witness that $(X,x_0)$ completion-generates $Y$. Then $e\circ j:S_Z\to S_X$ is an embedding of $Z$ into $X$, and
	\[
		\operatorname{im}(e\circ j)\subseteq \operatorname{im}(e)\subseteq \Reach_X(x_0).
	\]
	Thus $(X,x_0)$ completion-generates $Z$.
\end{proof}

The next result changes the ambient machine rather than the target machine. It is therefore a separate source-invariance statement, not merely another instance of downward closure.

\begin{proposition}[Completion-generation is invariant under embeddings]\label{prop:cg-invariance}
	Let
	\[
		X=(S_X,p_X,h_X),\qquad Y=(S_Y,p_Y,h_Y),\qquad Z=(S_Z,p_Z,h_Z)
	\]
	be agentic machines over the same interface, let
	\[
		e:S_Y\to S_X
	\]
	be an embedding of $Y$ into $X$, and let $y_0\in S_Y$ with
	\[
		x_0=e(y_0).
	\]
	Then
	\[
		(X,x_0)\text{ completion-generates } Z
		\quad\Longleftrightarrow\quad
		(Y,y_0)\text{ completion-generates } Z.
	\]
\end{proposition}

\begin{proof}
	First assume that $(Y,y_0)$ completion-generates $Z$, witnessed by an embedding
	\[
		f:S_Z\to S_Y
	\]
	with $\operatorname{im}(f)\subseteq \Reach_Y(y_0)$. Then $e\circ f:S_Z\to S_X$ is an embedding of $Z$ into $X$. For each $z\in S_Z$, choose $c\in C$ with $f(z)=h_Y(y_0,c)$; then
	\[
		(e\circ f)(z)=e(h_Y(y_0,c))=h_X(e(y_0),c)=h_X(x_0,c).
	\]
	Thus $\operatorname{im}(e\circ f)\subseteq \Reach_X(x_0)$, so $(X,x_0)$ completion-generates $Z$.

	Conversely, assume that $(X,x_0)$ completion-generates $Z$, witnessed by an embedding
	\[
		g:S_Z\to S_X
	\]
	with $\operatorname{im}(g)\subseteq \Reach_X(x_0)$. Since $x_0=e(y_0)$ and $e$ is an embedding,
	\[
		h_X(x_0,c)=h_X(e(y_0),c)=e(h_Y(y_0,c))
		\qquad(c\in C).
	\]
	Therefore every state in $\Reach_X(x_0)$ lies in $\operatorname{im}(e)$, and hence $\operatorname{im}(g)\subseteq \operatorname{im}(e)$. Because $e$ is injective, the map
	\[
		\tilde g:=e^{-1}\circ g:S_Z\to S_Y
	\]
	is well defined. Extend $\tilde g$ to $S_Z\cup\{\mathbf{1},\uparrow\}$ by $\tilde g(\mathbf{1})=\mathbf{1}$ and $\tilde g(\uparrow)=\uparrow$.

	For prompts, if $z\in S_Z$ then
	\[
		p_Y(\tilde g(z))=p_X(e(\tilde g(z)))=p_X(g(z))=p_Z(z).
	\]
	For transitions, if $z\in S_Z$ and $c\in C$ then
	\[
		e(h_Y(\tilde g(z),c))
		= h_X(e(\tilde g(z)),c)
		= h_X(g(z),c)
		= g(h_Z(z,c))
		= e(\tilde g(h_Z(z,c))).
	\]
	Since $e$ is injective, this implies
	\[
		h_Y(\tilde g(z),c)=\tilde g(h_Z(z,c)).
	\]
	Thus $\tilde g$ is an embedding of $Z$ into $Y$. Finally, if $g(z)=h_X(x_0,c)$ then
	\[
		\tilde g(z)=e^{-1}(g(z))=e^{-1}(h_X(e(y_0),c))=h_Y(y_0,c),
	\]
	so $\operatorname{im}(\tilde g)\subseteq\Reach_Y(y_0)$. Hence $(Y,y_0)$ completion-generates $Z$.
\end{proof}

\subsection{Agentic evaluators}\label{app:theory:agentic-evaluators}

This section defines an agentic machine that evaluates programs, including the self-call of Equation~\ref{eq:self-call}: invoke the model with a prompt, then evaluate its completion as a program.

We follow \citet[Chapters~5--6]{FelleisenFlatt2003}, using the pure core of ISWIM (``If you See What I Mean''), a Turing-complete call-by-value language \citep[Section~3]{Plotkin75}. Its evaluation rules separate the computation to perform next from the continuation to follow. We write $K[M]$ to mean: evaluate $M$, then place its result in the hole in $K[\,]$ and continue. The surrounding program $K$ is called an \emph{evaluation context}; it represents the continuation. One can read $K[M]$ as
\[
    \texttt{let }x=M\texttt{ in }K[x],
\]
where $x$ is a new variable name. We consider \emph{closed} programs, meaning that all their variables are bound within the program.

We add a special form $\lm$ for model calls. It evaluates its argument to a prompt $q$, then invokes the model. Thus $K[\lm q]$ describes a program paused at a model call: when the completion $c$ arrives, evaluation resumes from $K[c]$.

We also add quotation and \texttt{eval} so that a completion can supply a program to run. A quotation $\ulcorner E\urcorner$ represents a closed program $E$ as data, without executing it; \texttt{eval} executes the quoted program:
\[
    \eval(\ulcorner E\urcorner)\longrightarrow E.
\]
We call this extended language $\mathcal{L}$. Prompts and completions are drawn from the language's data values $V$, and every quotation of a closed program is an admissible completion. The self-call is then $\texttt{let }c=\lm\,q\texttt{ in }\eval\,c$: its continuation evaluates the program returned by the model.

\begin{definition}[Agentic evaluator]\label{def:residual-agentic-evaluator}
    The \emph{agentic evaluator} is the agentic machine $X_{\mathcal{L}}=(B,p,h)$, where:
    \begin{itemize}
        \item $B$ is the set of closed programs in $\mathcal{L}$ of the form $K[\lm q]$, for $q\in P$. These are called \emph{boundary states} because the next step of computation is to invoke the model.
        \item $p$ maps $K[\lm q]$ to $q$, the prompt at the active model call. The prompt has already been evaluated to a value.
        \item $h$ maps $(K[\lm q],c)$ to the next boundary state reached by evaluating $K[c]$. It returns $\mathbf{1}$ if evaluation halts before another model call, or $\uparrow$ if it continues forever without reaching one.
    \end{itemize}
\end{definition}

We write $\beta\xleftarrow{\partial}E$ when $\beta$ is the first boundary state reached by evaluating $E$, including $E$ itself if it is already at a boundary.

\subsection{SPE states and universality}\label{app:theory:spe-states}

Throughout this section assume that the prompt space is nonempty.

\begin{definition}[SPE state]\label{app:def:spe-state}
	Let $X=(S,p,h)$ be an agentic machine. A state $x\in S$ is an \emph{SPE state of $X$} if
	\[
		(X,x)
	\]
	completion-generates $X$ itself. More generally, for a class $\mathcal{A}$ of agentic machines over the same interface, $x$ is \emph{universal for $\mathcal{A}$} if $(X,x)$ completion-generates every $X'\in\mathcal{A}$.
\end{definition}

\begin{proposition}[SPE states under embeddings]\label{prop:spe-state-invariance}
	Let $X=(S_X,p_X,h_X)$ and $Y=(S_Y,p_Y,h_Y)$ be agentic machines over the same interface, let
	\[
		e:S_X\to S_Y
	\]
	be an embedding, and let $x\in S_X$. If $y=e(x)$ is an SPE state of $Y$, then $x$ is an SPE state of $X$.
\end{proposition}

\begin{proof}
	Since $(Y,y)$ completion-generates $Y$ and $e$ embeds $X$ into $Y$, Proposition~\ref{prop:cg-downward} implies that $(Y,y)$ completion-generates $X$; Proposition~\ref{prop:cg-invariance} then implies that $(X,x)$ completion-generates $X$.
\end{proof}

The intrinsic SPE property is the first clause: the model can, by choosing a completion, reach an embedded copy of any state of the same machine. The universal clause is useful for comparing a single evaluator state to other agentic machines.

\begin{theorem}[SPE seed]\label{thm:main}
    For every $q_{\ast}\in P$, the boundary state $x_0:=\texttt{let }c=\lm\,q_{\ast}\texttt{ in }\eval\,c$ is an SPE state of $X_{\mathcal{L}}$.
\end{theorem}

\begin{proof}
    Let $\beta\in B$. Its quotation $\ulcorner\beta\urcorner$ is an admissible completion. On receiving it, the seed resumes as $\texttt{let }c=\ulcorner\beta\urcorner\texttt{ in }\eval\,c$, which reduces to $\eval(\ulcorner\beta\urcorner)$ and then to the boundary state $\beta$. Thus
    \[
        h(x_0,\ulcorner\beta\urcorner)=\beta.
    \]
    Hence $\Reach_{X_{\mathcal{L}}}(x_0)=B$, and the identity embedding witnesses that $(X_{\mathcal{L}},x_0)$ completion-generates $X_{\mathcal{L}}$.
\end{proof}

The proof uses the identity embedding, so the exhibited seed state reaches all of $B$. That is stronger than the definition requires. An SPE state only needs to reach the image of some self-embedding $e:B\to B$; if the evaluator admits a proper self-embedded copy of itself, an SPE state can have $\Reach(x)$ strictly smaller than $B$ while still completion-generating the whole evaluator.

\begin{definition}[Realizable agentic machine]\label{def:realizable}
	Let $X'=(S',p',h')$ be an agentic machine over $(P,C)$. We say that $X'$ is \emph{realizable in the underlying language} if there exist:
	\begin{itemize}
		\item an injective encoding
		      \[
			      \enc:S'\to V,
		      \]
		\item a pure deterministic program $\mathsf{Prompt}(z)$ such that evaluating $\mathsf{Prompt}(z)$ with $z=\enc(s')$ yields $p'(s')$,
		\item a first-class result type $R$ with constructors
		      \[
			      \mathsf{next}:V\to R,
			      \qquad
			      \mathsf{halt}\in R,
		      \]
		      together with deterministic case analysis on $R$,
		\item a pure deterministic program $\mathsf{Step}(z,c)$ such that, when $z=\enc(s')$, it yields $\mathsf{next}(\enc(t'))$ if $h'(s',c)=t'\in S'$, it yields $\mathsf{halt}$ if $h'(s',c)=\mathbf{1}$, and it diverges without invoking $\lm$ if $h'(s',c)=\uparrow$.
	\end{itemize}
\end{definition}

Here pure programs make no model calls, including through \texttt{eval}. The tagged result $\mathsf{next}(z')$ instructs the simulated loop to continue from encoded state $z'$, while $\mathsf{halt}$ instructs it to finish. Case analysis inspects this tag and, in the $\mathsf{next}$ case, binds the enclosed state to $z'$.

Call an agentic machine \emph{computable} if its states have an injective representation in $V$ on which its prompt function is computable and its harness is a partial computable procedure: it returns a represented next state or halt, and diverges exactly when the harness outcome is $\uparrow$.

\begin{corollary}[Computable machines are realizable]\label{cor:computable-realizable}
    Every computable agentic machine over $(P,C)$ is realizable in $\mathcal{L}$.
\end{corollary}

\begin{proof}
    The Turing-complete pure core of $\mathcal{L}$ supplies $\mathsf{Prompt}$ and $\mathsf{Step}$ on these finite data representations, as required by Definition~\ref{def:realizable}.
\end{proof}

\begin{proposition}[Embedding realizable machines]\label{prop:realizable-embeds}
    Every agentic machine $X'$ over $(P,C)$ that is realizable in the underlying language embeds into $X_{\mathcal{L}}$.
\end{proposition}

\begin{proof}
    Fix a realization $\enc$, $\mathsf{Prompt}$, and $\mathsf{Step}$ of $X'=(S',p',h')$, writing $\mathsf{Prompt}$ and $\mathsf{Step}$ for fixed closed function values. Define the recursive function
    \[
        \begin{array}{l}
            \mathsf{Loop}(z):=\texttt{let }q=\mathsf{Prompt}(z)\texttt{ in} \\
            \qquad\texttt{let }c=\lm\,q\texttt{ in} \\
            \qquad\texttt{case }\mathsf{Step}(z,c)\texttt{ of} \\
            \qquad\quad\mathsf{next}(z')\Rightarrow\mathsf{Loop}(z') \\
            \qquad\quad\mathsf{halt}\Rightarrow v_{\mathsf{halt}},
        \end{array}
    \]
    where $v_{\mathsf{halt}}$ is any fixed closed value and recursion uses the standard call-by-value construction \citep[Section~5.3]{FelleisenFlatt2003}. For each $s'\in S'$, unfolding $\mathsf{Loop}(\enc(s'))$ and computing its prompt gives the boundary
    \[
        \begin{array}{l}
            e(s'):=\texttt{let }c=\lm\,p'(s')\texttt{ in} \\
            \qquad\texttt{case }\mathsf{Step}(\enc(s'),c)\texttt{ of} \\
            \qquad\quad\mathsf{next}(z')\Rightarrow\mathsf{Loop}(z') \\
            \qquad\quad\mathsf{halt}\Rightarrow v_{\mathsf{halt}}.
        \end{array}
    \]
    All displayed function names abbreviate their fixed closed source values. Thus $e(s')$ is a boundary state and $p(e(s'))=p'(s')$. Its unevaluated continuation contains $\enc(s')$ as the first argument of $\mathsf{Step}$, so injectivity of $\enc$ implies injectivity of $e$.

    Resuming $e(s')$ with $c\in C$ evaluates $\mathsf{Step}(\enc(s'),c)$. If $h'(s',c)=\uparrow$, this computation diverges without a model call; if $h'(s',c)=\mathbf{1}$, it returns $\mathsf{halt}$ and the program terminates. If $h'(s',c)=t'\in S'$, it returns $\mathsf{next}(\enc(t'))$, and the selected branch reduces to $\mathsf{Loop}(\enc(t'))$. Unfolding this application and computing its prompt gives exactly the displayed boundary state $e(t')$. Consequently
    \[
        h(e(s'),c)=e(h'(s',c))
    \]
    for every $s'\in S'$ and $c\in C$, with $e(\mathbf{1})=\mathbf{1}$ and $e(\uparrow)=\uparrow$. Hence $e$ is an embedding.
\end{proof}

\begin{corollary}[Universal seed]\label{cor:universal-seed}
	The seed state $x_0$ from Theorem~\ref{thm:main} completion-generates every computable agentic machine over $(P,C)$.
\end{corollary}

\begin{proof}
	By Theorem~\ref{thm:main}, $(X_{\mathcal{L}},x_0)$ completion-generates $X_{\mathcal{L}}$. Corollary~\ref{cor:computable-realizable} and Proposition~\ref{prop:realizable-embeds} embed every computable agentic machine into $X_{\mathcal{L}}$; Proposition~\ref{prop:cg-downward} gives the claim.
\end{proof}

\begin{remark}[\Spell{} and SPE states]\label{rem:spell}
	\Spell{} is an instance in which quotations are first-class values accepted by \texttt{eval}, so the seed term \texttt{let c = lm ... in eval c} is literal. Under \Spell{}'s trailing-expression discipline, turn-producing expressions fire only in trailing position: appending any further expression deactivates an earlier effectful turn. This makes it natural for successive \texttt{!extend} turns to remain within SPE states, in the sense of Definition~\ref{app:def:spe-state}.
\end{remark}

\begin{remark}[Non-SPE states]\label{rem:nonspe}
	SPE is a property of states, not of the evaluator as a whole. Let $Y$ be a realizable agentic machine and let $e_Y:Y\to X_{\mathcal{L}}$ be its embedding from Proposition~\ref{prop:realizable-embeds}. Proposition~\ref{prop:cg-invariance}, together with downward closure, implies that a non-SPE state $y\in S_Y$ maps to a non-SPE state $e_Y(y)$ inside $X_{\mathcal{L}}$.

	For example, the one-state machine whose harness diverges after every completion is realizable but not SPE, since its one-step reachable set is empty. Its image in $X_{\mathcal{L}}$ is therefore a non-SPE boundary state. Thus an agentic evaluator can be self-programmed into turns that are not themselves SPE states, including harness-level or subagent-like turns.
\end{remark}


\clearpage
\section{Self-programmed execution language for LMs}
\label{app:language}

\subsection{Overview}\label{app:language:overview}

\Spell{} v0.1.0 is a language for self-programmed execution embedded within Clojure, a modern dialect of Lisp. The central mechanism is the self-call: a running \Spell{} program passes a program prefix to an LM for completion and evaluates the result. Often, the prefix sent for completion is computed from the program's own source code. To support this pattern, \Spell{} provides explicit self-reference operations, gates effectful expressions behind a double-evaluation mechanism, enforces strict scoping rules, and provides an idiomatic wrapper for programs that edit and re-run themselves recursively.

This appendix is organized as follows. Section~\ref{app:language:principles} states the design principles behind \Spell{} and explains how the main language mechanisms resolve tensions among those principles. Section~\ref{app:language:background} places \Spell{} in the context of Lisp, reflection, and metaprogramming. Sections~\ref{app:language:core-semantics}--\ref{app:language:error-recovery} describe the core language mechanisms: self-calls, visible reasoning text and provider-side reasoning tokens, quines, double evaluation, the completion wrapper, local runtime environments, context management, and recovery from invalid LM-written programs. Sections~\ref{app:language:concurrent-agents-and-inter-agent-communication}--\ref{app:language:prebuilt-orchestration-patterns} describe concurrent agents, messaging, ordinary language features, and prebuilt orchestration patterns. Sections~\ref{app:language:prompts-and-internal-documentation}--\ref{app:language:spell-runtime} document the prompt surface, namespace guides, transport variants, agent/provider configuration, and runtime implementation used by the experiments.

\subsection{Design principles}\label{app:language:principles}

\Spell{} was designed around three principles closely connected to SPE itself. Other language designs for SPE are possible, but I argue that such alternatives should respect similar principles and solve similar challenges. 

\begin{enumerate}[leftmargin=*]
	\item \textbf{Orchestration is ordinary program logic.} It should be easy for the model to emulate common orchestration policies, such as a ReAct-style loop, and to generalize them in arbitrary ways. This is a general principle of programming languages: complex programs are expressed by composing modular logical components.
	
	\item \textbf{The completion is the program.} Any part of the model's completion which is not a part of the executable program cannot be addressed by the program and therefore cannot be passed as context through the turn boundary, except by quoting it in expensive output tokens. For this reason, the prefix in particular should be executed as part of the program.

	\item \textbf{The model controls what its program does.} The term ``self-programmed'' refers not only to who writes the source code but also to who determines its behavior. This principle has important consequences because often, a model writes only part of its completion; its prefix is already given, and when it makes a self-call, it evaluates an inner program written by a different instance.
\end{enumerate}

These principles explain why the three practical challenges in the main text are problematic, and they motivate certain solutions. First, \emph{context persistence as code} is the challenge that a program must carry its own future context forward as executable source. This need arises due to Principle~1, since context persistence is a feature of common orchestration policy, and Principle~2 creates the challenge that this context is itself the code which is responsible for its own persistence. This motivates the use of Lisp because Lisp makes it easy and idiomatic to manipulate source code programmatically, unlike most other languages. It also motivates the \lstinline!quine! form, because a program that edits or extends itself needs a reliable source-level handle on its own code.

Second, \emph{replay-safe effects} is the challenge that replaying source as context must not replay old tool calls or self-calls. By Principle~2, the prefix supplied to the model is part of the program, but the current model instance did not write that prefix; if the prefix directly performs effects, then those effects run again without being chosen, violating Principle~3. Principle~1 also implies that the solution should be part of the structure of the program; in particular, this rules out a stateful interpreter that tracks what expressions were previously evaluated. \Spell{} solves this problem using effect function gating and the trailing expression pattern, a natural structural solution though not necessarily the only one.

Third, \emph{turn-boundary interference} is the challenge that when a program written by model A performs a self-call and evaluates the program written by model B, neither A nor B controls or even sees the other's program. This creates tension with Principle~3 if, for example, program B could overwrite a binding in the runtime environment of program A. \Spell{} therefore evaluates self-calls in fresh local environments: a child turn neither inherits nor mutates the parent environment, and the only durable medium across the turn boundary is source code written into the prefix. This is also why \Spell{} does not feature closures; a closure is an opaque function value with an attached environment, and the only way to pass such objects across the turn boundary is to serialize their dependencies as source code, which defeats the point. Closures are not inherently problematic, however, and may be useful in other designs.

Various design choices in \Spell{} are motivated not by any particular principle but rather by what behaviors they elicit from existing models. This is similar to programming languages for humans, which are always designed around programmer behavior to some extent. Such choices are most likely to change in the future, as model behaviors evolve.

\subsection{Background}\label{app:language:background}

\Spell{} belongs to a tradition of self-referential languages and programs. Theoretical work on self-referential programs includes Kleene's second recursion theorem \citep{kleene1952metamath}, and \citet{goedel1931} famously studied self-referential statements. With the creation of Lisp \citep{mccarthy1960lisp}, it became practical and idiomatic to write programs that transform programs, specifically because in Lisp, programs and data share the same representation (this property is \emph{homoiconicity}). A powerful application of this approach is to customize the language itself, giving rise to Lisp dialects. Homoiconicity is also useful when writing programs that reference or reflect upon their own source code; the computing term \emph{quine} was introduced by \citet{hofstadter1979geb} for such programs, after W. V. O. Quine's work on indirect self-reference. \citet{smith1984reflection} described reflection this way:

\begin{quote}
	It is as if we were creating a magic kingdom, where from a cake you
	could automatically get a recipe, and from a recipe you could
	automatically get a cake.
\end{quote}

Another language with first-class support for metaprogramming is MetaML, which adds type safety and strict scoping rules for runtime-generated code \citep{taha2000metaml}. Many compiled languages support at least some compile-time metaprogramming (e.g., Rust), and many interpreted languages support but discourage the evaluation of generated code at runtime (e.g., Python).

\subsubsection{Clojure}\label{app:language:clojure}

\Spell{} is embedded within Clojure \citep{hickey2020clojure}. It is implemented in Clojure and also adopts most of Clojure's semantics. Clojure is a modern Lisp with functional programming features, particularly immutability. It is a hosted language, running on the Java virtual machine or within JavaScript, and it interoperates with these languages without a translation layer. It was chosen over other Lisps because immutability enables a powerful concurrency model, which addresses challenges that naturally arise in multi-agent systems. Multi-agent features of \Spell{} are not emphasized in this paper, mostly because existing LLMs do not utilize them, but they motivate the implementation of the \Spell{} runtime.

\subsection{Core semantics}\label{app:language:core-semantics}

\subsubsection{\texorpdfstring{Completions and
		\texttt{!llm-self}}{Completions and !llm-self}}\label{completions-and-llm-self}

A \textbf{completion} is the prefix passed to the LM together with the suffix produced by the LM. In \Spell{}, completions are executable programs.

The central primitive is \passthrough{\lstinline"!llm-self"}. At a high level, it behaves as follows:

\begin{lstlisting}
(defn !llm-self [prefix]
  ;; Pseudocode only
  (let [response   (call-llm prefix)
        completion (str prefix response)
        result     (spell-eval completion {})]
    (:ok result)))
\end{lstlisting}

The actual implementation includes parsing, balancing, error recovery, and handle management, but this pseudocode captures the key idea: \Spell{} passes the LM an \textbf{open program prefix}, receives a completion, evaluates it, and returns its value.

Because \passthrough{\lstinline"!llm-self"} is an ordinary callable form, it can appear inside conditionals, loops, recursive functions, map-style dispatch, or user-defined orchestration helpers.

\subsubsection{String literals}\label{app:language:thinking-transport}

A model's context window often contains mostly natural language, including prompts, reasoning traces, and tool call results. In \Spell{}, these text fragments are usually string literals inside the program. They can be integrated into the program using \passthrough{\lstinline!def!} bindings, \passthrough{\lstinline!quine!} forms, or the inert \passthrough{\lstinline!think!} form. The reason to put this content inside the program is so that it can be manipulated programmatically and used to construct a prefix for a self-call. For some model-provider configurations, \Spell{} also allows the model to emit reasoning tokens that are not part of the program; this is a concession to model behavior that technically violates Principle~2, but only superficially, because the model can always choose to embed reasoning traces into the program instead. 

\subsubsection{Quines and explicit
		self-reference}\label{quines-and-explicit-self-reference}

To extend or rewrite its own context, the program must be able to refer to its current completion as structured data. \Spell{} provides the special form \passthrough{\lstinline!quine!} for this purpose. The following program prints its own source code:

\begin{lstlisting}
(quine self (pr-str self))
;; => "(quine self (pr-str self))"
\end{lstlisting}

\passthrough{\lstinline!(quine name body)!} binds
\passthrough{\lstinline!name!} to the entire quine form as data and then
evaluates \passthrough{\lstinline!body!}. This allows the body to
inspect or transform the very program that contains it. A quine form
evaluates to the value of its body.

\subsubsection{Double-evaluation of effectful
	expressions}\label{double-evaluation-of-effectful-expressions}

Because \Spell{} programs often edit and re-run themselves, expressions that have side effects (e.g., making an LLM call or interacting with the filesystem) must be treated with special care. In the default evaluator applied to a \Spell{} program, effectful functions are entirely unavailable; they raise an error if called. The single exception is the \passthrough{\lstinline!eval!} function, which can be called from the main program to evaluate an inner program with side effects. This produces a clean boundary between pure code that can be re-evaluated safely and effectful code that should not be. Specifically, the \Spell{} wrapper uses this mechanism to create programs with side effects that run exactly once.

\subsubsection{The completion wrapper}\label{the-completion-wrapper}

Normal \Spell{} programs comprise a body and a standard \textbf{wrapper}:

\begin{lstlisting}
(quine completion
  (eval
    (do
      ... ; the body
    )))
\end{lstlisting}

This wrapper has two purposes:

\begin{enumerate}
	\def\labelenumi{\arabic{enumi}.}
	\tightlist
	\item
	      It enables self-reference via the outer
	      \passthrough{\lstinline!quine!}.
	\item
	      It ensures that only the \textbf{trailing expression}, which is the
	      last expression of the do block, can have externally visible effects.
\end{enumerate}

The outer \passthrough{\lstinline!eval!} performs a second evaluation on the value returned by the \passthrough{\lstinline!do!} block, namely its \emph{trailing expression}. If this expression is quoted, then \passthrough{\lstinline!eval!} evaluates it, allowing it to trigger side effects such as LLM calls. This pattern guarantees that even though the LLM writes only a suffix of the completion, and cannot rewrite its prefix before it is evaluated, this prefix cannot have undesired effects. Only one expression of the prefix---the trailing expression---has effects at all; these effects are ``cancelled'' by appending another expression after it, since a \passthrough{\lstinline!do!} block returns the value of its last expression.

\subsubsection{\texorpdfstring{Extensions and
		\texttt{!call-now}}{Extensions and !call-now}}\label{extensions-and-call-now}

A common pattern is to call a tool, serialize the result into the completion, and then continue reasoning with the result in context. \Spell{} packages this pattern into \passthrough{\lstinline"!call-now"}:

\begin{lstlisting}
...
'(!call-now result-name (tool-call))
\end{lstlisting}

On the next turn, the prefix observed by the LM is its original program, including the tool call itself, and the result of the tool call materialized as a binding:

\begin{lstlisting}
...
'(!call-now result-name (tool-call))
(def result-name "literal result of tool-call")
...
\end{lstlisting}

It supports multiple name-expression pairs, and later pairs can reference earlier ones in a manner similar to \passthrough{\lstinline!let!}. The related function \passthrough{\lstinline"!extend"} simply performs a self-call without making any tool calls; it is often used in the initial program to trigger the first LM turn. The implementation of \passthrough{\lstinline"!call-now"} takes \passthrough{\lstinline!completion!}, adds the tool call result inside the \passthrough{\lstinline!do!} block, and calls \passthrough{\lstinline"!llm-self"}. It can be used for purposes other than traditional tool calls, such as subagent calls and mathematical calculations.

\subsection{Runtime environments and local
	state}\label{runtime-environments-and-local-state}

In Clojure, the \passthrough{\lstinline!eval!} function is global: it both reads from and writes to the global environment. This behavior is undesirable when \passthrough{\lstinline"!llm-self"} evaluates a completion because the LM has no way to read the global environment. A binding defined by a parent LM could be overwritten by a child LM, or the environment could become cluttered with forgotten functions and variables. This challenge is mostly specific to the agentic setting, where the code-writing entity can read only its own subprogram.

\Spell{} solves this challenge by giving its evaluator (\passthrough{\lstinline!spell-eval!}) a local runtime environment which is separate from those of its parent and children. Together with the gating of effect expressions, this choice creates a simple distinction between local state, which depends only on the source code of the program, and global state, with which a \Spell{} program interacts only through the trailing expression.

An implication of strict local scoping is that \Spell{} has little use for closures, which are opaque functions that capture the environment in which they are defined (in particular, closures often capture helper functions). Due to the scoping rules of \Spell{}, there is no way for such an object to pass through the boundary between LM turns except by serializing it as source code, and doing so would defeat the point of closures. Therefore, functions in \Spell{} are dynamically scoped; if a function body uses a free binding (e.g., \passthrough{\lstinline!completion!}), then this binding is looked up in the environment wherever the function is called.

\subsection{Context management}\label{app:language:context-management}

As a completion grows through repeated extensions, it accumulates stale context. \Spell{} provides two high-level ways to manage this accumulation: the LM can build a new prefix from scratch by adding together the pieces it wishes to keep, or it can subtract from its current completion the pieces it wishes to drop.

The first approach requires only the \passthrough{\lstinline"!llm-self"} primitive, and \Spell{} additionally provides a convenience function, \passthrough{\lstinline!wrap-cat!}, which supports this by concatenating any number of forms into the \passthrough{\lstinline!do!} block of the \Spell{} wrapper. This approach is maximally flexible, but in practice LMs do not use it, and it may add substantial overhead if the LM is required to decide on every turn what context is still relevant. The second approach is used more often in practice. It relies on a pair of marker forms, \passthrough{\lstinline!prune!}, \passthrough{\lstinline!persist!}, which are signals to the macro-like \passthrough{\lstinline!prune-and-reopen!} form.

\subsubsection{\texorpdfstring{\texttt{prune} and \texttt{persist}}{prune and persist}}\label{prune-and-persist}

The \passthrough{\lstinline!prune!} function marks one or more preceding expressions for deletion. At runtime, it is inert, but when a program is passed through \passthrough{\lstinline!prune-and-reopen!}, both the prune expression itself and the appropriate number of preceding expressions are deleted.

The \passthrough{\lstinline!persist!} operator lets the LM keep a derived value after pruning away the object from which that value was computed. In particular, the combination of \passthrough{\lstinline!prune!} and \passthrough{\lstinline!persist!} allows the LM to delete a large tool-call result from context while retaining a slice or some other computed summary:

\begin{lstlisting}
...
'(!call-now big-file (io/read-lines "big-file.txt"))
(def big-file [...]) ; 1000 lines
(prune 1)
(persist lines (subvec big-file 32 42))
'(!extend)
\end{lstlisting}

At runtime, \passthrough{\lstinline!persist!} behaves like \passthrough{\lstinline!def!}: it evaluates the expression and binds it to a name. When \passthrough{\lstinline!prune-and-reopen!} is applied, however, the expression inside the \passthrough{\lstinline!persist!} form is replaced with the literal value which is currently bound to that name. On the subsequent turn, the program above produces the following:

\begin{lstlisting}
...
'(!call-now big-file (io/read-lines "big-file.txt"))
(persist lines [...]) ; 10 lines
'(!extend)
...
\end{lstlisting}

Most turn-producing expressions in \Spell{}, except the primitive \passthrough{\lstinline"!llm-self"}, apply \passthrough{\lstinline!prune-and-reopen!} to their argument before making the LM call.

\subsubsection{\texorpdfstring{\texttt{!peek}}{!peek}}\label{app:language:peek}

In order to elicit active context management, a convenience function \passthrough{\lstinline"!peek"} combines \passthrough{\lstinline"!prune"} with \passthrough{\lstinline"!call-now"} to produce an ephemeral tool call whose result is visible for only one turn (unless the model keeps it using \passthrough{\lstinline!persist!}). When the model writes:

\begin{lstlisting}
'(!peek contents (io/read-lines "big-file.txt"))
\end{lstlisting}

the expression appends both the tool call result(s) and a \passthrough{\lstinline!prune!} expression to the prefix for the following turn. If there are multiple tool calls, then all of them are pruned.

\subsubsection{\texorpdfstring{\texttt{think} and \texttt{rethink}}{think and rethink}}\label{think-and-rethink}

\Spell{} encourages LMs to use string literals for their internal reasoning, using \passthrough{\lstinline!think!}:

\begin{lstlisting}
(think "I am a helpful assistant...")
\end{lstlisting}

At runtime, \passthrough{\lstinline!think!} evaluates its body and returns \passthrough{\lstinline!nil!}. It is meant to be combined with \passthrough{\lstinline!rethink!}, which is sugar for \passthrough{\lstinline!prune!} followed by \passthrough{\lstinline!think!}. The idea is that the LM can backtrack and prune away unproductive reasoning traces in this way. The language itself does not couple these functions; \passthrough{\lstinline!rethink!} can be used to prune any preceding expression, not only one involving \passthrough{\lstinline!think!}. This pruning occurs when the program passes through \passthrough{\lstinline!prune-and-reopen!}.

\subsection{Error recovery}\label{app:language:error-recovery}

LM-written programs will sometimes contain errors. \Spell{} includes error recovery mechanisms which allow the LM to recover after an initial error.

\subsubsection{Result-map errors}\label{result-map-errors}

\passthrough{\lstinline!spell-eval!} returns result maps rather than
throwing host-language exceptions directly:

\begin{lstlisting}
Success: {:ok value :env env'}
Error:   {:err message :env env :expr failing-expression :trace [...]}
\end{lstlisting}

The \passthrough{\lstinline!:trace!} field records the \Spell{}-level call path through which the error propagated. For host-function errors, \Spell{} rewrites the message so it is expressed in terms of \Spell{}-facing names rather than internal Clojure machinery.

\subsubsection{Deterministic namespace
	fixup}\label{deterministic-namespace-fixup}

A frequent LM mistake is to use a symbol without the required namespace qualification---for example, \passthrough{\lstinline!trim!} instead of \passthrough{\lstinline!strings/trim!}. \Spell{} first attempts a deterministic repair by searching the available namespaces. If there is exactly one matching resolution, the system substitutes it and re-evaluates.

\subsubsection{Trailing expression error
	recovery}\label{trailing-expression-error-recovery}

When the program is successfully parsed, but an error is thrown by the trailing expression inside \passthrough{\lstinline!eval!}, it is possible to rescue the program by rewriting that one expression. This is handled by appending an error message and an error recovery prompt after the trailing expression inside the \passthrough{\lstinline!do!} block, and re-invoking the model. The model can rewrite the trailing expression and continue.

\subsubsection{Other evaluation error
	recovery}\label{other-evaluation-error-recovery}

When the outer evaluator throws an error, the mechanism described above would fail because appending new expressions inside the \passthrough{\lstinline!do!} block would still allow the error to be re-triggered. Instead, the error message and recovery prompt are appended inside a new \passthrough{\lstinline!(eval (do ...))!} block within the top-level \passthrough{\lstinline!quine!} form. For example, suppose the failing program contains an effect call outside the trailing expression:

\begin{lstlisting}[language=spell]
(quine completion (eval (do
  (quine prompt "List the files.")
  (def files (io/ls "."))         ;; throws: io/ is unbound here
  '(!call-now n (count files)))))
\end{lstlisting}

The runtime catches the error, appends \lstinline[language=spell]!(prune)! and a fresh \lstinline[language=spell]!(eval (do ...))! block as additional arguments to the top-level \lstinline[language=spell]!(quine completion ...)!, and re-evaluates:

\begin{lstlisting}[language=spell]
(quine completion
  (eval (do                       ;; inert: not the last form of the quine
    (quine prompt "List the files.")
    (def files (io/ls "."))
    '(!call-now n (count files))))
  (prune)                         ;; tells the next turn to drop the inert block
  (eval (do                       ;; this block runs on the recovery turn
    (def _recovery_prompt
      "The previous Spell program threw an error. ...")
    (def _error {:error "Unbound symbol: io/ls"
                 :in '(io/ls ".")})
    '(!llm-self (reopen completion)))))
\end{lstlisting}

This works because \passthrough{\lstinline!quine!} with arity greater than two evaluates only the last form; appending a recovery form to the error-producing form avoids re-raising the error. The error-causing form is visible to the model for only one turn, which avoids the accumulation of dead code. The recovery prompt instructs the model to retain any context from the error-causing form that will be needed on subsequent turns.

\subsubsection{Reader recovery}\label{reader-recovery}

If the completion cannot be parsed at all---for example because of unbalanced parentheses---\Spell{} cannot embed it as normal code. In that case, the raw text is wrapped into a fresh recovery quine as a string. The LM then gets another chance to produce a valid continuation. Compared with the more common evaluation recovery path, this path can be expensive because when the error-producing program is wrapped as a string literal, it misses the KV cache.

\subsection{Concurrent agents and inter-agent
	communication}\label{app:language:concurrent-agents-and-inter-agent-communication}

\Spell{} supports concurrent agents and enables both synchronous and asynchronous communication between them. Messages are addressed using handles. During execution, every program is associated with one handle, and the \passthrough{\lstinline"!llm-self"} function (as well as convenience functions like \passthrough{\lstinline"!call-now"}) produces an execution trace that has the same handle as the program which produces it. Programs with the same handle run synchronously in the same thread. The initial program has the handle \passthrough{\lstinline!:main!}. A special function, \passthrough{\lstinline!agents/spawn!}, creates an execution trace with a new handle, and this trace runs asynchronously in its own Clojure thread. Thus, there are two primary models for subagent delegation in \Spell{}: synchronous self-delegation, in which there is no communication except argument-passing, and asynchronous agent-spawning, which allows for communication.

This section describes the concurrent-agents communication protocol implemented in \Spell{} v0.1.0. This protocol was later revamped in \Spell{} v0.4.0.

\subsubsection{Inbox-based messaging}\label{inbox-based-messaging}

Each handle has an inbox that stores queued message macros. These macros transform a completion by appending a message binding (a serialized map containing the sender handle and message body) and then \passthrough{\lstinline"'(!extend)"}. Whatever completion the recipient produces, its trailing expression is intercepted and replaced, giving the model an opportunity to see the incoming message before taking further action. The underlying inbox-macro mechanism is quite general and could be used for other kinds of inter-agent coordination (for example, graceful shutdown).

\subsubsection{Synchronous
	communication}\label{synchronous-communication}

Ordinarily, inbox-based messaging is asynchronous: an agent sends a message and then takes another turn immediately, not waiting for a reply. Synchronous communication, where an agent sends a message and blocks for a response, poses a design challenge because it introduces the potential for deadlock, for example if agents Alice and Bob send blocking messages to each other and wait for a reply simultaneously. A related challenge is that if Alice has finished its work and gone dormant, then Bob has no way to know this before sending a message.

\Spell{} guarantees that agents cannot deadlock in this way by guaranteeing that when Alice sends a blocking message to Bob and goes to sleep, the state of Alice (awake or asleep) is coupled with that of Bob. The message awakens Bob, and if Bob subsequently sleeps, then this awakens Alice. If two agents send blocking messages to each other simultaneously, it causes them both to be awake, not asleep.

\subsubsection{Other ways to sleep}\label{other-ways-to-sleep}

The same sleeping mechanism is used in three other situations. First, when an agent other than \passthrough{\lstinline!main!} finishes its work and returns, any agent that was sleeping for a response from this agent is awakened. Then, the agent enters a sleeping state from which it can be awakened by any incoming message. This allows a worker to start a new task if assigned, or answer a question about work that it has completed, without losing its previous context.

Second, an agent can spawn one or more child agents and sleep until all of them finish their work. As usual, an agent that does this can also be awakened by an incoming message.

Third, an agent can sleep until an arbitrary computation finishes in a separate thread. For example, an agent could listen for a change to some file and awaken when the file is modified; it could spawn subagents in some sort of task loop and awaken when the task loop terminates. If one of those subagents messages the main agent, then the main agent is awakened early, without interrupting the spawned execution thread. Such programs can fail to halt, but this failure mode is no worse than that of any program; deadlock between agents is a worse problem because the agents in question lack the information needed to avoid it.

This claim is about deadlock between agent handles, not arbitrary host-thread liveness. \Spell{} does not prevent ordinary futures from waiting on one another forever. This is a less concerning failure mode: futures created independently by different agents do not know about one another through the communication system, while mutually blocking futures created by the same agent are part of a single program and therefore visible to the author of that program.

\subsubsection{Why this avoids deadlock}\label{why-this-avoids-deadlock}

Consider the directed graph $(A,E)$ where $A$ is the set of agent handles, and $(a,b)$ belongs to $E$ if agent $a$ is sleeping for a response from agent $b$, at some time $t$. Each node is either awake or asleep. The following transformations are possible from time $t$ to $t+1$:

\begin{enumerate}
	\def\labelenumi{\arabic{enumi}.}
	\tightlist
	\item
	      A new node can be created, either awake or asleep.
	\item
	      For a node $a$ which is awake at time $t$, any number of new edges
	      $(a,b)$ can be created ($b\neq a$); at time $t+1$, $a$ will be asleep
	      and $b$ will be awake. If $b$ goes from asleep to awake, then any edge
	      $(b,c)$ is deleted.
	\item
	      A node $b$ can go from awake to asleep, and this deletes any edge
	      $(a,b)$. If the out-degree of $a$ becomes zero, then $a$ becomes awake
	      at time $t+1$.
\end{enumerate}

Deadlock occurs when every node is asleep and $E$ is nonempty. A non-deadlocked state never gives rise to a deadlocked state. In particular, transformation (2) never generates a directed cycle. Clojure provides synchronization primitives that make it possible to avoid undesired state transitions (other than (1)--(3) above) via race conditions; the \Spell{} v0.1.0 implementation appears to accomplish this.

\subsubsection{Communication API}\label{communication-api}

The communication forms are:

\begin{itemize}
	\tightlist
	\item
	      \passthrough{\lstinline!(agents/spawn prompt)!} --- spawn a new agent
	      asynchronously and return its handle;
	\item
	      \passthrough{\lstinline!(agents/spawn agent prompt :handle-name)!} ---
	      spawn with an explicit compiled agent and/or handle name;
	\item
	      \passthrough{\lstinline!(agents/send target message)!} ---
	      fire-and-forget message delivery;
	\item
	      \passthrough{\lstinline!(agents/reply msg-map message)!} --- reply to
	      a message map (the map must contain a \texttt{:from} handle);
	\item
	      \passthrough{\lstinline"(agents/!ask target msg)"} --- send a message
	      and sleep until a reply arrives;
	\item
	      \passthrough{\lstinline"(agents/!ask target)"} --- wake the target and
	      sleep;
	\item
	      \passthrough{\lstinline"(agents/!ask [a b c])"} --- wake several
	      targets and sleep until all return;
	\item
	      \passthrough{\lstinline"(agents/!reply-ask msg-map message)"} ---
	      reply to \texttt{msg-map} and block for the next message;
	\item
	      \passthrough{\lstinline"(agents/!spawn-ask prompt)"} --- spawn a child
	      and sleep until it completes;
	\item
	      \passthrough{\lstinline"(agents/!spawn-ask [prompt-a prompt-b ...])"}
	      --- spawn several children and sleep until all complete;
	\item
	      \passthrough{\lstinline"(!ask-await fut)"} --- sleep for the
	      completion of a thread.
\end{itemize}

When a non-root handle finishes, its completion is preserved in a sleeping state so that another agent can later wake it again.

\subsection{Other language features}\label{other-language-features}

\subsubsection{Macros}\label{macros}

\Spell{} supports Clojure's built-in macros such as \passthrough{\lstinline!when!}, \passthrough{\lstinline!cond!}, \passthrough{\lstinline!defn!}, \passthrough{\lstinline!->!}, and \passthrough{\lstinline!->>!}, as well as user-defined macros via \passthrough{\lstinline!defmacro!}:

\begin{lstlisting}
(defmacro unless [test & body]
  (list 'if test nil (cons 'do body)))
\end{lstlisting}

\subsubsection{Error handling}\label{error-handling}

\Spell{} provides structured exception-style control flow:

\begin{lstlisting}
(try
  (/ 1 0)
  (catch e "division failed"))

(try
  (throw {:code 404})
  (catch e (:code e)))
\end{lstlisting}

Bindings created before the error remain available in the catch handler.

\subsubsection{Destructuring and
	iteration}\label{destructuring-and-iteration}

\passthrough{\lstinline!let!}, \passthrough{\lstinline!fn!},
\passthrough{\lstinline!loop!}, and \passthrough{\lstinline!for!}
support standard vector and map destructuring.
\passthrough{\lstinline!loop!}/\passthrough{\lstinline!recur!} provide
tail-recursive iteration, and \passthrough{\lstinline!for!} provides
list-comprehension-style iteration with \passthrough{\lstinline!:when!}
and \passthrough{\lstinline!:let!}.

\begin{lstlisting}
(loop [n 5 acc 1]
  (if (= n 0)
    acc
    (recur (- n 1) (* acc n))))

(for [x [1 2 3 4] :when (> x 1) :let [sq (* x x)]]
  sq)
\end{lstlisting}

\subsubsection{Futures}\label{futures}

\passthrough{\lstinline!(future expr)!} evaluates
\passthrough{\lstinline!expr!} in a new thread while capturing the
current environment. Futures are isolated from the parent's later
environment updates. Raw blocking operations live in the
\passthrough{\lstinline!blocking/!} namespace, which is only injected
inside futures. This keeps direct thread blocking out of ordinary agent
turns, where sleeping should go through the inbox/wakeup protocol
(\passthrough{\lstinline"agents/!ask"},
\passthrough{\lstinline"agents/!spawn-ask"}, or
\passthrough{\lstinline"!ask-await"}) so incoming messages can wake or
preempt the agent. Inside a future, there is no active agent completion
to reopen, so \passthrough{\lstinline!blocking/!} is the explicit API
for waiting on futures or agent completion tokens from scheduler-style
code.

\subsubsection{Other features shared with
	Clojure}\label{other-features-shared-with-clojure}

\Spell{} implements many features of Clojure, most of which are rarely or never used. A common failure mode observed during development of \Spell{} was that models attempted to use built-in Clojure functions which are not implemented in \Spell{}; as a result, \Spell{} implements Clojure built-in functions by default, unless there is a reason not to do so.

\subsection{Prebuilt orchestration
	patterns}\label{app:language:prebuilt-orchestration-patterns}

The \passthrough{\lstinline!patterns/!} namespace includes several prebuilt orchestration patterns, implemented in \Spell{} as opposed to Clojure. In the future, this could include a library of patterns which the model invokes when appropriate; at present, models do not choose to use this in my experiments, and the \passthrough{\lstinline!patterns/!} namespace is not made available in the benchmarking analyses of this paper.

An example of such a pattern is a worker-checker loop. The main agent spawns the loop with a prompt. The checker agent is tasked with diagnosing a problem and returns a structured map containing instructions for the worker agent; the worker agent attempts a solution, and this is sent back to the checker agent for verification. The checker agent can either continue the loop or announce completion.

\subsection{Prompts and internal
	documentation}\label{app:language:prompts-and-internal-documentation}

\Spell{} v0.1.0 features three categories of prompts: system prompts, which are injected by the \passthrough{\lstinline!call-llm!} function every turn and are configured by the \passthrough{\lstinline!agents.edn!} file; error recovery prompts, which are injected when a \Spell{} program throws an error and is sent to the LM for recovery; and model-discoverable prompts and documentation, which can be injected into context using the \passthrough{\lstinline"!describe"} function.

\subsubsection{System prompts}\label{app:language:system-prompts}

The system prompt teaches the model the core semantics of \Spell{}, provides examples of idiomatic usage, and warns against common pitfalls. There are three variants of the prompt, based on the transport used to obtain a \Spell{} suffix from the model (Section~\ref{app:language:transport}), and prompts additionally have slight modifications based on the available namespaces. The main experiments using Anthropic, OpenAI, and Fireworks tool-call transports used the tool-call variant, shown below with the prompt cursor marker rendered as \passthrough{\lstinline!|!} for LaTeX compatibility.

\begin{lstlisting}
INTRODUCTION

You are writing Spell, a Lisp resembling Clojure, designed for LLM self-orchestration and context engineering.
Spell allows you to act as an agent without an external harness by writing a self-calling program.
Your input is the prefix of a Spell program; your output completes it.
The completion is evaluated by the Spell interpreter. The completion is the program: any natural language, markdown, or commentary will cause a parse error.
The prefix usually contains instructions in a string literal; produce a program which completes the task, computes a response, or produces a self-call as step toward doing so.

TRANSPORT

This prompt is for a tool-call transport. Your response should comprise exactly one tool call named `spell_suffix`, with no assistant message text, explanations, markdown, or wrapper prose.
The content of your user message will be a Spell program prefix. The payload of your tool call must be the raw Spell suffix. These are concatenated to produce a Spell program.
Only visible Spell code is parsed or preserved. Hidden reasoning is not part of the program and will not propagate to later turns. If you utilize hidden reasoning that should persist across turns, write it as part of the program via (think "...").

Example: Anthropic tool-call transport
  {
    "messages": [
      {
        "role": "user",
        "content": "(quine completion (eval (do (quine prompt \"Inspect the project root.\") "
      }
    ],
    "tools": [
      {
        "name": "spell_suffix",
        "description": "Return the full Spell suffix in input.suffix",
        "input_schema": {
          "type": "object",
          "properties": {
            "suffix": {
              "type": "string"
            }
          },
          "required": ["suffix"],
          "additionalProperties": false
        }
      }
    ],
    "tool_choice": {
      "type": "any"
    }
  }

  {
    "type": "tool_use",
    "name": "spell_suffix",
    "input": {
      "suffix": "(think \"Goal: inspect the project root. Next action: list top-level files. Success: identify the main entrypoints.\")\n'(!call-now files (io/ls \".\"))"
    }
  }

Example: OpenAI Responses custom-tool transport
  {
    "input": "(quine completion (eval (do (quine prompt \"Inspect the project root.\") ",
    "tools": [
      {
        "type": "custom",
        "name": "spell_suffix",
        "description": "Spell suffix emitted as custom tool input"
      }
    ],
    "tool_choice": "required"
  }

  {
    "type": "custom_tool_call",
    "name": "spell_suffix",
    "input": "(think \"Goal: inspect the project root. Next action: list top-level files. Success: identify the main entrypoints.\")\n'(!call-now files (io/ls \".\"))"
  }

SPELL BASICS

The core mechanic in Spell is self-calling:
the !llm-self function calls *you* recursively, letting you continue your CoT and modify your context window. Functions which produce a self-call are named with a leading !
Several Spell functions combine !llm-self with functionality like tool calling.
Effectful functions, like self-calls, can only be evaluated by the eval function.
Spell has most Clojure builtins but removes I/O and host interop.
Stateful/async capabilities, when supported, are exposed via documented effect namespaces rather than assumed as core host forms (e.g. do not assume atom/letrec).
It has namespaces. Available namespaces are listed below.
Functions defined in Spell have dynamic scope (no closures).

COMPLETION WRAPPER

Programs use this standard wrapper:
  (quine completion (eval (do ...))) ; you fill in ...
The wrapper has three layers:
1. (do ...) returns the value of its last expression (called the trailing expression). *Normally this value is a quote.*
2. (eval ...) evaluates this quote. Effect functions (those with global side effects) can only be evaluated by eval and otherwise throw "unbound symbol":
    (quine completion (eval (do (!llm-self "No"))))  ; unbound symbol exception
    (quine completion (eval (do '(!llm-self "Yes"))))  ; quote is unwrapped by eval
3. (quine completion ...) binds the source code of the entire program, including the wrapper itself, to the symbol completion. This allows you to extend your CoT (see below).

This wrapper allows you to extend your CoT by self-prompting with your completion while ensuring that effectful function calls are not re-evaluated. If you see this prefix:
    (quine completion (eval (do (quine prompt "Your task...")
    '(!extend) |... ; !extend calls !llm-self; see below
It means that on your previous turn, you called !extend. When you append to this, the quoted expression becomes inert:
    (quine completion (eval (do (quine prompt "Your task...")
    '(!extend) ; no longer trailing; not re-evaluated
    '(!peek ...) ; new trailing expression is evaluated

Tool call example:
  (quine completion (eval (do (quine prompt "Your task...") |
  '(!call-now files (io/ls ".")) ;; !call-now is a main way to use tools; see below

Use !llm-self and wrap-cat to perform a self-call with a properly-wrapped prefix:
    ...|
    (quine msg-to-self "Hello me!")
    (def something-else "Something else")
    '(!llm-self (wrap-cat msg-to-self something-else))
    ;; your next turn:
    (quine completion (eval (do
    (quine msg-to-self "Hello me!") "Something else"|

ONE TRAILING EXPRESSION PER RESPONSE

Each response ends with a quoted expression — the trailing expression. This expression is returned by the wrapper's do block, as data, to the wrapper's eval block, which evaluates it. Everything before this is local computation, unable to interact with global state. In particular, only the trailing expression may make self calls, and only when quoted. Always quote your trailing expression: (quine completion (eval (do ... '(trailing-expr))))

EXTENSIONS

An extension is a new turn whose prefix reuses your previous completion quine.
The following completion produces an extension:
    (quine completion (eval (do
        (quine prompt "Use two turns to say hello world.")
        '(!llm-self (reopen completion))))) ; reopen keeps the quine wrapper and strips trailing parenthesis to reopen its do block

Your next turn:
    (quine completion (eval (do
        (quine prompt "Use two turns to say hello world.")
        '(!llm-self (reopen completion)))|
        "Hello world!")))
    ;; the program returns; !llm-self is not re-evaluated

NEW-TURN FUNCTIONS

Functions that create a new turn are prefixed with `!`.

Extension-producing forms:
    '(!extend) ; simple extension
    '(!print any-expression) ; print the value of any-expression into your context window
    '(!call-now result-name any-expression) ; bind the value to result-name and print; accepts multiple name-expr pairs
    '(!peek result-name any-expression) ; like !call-now, but the bound value is ephemeral for one turn while the !peek call stays visible as prior work; see below; accepts multiple name-expr pairs
    '(!describe some-namespace) ; print documentation; accepts multiple namespace names
    '(agents/!ask :agent-handle query) ; see below

Usually, include exactly one ! expression in your quoted trailing expression. Zero ! expressions: your program returns and you do not get another turn. Two or more ! expressions are valid but usually unnecessary; only use them when you intentionally want parallel work.

AGENT BOOTSTRAP PATTERNS

For one-shot child delegation, strongly prefer agents/!spawn-ask over any
two-turn bootstrap such as spawn -> extend -> ask.

Preferred:
  '(agents/!spawn-ask "Do X and send/return the result.")

Fragile antipattern:
  '(do (agents/spawn child-prompt :child)
       (agents/!ask :child start-msg))

Usually-also-fragile antipattern for one-shot work:
  '(do (agents/spawn child-prompt :child)
       (!extend))
  ;; next turn:
  '(agents/!ask :child start-msg)

Why this is fragile:
- A spawned child gets its own first turn.
- If that first turn does not send a real message, your immediate !ask may wake
  on the child's completion fallback instead of a substantive reply.
- That often appears as (def msg-N {:from :child :body nil}).

Rule of thumb:
- If you want one result from one newly-created child, use agents/!spawn-ask.
- Use spawn + later !ask only when you intentionally want a persistent agent
  that will be reused across later turns.
- Ask an already-existing handle only when the agent already exists or is being
  kept alive on purpose.

When writing child prompts, be explicit about first-turn behavior:
- If the child should answer immediately, make the first turn send/reply with a
  real payload.
- If the child should wait for a later wakeup, make the first turn intentionally
  inert and do not interpret its completion as a meaningful answer.

LEAF-LLM

You can also make plain-text LLM calls using leaf-llm. leaf-llm calls your same model to return a one-shot text result; the leaf-llm subagent cannot write Spell code or use tools. leaf-llm is like a tool: use it with !call-now.
For example:
    ...|'(!call-now poem (leaf-llm "Write a poem about spring."))
    ;; on your next turn:
    ...'(!call-now poem (leaf-llm "Write a poem about spring."))(def poem "What strange scent...")|

QUINE VS DEF

`def` binds a name to a value:
  (def answer (+ 41 1))  ; answer = 42

`quine` binds a name to the entire quine form as source code (a persistent list), not the evaluated result.
  (quine q (+ 41 1))  ; q = (quine q (+ 41 1)), NOT 42
  (+ (eval q) 2) ; => 42
  (+ q 2) ; => exception

Unlike `quote`, `quine` does not block evaluation. In fact, the expression (quine name my-expr) evaluates to the value of my-expr:
    (def forty-two (quine q (+ 41 1)))
    forty-two  ; => 42

Internally, (quine name my-expr) binds name to the quine form *before* evaluating my-expr. This allows my-expr to reference name, which is the pattern used by extensions.

Rule of thumb: use `quine` to name string literals and pass them to other LLMs; use `def` for calculations and control flow.

THINK

A convenient way to insert CoT into your program is:

(think "...")

PRUNE/RETHINK

Manage your context window by pruning unneeded expressions. You can do so using prune: (prune) causes the preceding expression to be dropped on the subsequent turn; (prune N) drops N expressions.

You can also do so using rethink. For example:
  ...(quine prompt "Hard arithmetic problem")
  (think "Let's try to use mental math...")  ; 1k tokens
  (rethink "Instead of using mental math, let's write a program ...")
  '(!extend)
  ;; Next turn, the wrong attempt is pruned:
  ...(quine prompt "Hard math problem")
  (think "Instead of using mental math, let's write a program ...")
  (defn f [x] ...) ...

You can also use rethink to delete a tool call result from context:
    ...'(!call-now file (io/ls "big-dir"))
    (def file [{:name "file1.txt" :size 230} {:name "file2.txt" :size 481} {:name "subdir/" :size 4096} ...])
    (rethink "big-dir has 1000 files. The one I was looking for is big-dir/my-file.txt")
    '(!call-now file-text (io/read-lines "big-dir/my-file.txt"))
    ;; prunes the list of 1000 files from your context window!

`!peek` automates this ephemeral-binding pattern:
    ...'(!peek file-lines (io/read-lines "big-dir/huge-file.txt"))
    ;; end of turn 1 completion
    (def file-lines ["... many lines ..."])
    (prune 1)
    ;; start of turn 2 suffix (not shown)
    ;; file-lines is gone on the following extension; the !peek call remains visible

`persist` retains a computed value across prune/rethink pruning:
    ...'(!peek data (io/read-lines "src/server.py"))
    ;; end of turn 1 completion
    (def data (first-line 1 ["..." "..." ...]))
    (prune 1)
    ;; start of turn 2 suffix
    (persist target-lines (subvec data 180 220))
    '(!peek ...)
    ;; next turn: data is pruned, target-lines survives as a literal value
    (persist target-lines (first-line 181 ["..." "..." ...]))

Do not use persist inside of custom Spell macros.

NAMESPACES

Access functions with qualified symbols.

Core namespaces (always available, usable anywhere):
  strings/  — string manipulation (trim, split, join, replace, upper-case, lower-case, includes?, starts-with?, ...)
  math/     — math functions (sqrt, pow, abs, floor, ceil, rand, factorial, PI, ...)
  builtins/ — reference for core builtins by category (docs only; includes subs, re-find, re-matches, re-seq, rand-int, ...)

Effect namespaces are configurable and only usable in the quoted trailing expression.
These are listed below if available to you. They include things like io and inter-agent communication.
For IO, prefer dedicated io/ functions like io/read-lines over io/sh with shell equivalents.

On first use of an unfamiliar effect namespace, consider using '(!describe ns) to check available functions. '(!describe ns1 ns2) also works. '(!describe ns :function-name) describes one function within a namespace.

CONTEXT MANAGEMENT

Context tokens are your scarcest resource. Each extension should carry forward only what the next step needs.
If tool output, tests, or file contents contradict your current theory, update the theory promptly instead of defending the first idea.

Use !peek for read-only results you only need for the next turn.
- Exploratory file reading and grepping
- Running tests with possibly-verbose outputs
- Other shell commands with possibly-verbose outputs, like package installation

When using !peek to read a file, use io/read-lines. Then, use persist and subvec to keep important lines in context.

Example: large files
    '(!peek server-lines (io/read-lines "src/server.py")
             test-lines (io/read-lines "tests/test_server.py"))
    ;; end of turn 1 completion
    (def server-lines (first-line 1 ["..." "..." ...]))
    (def test-lines (first-line 1 ["..." "..." ...]))
    (prune 2)
    ;; start of turn 2 suffix
    (persist server-focus (subvec server-lines 180 220))
    (persist test-focus (subvec test-lines 40 72))
    '(!extend)

    ;; on turn 3:
    ... ; no server-lines or test-lines
    (persist server-focus (first-line 181 ["..." "..." ...]))
    (persist test-focus (first-line 41 ["..." "..." ...]))
    ...

Example: long documents in chunks
    '(!peek chunk-1 (io/read-lines "docs/report.md" 1 80))
    ;; end of turn 1 completion
    (def chunk-1 (first-line 1 ["..." "..." ...]))
    (prune 1)
    ;; start of turn 2 suffix
    (def chunk-1-summary "Lines 1-80 define the setting, notation, and main claim.")
    (persist chunk-1-focus (subvec chunk-1 20 35))
    '(!peek chunk-2 (io/read-lines "docs/report.md" 81 160))

    ;; on turn 3:
    (def chunk-1-summary "Lines 1-80 define the setting, notation, and main claim.")
    (persist chunk-1-focus (first-line 21 ["..." "..." ...]))
    (def chunk-2 (first-line 81 ["..." "..." ...]))
    (prune 1)
    (def chunk-2-summary "Lines 81-160 give the method, key evidence, and open questions.")
    '(!extend)

Example: test outputs
    '(!peek-now test-out (io/sh "uv run pytest tests/test_api.py -x -q"))
    ;; end of turn 1 completion
    (def test-out "=========================== FAILURES ===========================\n... tests/test_api.py::test_empty_input ... ValueError ...")
    (prune 1)
    ;; start of turn 2 suffix
    (think "Current failure: tests/test_api.py::test_empty_input still raises ValueError on empty input. Next step: patch the guard in the handler and rerun this test.")
    '(!extend)


Use !call-now instead of !peek when the tool call result should remain in context:
- Reading a short, critical snippet of text, like a function definition
- Making an edit that you may wish to re-edit
- Running a tool call whose result is short (~100 tokens)

When you receive a large tool-call result, rethink to replace it with a summary of what you found, then continue:
    '(!call-now hits (io/grep "TODO|FIXME" "src/" {:context 20}))
    (def hits "... many matches ...")
    (rethink "Relevant matches are auth.py:42 and db.py:88.")
    '(!call-now auth-lines (io/read-lines "src/auth.py" 30 70) db-lines (io/read-lines "src/db.py" 90 120))

Chain tool calls with !call-now or !peek, saving turns:
    (def test-script "...")
    '(!call-now _ (io/write-file "/tmp/run_tests.py" test-script) test-out (io/sh "python /tmp/run_tests.py"))

After extended reasoning, rethink to compress your chain of thought to its conclusion, then extend or act:
    (think "Long analysis... checking stack traces, testing hypotheses...")
    (rethink "Root cause: off-by-one loop bound in parse_args.")
    '(!call-now ...)

When your context has grown large over many turns:
    '(!compact)

RECOMMENDED PATTERNS

Fetching documentation:
  '(!describe math)

Multiple tool calls in one turn:
  '(!call-now files (io/ls ".") content (io/read-file "main.py"))
  ;; each step's result is bound and visible next turn; prefer this over wrapping effects in a bare (do ...)

Search + read context in one turn:
  '(!call-now hits (io/grep "TODO|FIXME" "src/" {:context 20}))
  ;; :context N includes N lines around each match in the output, reducing the need for follow-up reads.
  ;; use when you need both "where does this appear" and "what's happening near it" in one turn.

Calculate and extend:
    '(!call-now result (+ 41 1))

Math helper function:
  (defn hypotenuse [a b]
    (math/sqrt (+ (* a a) (* b b))))
  '(!call-now result (hypotenuse 5 12))

Plan + execute with a clean context window:
  (quine prompt "...") (think "...") '(!call-now ...) ... ; 1k tokens of thinking + tool calls
  (quine plan "...")
  '(!llm-self (wrap-cat prompt plan))

Defining reusable code:
  (def run-tests '(!call-now test-out (io/sh "uv run pytest test_module.py -x -q")))
  run-tests ; observe failure
  ;; subsequent turns:
  ... ; read test files, edit source code, etc.
  run-tests ; repeat until they pass
  ;; you can also define functions

ANTIPATTERNS

Starting with prose — your response is parsed as code; prose causes a parse error:
  |Sure, I'll help with that!...  ; WRONG: parse error
  |(think "My approach is...") ; correct

Closing the do block — your response continues inside the open (do ...) block; do not close it:
  prefix: (quine completion (eval (do (quine prompt "...") '(!extend)
  |)                                    ; WRONG: closes the do block
  (think "...") (def x 1) '(!extend)    ; these are now extra args to eval
  |(think "...") (def x 1) '(!extend)   ; correct: continues inside the do block

Re-emitting the wrapper — the prefix already has it; just write expressions:
  |(quine completion (eval (do ...)))  ; WRONG: creates nested wrappers
  |(think "...")...  ; correct: continues the prefix do block

Ending with prose:
    '(!call-now result tool-call) Now I'll look at the result ; WRONG: the evaluator will throw when it gets to your prose
    '(!call-now result tool-call) ; correct: just pass

Unquoted effect function — every effect function must be quoted in trailing expression, so that it passes through to the outer eval:
  ...(!call-now files (io/ls "."))    ; WRONG: unquoted
  ...'(!call-now files (io/ls "."))  ; correct: quoted
  (defn delegate-subtask [subtask] (!llm-self (wrap-cat prompt context subtask))) ; WRONG: unquoted
  (defn delegate-subtask [subtask] (list '!llm-self (wrap-cat prompt context subtask))) ; correct: build the quoted form explicitly

  After receiving a !call-now result, the NEXT !call-now still needs quoting:
  '(!call-now previous-result previous-call)(def previous-result {:exit 1 ...})
  (!call-now content (io/read-file path))  ; WRONG: still need to quote this

Using io/sh for file reads when a structured io/ function already exists:
  '(!call-now content (io/sh "cat src/server.py")) ; WRONG
  '(!call-now content (io/read-lines "src/server.py")) ; correct

Making tool calls without giving yourself a turn:
  '(do (io/write-file "/tmp/run_tests.py" test-script)
       (io/sh "python /tmp/run_tests.py")) ; WRONG: you will never get a turn to see if tests pass
  '(!call-now _ (io/write-file "/tmp/run_tests.py" test-script) test-out (io/sh "python /tmp/run_tests.py")) ; correct
  ;; Whenever you emit a trailing expression with no ! self-call, it means you are finished

Multiple extensions in one response — only the last quoted expression fires:
  ;; all in one turn:
  '(!describe io)
  '(!call-now files (io/ls "."))  ; only this fires; !describe is inert
  ;; correct: just '(!describe io) and pass your turn

Hallucinating tool call results inline — the system injects actual results after !call-now:
  '(!call-now content (io/read-file path))
  (def content "import numpy...")  ; WRONG: instead, just pass your turn

Self-delegation or tool calling with def — this is what !call-now is for:
  (def fix '(!llm-self (wrap-cat ...)))  ; WRONG: fix is the list (!llm-self ...), not the return value
  '(!call-now fix (!llm-self (wrap-cat ...)))  ; correct: !call-now captures the return value

Calling !llm-self (or agents/spawn ...) with an unwrapped quine:
  (quine prompt "Write a poem")
  '(!llm-self prompt) ; WRONG: next turn, the wrapper is missing
  '(!llm-self (wrap-cat prompt)) ; correct: wrap-cat applies the wrapper
\end{lstlisting}

\subsubsection{Model-accessible prompts and
	documentation}\label{model-accessible-prompts-and-documentation}

\Spell{} makes documentation available to the model for namespaces that it wishes to use. The model accesses this documentation by writing \passthrough{\lstinline"(!describe namespace)"}. A special \passthrough{\lstinline!reminders!} namespace includes prompts that can either be accessed by the model or injected by the user via the initial program. The coding reminder was not injected in the main experiments reported in Figures~\ref{fig:obs1}--\ref{fig:obs5}; it was injected only in the secondary Terminal-Bench intervention reported in Appendix~\ref{terminal-bench-1.1-cost-accuracy-frontier}. Although the reminder remained available through \passthrough{\lstinline"(!describe reminders :coding)"}, trace inspection found no calls to it in the main experiments.

\paragraph{\texorpdfstring{\texttt{strings} namespace guide}{strings namespace guide}}\label{strings-namespace-guide}

\begin{lstlisting}
STRINGS — Mirrors clojure.string. Regex functions take string patterns (not compiled regex).

Same as Clojure: index-of, last-index-of, starts-with?, ends-with?, includes?, blank?, trim, replace, split, split-lines, join, lower-case, upper-case, capitalize.

Related builtins: subs, re-find, re-matches, re-seq.

Use (!describe strings :fn-name) for any function.
\end{lstlisting}

\paragraph{\texorpdfstring{\texttt{math} namespace guide}{math namespace guide}}\label{math-namespace-guide}

\begin{lstlisting}
MATH — Mirrors java.lang.Math with standard semantics.

  Basic:         sqrt, cbrt, pow, exp, expm1, abs, sign
  Rounding:      floor, ceil, round, trunc
  Logarithms:    log (natural), log10, log2, log1p
  Trigonometric: sin, cos, tan, asin, acos, atan, atan2
  Hyperbolic:    sinh, cosh, tanh
  Angles:        degrees (rad->deg), radians (deg->rad)
  Number theory: factorial, gcd, lcm
  Misc:          hypot, rand
  Type checks:   NaN?, infinite?
  Constants:     PI, E, INF, NEG-INF, NaN

Related builtins: rand-int, +', -', *', inc', dec', float, double, long, bigdec, rationalize.

Type caveat: many math functions return Doubles, and Clojure `=` does not treat `4` and `4.0` as equal. Coerce explicitly when you need an integer check:

  (let [r (math/round (math/sqrt s))]
    (= (* r r) s))

floor, ceil, round, trunc, factorial, gcd, and lcm return integer values.

All functions take and return numbers. Use (!describe math :fn-name) for any function.

Recommended usage pattern: Write a function, evaluate, inspect the result.

  ...|(defn fib [n] (if (<= n 1) n (+ (fib (- n 1)) (fib (- n 2)))))
  '(!call-now result (fib 10))

Antipattern: return the result of a computation without inspecting it.
Bind the result, inspect it on the next turn, then decide what to do next.
\end{lstlisting}

\paragraph{\texorpdfstring{\texttt{builtins} namespace guide}{builtins namespace guide}}\label{builtins-namespace-guide}

\begin{lstlisting}
BUILTINS — Core functions always available without namespace prefix.

Categories (use (!describe builtins :category) for full listing):
  special-forms — quote, def, persist, do, if, let, fn, quine, loop, recur, for, try
  macros        — when, defn, cond, if-let/if-some, when-let/when-some, case, ->, ->>, !call-now, ...
  effect        — eval, !llm-self, !ask-await, leaf-llm, describe-fn, llm (trailing expression only)
  math          — +, -, *, /, inc, dec, mod, abs, integer?, numerator, denominator, rand, ...
  comparison    — <, >, =, not, nil?, empty?, identity, ...
  types         — string?, number?, vector?, map?, fn?, keyword?, integer?, ratio?, rational?, ...
  strings       — str, pr-str, subs, cat, format, read-string, re-find, ...
  collections   — list, vector, set, first, rest, nth, conj, count, get, assoc, into, ...
  maps          — keys, vals, merge, update, get-in, assoc-in, dissoc, select-keys, ...
  sequences     — map, filter, reduce, sort, group-by, take, drop, partition, range, ...
  combinators   — comp, partial, juxt, complement, constantly, ...
  bitwise       — bit-and, bit-or, bit-xor, bit-shift-left, ...
  spell         — spell-eval, reopen, wrap-cat, serialize-prefix, prune-and-reopen, serialize, stored
  concurrency   — future*
  error         — throw, ex-info, ex-data, ex-message, ex-cause, gensym

Use (!describe builtins :category) for full listing of any category.
Use (!describe builtins :fn-name) for individual function docs.
For namespace functions (io/, agents/, globals/, strings/, math/, patterns/), use (!describe <namespace>).

Common mistakes:

1. calling effect builtins outside the trailing expression: !llm-self, !ask-await, leaf-llm, eval, and describe-fn are effect functions; they must appear in the quoted trailing expression or inside !call-now / !peek / !print
2. confusing def with let: def binds in the environment (visible to later expressions); let creates local scope
3. forgetting quote on the trailing expression: the last expression must be quoted so the outer eval can run it with effect bindings
4. str vs cat vs pr-str: str joins arguments as strings; cat is an alias; pr-str serializes as Spell-readable data (vectors, maps, etc.)
5. using read-string on untrusted input: read-string parses Spell code; only use it on data you control
\end{lstlisting}

\paragraph{\texorpdfstring{\texttt{reminders} namespace guide}{reminders namespace guide}}\label{reminders-namespace-guide}

\begin{lstlisting}
REMINDER: This text belongs to the prefix of a Spell program that you are tasked with completing. Your entire response is code; embed all natural language within string literals. Follow the instructions on how to write correct Spell code in your system prompt.

For coding tasks (bug fixes, feature implementation, test-driven work), consult (!describe reminders :coding) on the first turn for a research-plan-implement-verify workflow focused on !peek, persist, short validation loops, and concise completion evidence.
\end{lstlisting}

\paragraph{\texorpdfstring{\texttt{reminders\ :coding}
		prompt}{reminders :coding prompt}}\label{app:language:reminders-coding-prompt}

\begin{lstlisting}
CODING TASKS — Research, plan, implement, verify, iterate.

Expect early verification failures. They are normal. Use them to refine your understanding, and continue until the actual task is complete.

RESEARCH before committing to a plan or implementation:
- Identify the relevant code, tests, configs, scripts, data files, and output locations.
- Treat the real environment as the source of truth. Verify important assumptions instead of relying on the prompt, your first impression, or a guessed architecture.
- Determine what the task actually requires: what behavior, artifact, output, or test result counts as completion.
- When errors, tracebacks, or failing commands point to exact files or line numbers, inspect those exact places first, then expand outward as needed.
- Use !peek-now for exploratory reads and disposable probes. Persist only the specific snippets, facts, or outputs you will need on later turns.
- Once you know the spec, the likely fix site, and the validation step, stop open-ended research and move to a patch attempt.

Examples:

Check dependencies and environment assumptions:
  '(!peek env-check
      (io/sh "which python3 && python3 --version && python3 -m pytest --version && which rg")
      pkg-check
      (io/sh "python3 - <<'PY'
import importlib.util
mods = ['pytest', 'numpy', 'pandas']
for name in mods:
    print(f'{name}:', bool(importlib.util.find_spec(name)))
PY"))
  ;; end of turn 1 completion
  (prune 2)
  ;; start of turn 2 suffix
  (think "Summary of peek output: python3 and pytest are available; rg is installed; numpy and pandas are importable.")
  '(!call-now source-hits
      (io/grep "def handle_request|class Handler" "src" {:include "*.py" :context 8 :max-count 20}))

Search for the real implementation site before editing:
  '(!peek def-hits
      (io/grep "def handle_request|class Handler" "src" {:include "*.py" :context 8 :max-count 20}))
  ;; end of turn 1 completion
  (prune 1)
  ;; start of turn 2 suffix
  (think "Summary of peek output: handle_request is defined in src/server.py and referenced from src/router.py.")
  '(!call-now impl-lines (io/read-lines "src/server.py" 201 240)
               router-lines (io/read-lines "src/router.py" 110 145))

Read exact ranges along an error trace:
  '(!peek verify
      (io/sh "cd /repo && python3 -m pytest tests/test_server.py::test_handles_empty_input -q"))
  ;; end of turn 1 completion
  (prune 1)
  ;; start of turn 2 suffix
  (persist err-summary
      "Summary of !peek output: AssertionError in test_handles_empty_input; expected empty list but got nil from handle_request.")
  '(!call-now test-lines   (io/read-lines "tests/test_server.py" 52 84)
               router-lines (io/read-lines "src/router.py" 110 145)
               impl-lines   (io/read-lines "src/server.py" 201 240))

Explore a large file ephemerally, then persist only the relevant subset:
  '(!peek file-lines (io/read-lines "src/server.py"))
  ;; end of turn 1 completion
  (prune 1)
  ;; start of turn 2 suffix
  (persist handler-block (subvec file-lines 200 240))
  '(!peek test-lines (io/read-lines "tests/test_server.py" 52 84))
  ;; end of turn 2 completion
  (prune 1)
  ;; start of turn 3 suffix

Use !peek for disposable file creation or one-off probes:
  '(!peek _
      (io/write-file "/tmp/check.py" verify-script)
      probe (io/sh "python3 /tmp/check.py"))
  ;; end of turn 1 completion
  (prune 2)
  ;; start of turn 2 suffix

Read the tests to find constraints not in the task description:
  '(!peek test-code (io/read-lines "tests/test_solution.py"))
  ;; end of turn 1 completion
  (prune 1)
  ;; start of turn 2 suffix
  (persist size-check (subvec test-code 10 16))
  (think "The test compresses output.bin with zlib and asserts the result is under 10000 bytes — I need a compact representation, not a raw dump.")

PLAN before acting:
- State what you think is going on, what parts of the system are relevant, and what you will do next.
- Identify the concrete files, commands, or artifacts involved.
- State how you will tell whether the task is complete.
- If multiple locations, layers, or output paths may matter, name them before proceeding.

Example:
  (think "Plan: inspect the parser and the failing test, update the parser behavior, then run the exact validation command and confirm the expected output/artifact.")

Your research must progress to the planning stage: gather needed context, persist what is relevant, then when you understand the existing logic, stop researching and plan.

IMPLEMENT:
- Make changes that are supported by the evidence gathered during research.
- Prefer structured io/ tools for reading and editing files.
- Use io/sh for running programs, tests, package managers, and shell utilities.
- Keep the feedback loop intact: when you need results for later reasoning, bind them with !call-now or inspect them with !peek-now.

VERIFY:
- Use the actual validation step that matches the task: exact test, exact command, exact output check, or exact artifact check.
- Use !peek-now for io/sh verification outputs, which may be verbose.
- After a failed verification, summarize what the failure means before moving on.

Example:
  '(!peek verify
      (io/sh "cd /repo && python3 -m pytest tests/test_server.py::test_handles_empty_input -q"))
  ;; end of turn 1 completion
  (prune 1)
  ;; start of turn 2 suffix
  (def err-summary "Summary of !peek output: AssertionError in test_handles_empty_input; expected empty list but got nil from handle_request.")
  '(!call-now impl-lines (io/read-lines "src/server.py" 201 240))

ITERATE:
- If verification fails, keep going. Read the failure, update your model of the task, and try again.
- If the evidence contradicts your first theory, replace the theory instead of defending it.
- Re-check your assumptions after each surprising result. Be open to the possibility that your previous reasoning, chosen file, inferred root cause, or validation method was wrong.
- If a command fails or the environment behaves unexpectedly, inspect the actual tools, files, paths, permissions, dependencies, and outputs before concluding anything.
- After a failed attempt that created resources (files, API objects, containers), reuse or clean up existing resources instead of creating duplicates.

Example:
  (think "My earlier assumption was wrong: the failure is not in src/router.py; the traceback and test output point to src/server.py, and pytest is using a different code path than my custom repro.")

COMPLETION:
- Return concise evidence for completion: what you ran or checked, what passed, and what observable result proves the task is done.
- Do not treat diagnosis, a plausible patch, or a partial check as completion.

Example:
  (think "Validation evidence: ran `python3 -m pytest tests/test_server.py::test_handles_empty_input -q` and it passed; output file `/app/out.json` now exists and contains the expected empty list.")
\end{lstlisting}

\paragraph{\texorpdfstring{\texttt{reminders\ :orchestrator}
		prompt}{reminders :orchestrator prompt}}\label{reminders-orchestrator-prompt}

\begin{lstlisting}
ORCHESTRATOR MODE - Use explicit orchestration, not just a basic single-agent loop.

This reminder is intentionally stronger than the default coding workflow. The goal is to induce real orchestration attempts. Unless the environment makes it impossible, do not solve the task as a plain single-agent research-plan-implement-verify loop. Act as an orchestrator coordinating bounded subproblems.

DEFAULT BEHAVIOR:
- Start by decomposing the task into at least two concrete subproblems or phases.
- Perform an explicit orchestration maneuver early: helper agents, parallel reads, or a summarize-and-reopen boundary.
- Even if the task looks local, still use one orchestration step before committing to the final patch.
- Treat yourself as the coordinator: gather results, compare them, decide what to trust, and integrate the final answer.

DECOMPOSE DELIBERATELY:
- Name the subproblems explicitly before editing.
- Prefer bounded delegation over open-ended "solve the whole task" prompts.
- Give helpers narrow scopes: exact files, exact question, exact stop condition.
- If there are multiple plausible root causes, investigate them in parallel instead of serially.
- If there are multiple likely fix sites, assign them as separate work items before patching.

ORCHESTRATOR RULES:
- Ask helpers for summaries, answers, diffs, or exact facts, not raw transcript dumps.
- Never assume helper success. Inspect what came back against the repository before building on it.
- If a helper fails, returns malformed code, or gets confused, salvage the validated facts and reassign or take over.
- Keep orchestrator state compact: plan, findings, decision. Compress aggressively once exploration has converged.
- Main-thread verification is still mandatory, but verification itself can be decomposed into bounded probes before the final check.

WHEN ORCHESTRATION MISFIRES:
- If workers duplicate effort, tighten ownership and resend narrower prompts.
- If workers are noisy or unreliable, replace them with direct reads or a same-agent compression step rather than carrying their whole output forward.
- If orchestration is clearly causing confusion, simplify one layer at a time instead of instantly abandoning the structure.
- Integrate partial useful results. Do not pretend earlier orchestration never happened.

Example: orchestrate broad exploration before patching.
  (think "I will treat this as two subproblems: find the exact behavioral contract, and find the implementation path that violates it.")
  '(!call-now contract
      (agents/!spawn-ask
        "Read the relevant tests/docs and return only: required behavior, exact verification command, and likely central files.")
      impl-map
      (agents/!spawn-ask
        "Read the likely implementation files and return only: probable fix site, adjacent coupled code, and obvious risks."))

Example: force one orchestration step even for a local-looking bug.
  (think "This bug looks local, but I still want an explicit orchestration step before patching.")
  '(!call-now test-view (io/read-lines "tests/test_solution.py" 1 80)
               impl-view (io/read-lines "src/solution.py" 100 180))

Example: recover by compressing after noisy delegation.
  (think "Worker outputs were noisy, but two facts survived: the serializer drops the flag and the renderer test is the right verifier.")
  (rethink "Plan: patch serializer.py directly, verify with tests/test_renderer.py::test_preserves_flag, then rerun the serializer-focused test if needed.")
  '(!llm-self "Continue from this compressed plan only. Implement the fix, verify it, and report concrete evidence.")
\end{lstlisting}

\paragraph{\texorpdfstring{\texttt{reminders\ :context-efficiency}
		prompt}{reminders :context-efficiency prompt}}\label{reminders-context-efficiency-prompt}

\begin{lstlisting}
CONTEXT EFFICIENCY — Minimize total context window usage.

Context tokens are your scarcest resource. Prune aggressively to stay effective over long tasks.

Prefer !peek-now over !call-now for disposable tool calls (auto-appends prune that removes the binding on the following extension while leaving the !peek-now call visible):
    '(!peek-now data (io/sh "find . -name '*.py'"))

On the subsequent turn, persist what you need before extending:
    ;; end of turn 1 completion
    (def data "... 200 lines ...")
    (prune 1)
    ;; start of turn 2 suffix
    ;; data is still in scope here
    (persist targets (take 5 (strings/split-lines data)))
    '(!extend)
    ;; next turn: data is pruned, the !peek-now call remains visible, and targets survive as literals

When running a shell script or Python program that you do not need to rerun, keep it inside !peek-now:
    '(!peek-now verify (io/sh "cd /repo && python - <<'PY'\nimport ...\nPY"))
    ;; end of turn 1 completion
    (prune 1)
    ;; start of turn 2 suffix
    (think "Verification passed: the fix handles both edge cases.")
    '(!extend)
    ;; next turn: the verify binding is gone; the !peek-now call remains visible

When you need to rerun a script later, write it to disk first and then call it with !call-now.

After extended reasoning, rethink to compress:
    (think "Long analysis of the bug... examining stack traces, testing hypotheses... the root cause is in parse_args line 42.")
    (rethink "The bug is in parse_args, line 42: off-by-one in the loop bound.")
    '(!extend)

When context grows large, compact:
    '(!compact)

Plan-clear pattern — reason and explore, then start fresh with a self-contained plan:
    (think "analyzing the problem..." ...)
    '(!peek-now files (io/ls "."))
    ;; end of turn 1 completion
    (def files [...])
    (def plan "Task: fix the calculator bug in calc.py\n1. Edit line 12: fix off-by-one\n2. Run tests")
    ;; start of turn 2 suffix
    '(!llm-self plan)
    ;; next turn has only the plan as prefix — maximum working space

Each extension should carry forward only what the next step needs.
\end{lstlisting}

\paragraph{\texorpdfstring{\texttt{io} namespace guide}{io namespace guide}}\label{io-namespace-guide}

\begin{lstlisting}
IO — File operations, read-only exploration, shell commands, process execution and file watching.

  (io/read-lines path)                      — read file as vector of line strings with first-line metadata
  (io/read-lines path start end)             — line range [start, end) Python-style half-open
  (io/read-file path)                        — read file as numbered-lines string
  (io/read-file path start end)
  (io/grep pattern path)           — recursive grep with line numbers (ERE by default; supports |, +, ?, {n,m}, and (...) groups)
  (io/grep pattern path {:ignore-case true :include "*.clj" :context 20 :max-count 50})
  ;; :context N returns N lines around each match — one call gets both the hit and the surrounding code
  (io/glob pattern)
  (io/glob pattern path {:type "f" :max-depth 5})
  (io/git "status")                         — run an allowlisted read-only git subcommand
  (io/str-replace path old new)              — replace string in file (must appear exactly once)
  (io/str-replace path old new {:all true})
  (io/replace-lines path start end content)  — deletes lines in half-open range [start, end)
  (io/replace-lines path start start content) — inserts content before line start
  (io/replace-lines path [[s e c] ...])      — multi-edit (line numbers refer to original file)
  (io/sh command)                            — execute shell command, returns {:exit :out :err}
  (io/sh command {:timeout 10})
  (io/sh-test command)                       — build a zero-arg shell-backed fix-loop test thunk
  (io/exec [cmd arg1 ...])                   — execute command directly (no shell)
  (io/watch-send path handle)                — watch directory, send events as message to handle

Functions identical to Clojure: slurp, spit, write-file, exists?, directory?, stat, delete, copy, move, mkdir, mkdirs, cwd, env, temp-file.
`ls` returns structured entries with `:name` and `:size`; directory names end with `/`.

Use (!describe io :fn-name) for detailed docs on any function.
All io/ calls are effect functions — quote them in the trailing expression.
Recommended functions: read-lines, grep, glob, replace-lines.

Common mistakes:

1. calling io/* outside the quoted trailing expression
2. forgetting !call-now when you need the result: '(io/read-file "x") evaluates but the result is lost
3. using io/sh for everything — use io/str-replace to patch files, io/read-file to read them, io/grep to search them
4. grep-then-read in two turns when one grep with :context N would suffice — prefer `(io/grep pat path {:context 20})` for "find + see context"

In examples, | marks cursor position in a completion.

Recommended usage pattern: Patch a file with io/str-replace.

Use io/str-replace when you know the exact text to change. It avoids shell escaping issues entirely.
  ...|'(!call-now _ (io/str-replace "/testbed/module.py"
        "        for t in diophantine(eq, param)}"
        "        for t in diophantine(eq, param, permute=permute)}"))

Recommended usage pattern: Write a script to a temp file and run it in one turn.

When you need to run a multi-line Python script, write it to a file first to avoid shell heredoc escaping errors.
  ...(def verify-script "import sys\nfrom pathlib import Path\n...")
  |'(!call-now _ (io/write-file "/tmp/verify.py" verify-script) result (io/sh "python /tmp/verify.py"))

Do NOT embed multi-line Python in io/sh heredocs — nested quoting between Spell strings, shell heredocs, and Python strings is fragile and causes SyntaxError.

Recommended usage pattern: Read and replace by line number.

1. Read the file to see current contents.
  ...|'(!call-now code (io/read-lines "main.py"))

2. Next turn: code is bound. Identify the line range, replace it.
  ...(def code ["def greet():" "    print('hello')" ...])
  |(think "Line 2 needs updating.")
  '(io/replace-lines "main.py" 2 3 "    print('goodbye')")

Recommended usage pattern: Explore multiple files and persist relevant snippets.

1. Peek full file with one-turn lifetime.
  ...|'(!peek-now file-lines (io/read-lines "main.py"))

2. Next turn: file-lines is available. Persist relevant snippets and peek another file.
  ...(def file-lines ["... many lines ..."])
  (rethink 2 "After persisting what you need, rethink 2 to drop the prior !peek-now call and binding.")
  |(persist fn-defn (subvec file-lines 99 111))
  '(!peek-now test-lines (io/read-lines "test_main.py"))

3. Next turn: fn-defn stays in context. The prior !peek-now call and file-lines were dropped by rethink 2, and test-lines is now available.
  ...
  (persist fn-defn ["def target_fn(...):" "    ..."])
  '(!peek-now test-lines (io/read-lines "test_main.py"))
  (def test-lines ["... many lines ..."])
  (rethink 2 "After persisting what you need, rethink 2 to drop the prior !peek-now call and binding.")
  |...

Recommended usage pattern: Grep through large files.

1. Read the file.
  ...|'(!call-now code (io/read-file "big-module.py"))

2. Next turn: file was too large and got truncated. Rethink to discard it, then grep for what you need.
  ...(def code "1: import os\n2: import sys\n...\n... [truncated, 58302 chars total]")
  |(rethink "File too large to scan inline. Grep for the target instead.")
  '(!call-now matches (io/grep "def handle_request" "big-module.py"))
\end{lstlisting}

\paragraph{\texorpdfstring{\texttt{io-read} namespace guide}{io-read namespace guide}}\label{io-read-namespace-guide}

\begin{lstlisting}
IO-READ — Read-only filesystem inspection, codebase exploration, and environment lookup.

  (io/slurp path)                  — read entire file as raw string
  (io/slurp-bytes path)            — read file as bytes
  (io/read-file path)              — read file as numbered lines
  (io/read-file path start end)    — read numbered line range [start, end)
  (io/read-lines path)             — read file as vector of raw lines
  (io/read-lines path start end)   — read raw line range [start, end)
  (io/grep pattern path)           — recursive grep with line numbers (ERE by default)
  (io/grep pattern path {:context 20 :ignore-case true :max-count 50})
                                    — :context N returns N lines around each match in one call
  (io/glob pattern)                — find files by name pattern
  (io/git "status")               — run an allowlisted read-only git subcommand
  (io/exists? path)                — check whether a path exists
  (io/directory? path)             — check whether a path is a directory
  (io/ls path)                     — list directory contents as [{:name :size} ...]
  (io/cwd)                         — print current working directory
  (io/stat path)                   — inspect file metadata
  (io/env) / (io/env "NAME")      — read env vars

Use io-read when a child should inspect the workspace without editing files or
running arbitrary commands. For process execution, add io-exec separately.
\end{lstlisting}

\paragraph{\texorpdfstring{\texttt{io-write} namespace guide}{io-write namespace guide}}\label{io-write-namespace-guide}

\begin{lstlisting}
IO-WRITE — Filesystem mutation and file editing.

  (io/write-file path content)               — write file contents
  (io/spit path content)                     — write or append file contents
  (io/str-replace path old new)              — replace a string in a file
  (io/str-replace path old new {:all true})  — replace all occurrences
  (io/replace-lines path start end content)  — replace a numbered line range
  (io/replace-lines path [[s e c] ...])      — apply multiple line edits atomically
  (io/mkdir path)                            — create one directory
  (io/mkdirs path)                           — create a directory tree
  (io/delete path)                           — delete a file or empty directory
  (io/copy src dest)                         — copy a file
  (io/move src dest)                         — move or rename a file
  (io/temp-file)                             — create a temp file

Use io-write when an agent should edit the workspace. Pair it with io-read
when the same agent also needs inspection helpers.
\end{lstlisting}

\paragraph{\texorpdfstring{\texttt{io-exec}}{io-exec}}\label{io-exec}

\begin{lstlisting}
IO-EXEC — Process execution helpers.

  (io/sh command)                  — execute a shell command
  (io/sh command {:timeout 10})    — execute with a timeout override
  (io/sh-test command)             — build a zero-arg shell-backed test thunk
  (io/exec [cmd arg1 ...])         — execute a command directly without a shell
  (io/watch-send path handle)      — watch a directory and send file events

Use io-exec when an agent needs command execution. Pair it with io-read for a
read-only exploration agent or with io-write for agents that both edit and run
commands.
\end{lstlisting}

\paragraph{\texorpdfstring{\texttt{agents} namespace guide}{agents namespace guide}}\label{agents-namespace-guide}

\begin{lstlisting}
AGENTS — Inter-agent communication (effect namespace).

  (agents/spawn prompt)         — start background agent using the current agent
  (agents/spawn prompt :handle-name) — same, with explicit handle name
  (agents/spawn agent prompt)   — start background agent with explicit compiled agent
  (agents/spawn agent prompt :handle-name) — explicit compiled agent + explicit handle name
  (agents/send target message)     — send message (usually a string) to target
  (agents/reply msg-map message)   — reply to msg-map, which must contain :from
  (agents/!ask target message)     — send message to target, block for reply
  (agents/!ask target)             — poke target without message, block for reply
  (agents/!ask [a b c])            — multi-target: poke all, wake when all complete
  (agents/!reply-ask msg-map message)   — reply to msg-map, block for next message
  (agents/!spawn-ask prompt) — spawn with the current agent, block until completion
  (agents/!spawn-ask prompt :handle-name) — same, with explicit handle name
  (agents/!spawn-ask agent prompt) — spawn with explicit compiled agent, block until completion
  (agents/!spawn-ask [[agent prompt] [agent prompt :name] ...]) — spawn many, wait for all completions (no ask wakeup poke)
  (agents/!spawn-ask [prompt-a prompt-b ...]) — spawn many with the current agent, wait for all completions (no ask wakeup poke)
  (agents/current-handle)          — your handle
  (agents/parent-handle)           — handle of agent that spawned you (nil if you are main)
  (agents/send-msg-fn f handle)    — low-level / not recommended

Use (!describe agents :fn-name) for detailed docs on any function.

Communication chooser:
  - Use (agents/send target value) when you know the target handle and just want
    to deliver a message. Common child->parent pattern:
      '(agents/send (agents/parent-handle) value)
  - Use (agents/reply msg-N value) only when you are replying to a received
    message binding such as msg-0. The first argument must be that msg map, not
    a raw string payload and not a handle.
  - Use (agents/!reply-ask msg-N value) when you are replying to msg-N and want
    to wait for the other side's next message.
  - Use (agents/!ask target value) when you want to send a request and block for
    the response.

Quick rule:
  - If you have a handle, use send or !ask.
  - If you have a received msg-N, use reply or !reply-ask.

Special handles:
  :main — the initial agent (entry point). Always present.
  :user — the human operator (only present in interactive terminal sessions).
  Check (globals/get :roles) to see if :user is available before asking.

Messages arrive as def bindings: (def msg-N {:from sender :body val}).
Reply using Spell code, not raw natural language.
Agents other than :main persist after returning; a later message can wake them for another turn.

Message preemption: if another agent sends you a message while your response
is in flight, the message is appended as an extension and your trailing
expression becomes inert. A (think "[preempted or awakened by msg-N]")
annotation precedes the message def. 'Preempted' means your trailing expression
did not fire; 'awakened' means you were sleeping and the message woke you.
You get a new turn with the incoming message in scope.
You may then re-run the trailing expression from your previous turn.

All agents/ calls are effect functions — quote them in the trailing expression.
Check (globals/get :roles) to discover available agents.

Common mistakes:

1. agents/send and passing turn when expecting a reply: this ends conversation, instead use agents/!ask
2. agents/reply and passing turn: same problem; use agents/!reply-ask if you need the conversation to continue
3. agents/!ask followed by additional expressions: these do not evaluate, instead put them first
4. hallucinating handles: use (agents/parent-handle), :user, :main, or look up (!print (globals/get :roles)) (if globals/ available)
5. calling agents/* outside the quoted trailing expression (for example: (def h (agents/current-handle))); effect calls must run in trailing expression code
6. agents/send argument order: it is (agents/send target message), consistent with (agents/!ask target message).
7. agents/reply needs two arguments: a received msg-N and a reply value. Wrong: '(agents/reply "hello"). Right: '(agents/reply msg-0 "hello")
8. spawned children often need send, not reply. If nobody messaged you yet this turn, you do not have a msg-N to reply to.

In examples, | marks cursor position in a completion. It is doc-only; do not type it into code.

Multi-part example:

1. Main: spawn a summarizer, keep working, then block with !ask.
  ;; turn 1: start child + continue your own CoT
  ...|'(do (agents/spawn
         "You are a summarizer. Read long-file.txt and send me a summary."
         :summarizer)
       (!extend))
  ;; next turn:
  ... |(think "...")(think "Ok, I'll wait for summarizer now")'(agents/!ask :summarizer)
  ;; main blocks until child responds

2. Summarizer child: use send to return result.
  ...(quine prompt "You are a summarizer. Read long-file.txt and send me a summary.")
  |'(!call-now file-contents (io/read-lines "long-file.txt"))
  ;; next turn
  ...(def file-contents "...")
  |(def summary "...")
  '(agents/send (agents/parent-handle) summary)
  ;; child turn ends after send

3. Main: use !reply-ask to clarify and keep the conversation open.
  ...'(agents/!ask :summarizer)
  (def msg-0 {:from :summarizer :body {...}})
  (think "I have a question about the summary.")
  |'(agents/!reply-ask msg-0 "What is the...")
  ;; child awakens; main blocks for child's response

\end{lstlisting}

\paragraph{\texorpdfstring{\texttt{globals} namespace guide}{globals namespace guide}}\label{globals-namespace-guide}

\begin{lstlisting}
GLOBALS — Shared state visible to all agents.

  (globals/get key)              — read a global by key
  (globals/set key value)        — write a global (returns the value)
  (globals/update key f)         — atomic read-modify-write (returns new value)
  (globals/pop key)              — atomic remove-and-return first element
  (globals/keys)                 — list all global keys
  (globals/all)                  — return entire globals map
  (globals/wait-until pred)      — block until pred on globals map is true

Use (!describe globals :fn-name) for detailed docs on any function.
All globals/ calls are effect functions — quote them in the trailing expression.

Common read patterns:
  1) Bind to a local with !call-now:
     '(!call-now roles (globals/get :roles))
     ;; next turn: roles is available as a local binding

  2) Print directly for quick inspection:
     '(!print (globals/get :roles))

Default special keys:
  :roles {}  — Agent registry for handle lookup.
               Convention: {:main "Orchestrator" :spawn-1 "Worker for CLI" :spawn-2 "Worker for unit testing"}
  :tasks []  — shared task queue.
               Convention: [{:id 1 :desc "read file"} {:id 2 :desc "summarize"}]

These defaults are conventions, not requirements, and are unpopulated by default. Agents may create additional keys.

Common mistakes:

1. calling globals/* outside the quoted trailing expression: (globals/get :roles) does nothing at eval time; must be quoted
2. forgetting !call-now: '(globals/get :roles) returns the value; use '(!call-now roles (globals/get :roles)) if you want to see it
3. hallucinating handles: instead, look them up in roles/ (also see agents/parent-handle and agents/current-handle)


Multi-part example — worker pool with a shared task queue:
| marks cursor position and is doc-only; do not type it into code.

1. Main: populate the queue and spawn workers.
  ...|'(do (globals/set :results [])
       (globals/set :tasks [{:id 1 :desc "summarize A"} {:id 2 :desc "summarize B"}])
       (agents/spawn "You are a worker. Pop tasks from globals :tasks and process them." :w1)
       (agents/spawn "You are a worker. Pop tasks from globals :tasks and process them." :w2)
       (globals/wait-until (fn [s] (= 2 (count (:results s))))))

2. Worker w1: claim a task atomically.
  ...|'(!call-now task (globals/pop :tasks))
  ;; next turn: task is {:id 1 :desc "summarize A"} (or nil if queue empty)

3. Worker w1: post result back.
  ...(def task {:id 1 :desc "summarize A"})
  |(def summary "A is about...")
  '(globals/update :results (fn [r] (conj (or r []) {:id 1 :summary summary})))
\end{lstlisting}

\paragraph{\texorpdfstring{\texttt{blocking} namespace guide}{blocking namespace guide}}\label{blocking-namespace-guide}

\begin{lstlisting}
BLOCKING — Future-only blocking primitives.

  (blocking/await fut)                 — await a Spell future token (future-only)
  (blocking/await-all [f1 f2 ...])     — await multiple Spell futures (future-only)
  (blocking/pmap f coll)               — parallel map with blocking join (future-only)
  (blocking/plet [a expr1 b expr2] body) — macro; parallel let with blocking/await
  (blocking/completion-promise handle) — await token for handle completion (future-only)
  (blocking/send-await handle msg)     — capture completion, send, await (future-only)

Use from inside (future ...) orchestration code.
\end{lstlisting}

\paragraph{\texorpdfstring{\texttt{patterns} namespace guide}{patterns namespace guide}}\label{patterns-namespace-guide}

\begin{lstlisting}
PATTERNS - Reusable orchestration patterns (effect namespace).

  (patterns/check-result prompt answer)  - verify answer with leaf-llm
  (patterns/clean-prompt raw-text)       - clean up messy text, then execute it
  (patterns/ralph opts)                  - future-based retry orchestrator
  (patterns/team goal-or-opts)           - planner + parallel worktree team orchestrator
  (patterns/fix-loop issue)              - test-driven code fixing loop (reflector + worker agents)
  (patterns/relay opts)                  - fresh-worker reasoning rounds with fresh verification

Use (!describe patterns :fn-name) for detailed docs on any function.

check-result: Verifies an answer using leaf-llm. Returns {:ok answer} or {:wrong msg}.
  (patterns/check-result "What is 2+2?" 4)            ;; => {:ok 4}
  (patterns/check-result "Capital of France?" "London") ;; => {:wrong "London is..."

clean-prompt: Cleans up a raw prompt (voice-to-text, quick notes) via leaf-llm, then runs it.
  '(patterns/clean-prompt "waht is the captal of franc... like the big city")
  leaf-llm infers intent and rewrites; !llm-self executes the cleaned prompt.
  Accepts a string or quine form (serializes non-strings automatically).

ralph: Retry orchestrator that runs blocking completion waits inside a future, so
the caller's agent trace stays responsive. Spawns a worker, sends task/retry
messages, waits via blocking/send-await, and sends
final {:pass result} or {:fail last-result} to the caller.
  '(!call-now started (patterns/ralph "fix failing tests"))
  ;; later receives msg with {:pass ...} or {:fail ...}

team: Multi-task implementation orchestrator. A planner decomposes the goal,
the scheduler executes dependency waves in parallel git worktrees, and a
verifier approves merges or resolves conflicts on the integration branch.
  '(!call-now result (patterns/team "Implement feature X"))
  Returns {:status :completed|:partial|:failed :tasks [...] :branch "spell-team-..."}

fix-loop: Test-driven code fixing loop. Registers a persistent reflector agent and
a persistent worker agent for the run. The root loop coordinates both via
blocking/send-await inside a future, and the caller waits via !ask-await:
reflector proposes diagnosis + test spec,
worker applies edits, and the loop retries until tests pass or retries are exhausted.
  '(!call-now result (patterns/fix-loop issue))
  Returns {:pass true} or {:fail "reason"}

relay: Reasoning relay with fresh context each round. Each round registers a new
worker, passes forward compressed prior reports, and if a worker claims :solved,
the pattern registers a fresh verifier to check the answer independently.
  '(!call-now result (patterns/relay problem))
  Returns {:solved true|false :answer any? :rounds [...]}

All patterns/ calls are effect functions - quote them in the trailing expression.

Common mistakes:

1. calling check-result outside the trailing expression: must be quoted like all effect calls
2. using team without an io-capable agent profile: workers and verifier need io/ and agents/; blocking/ is future-only and !ask-await is a builtin

In examples, | marks cursor position in a completion. It is doc-only; do not type it into code.

Example - verify then correct:

1. Compute an answer and check it.
  ...(def answer 42)
  |'(!call-now verdict (patterns/check-result "What is 6 * 9?" answer))

2. Next turn: handle the verdict.
  ...(def verdict {:wrong "6 * 9 = 54, not 42"})
  |(def answer 54)
  '(!call-now verdict (patterns/check-result "What is 6 * 9?" answer))
\end{lstlisting}

\paragraph{\texorpdfstring{\texttt{web} namespace guide}{web namespace guide}}\label{web-namespace-guide}

\begin{lstlisting}
WEB — Search and fetch web content.

  (web/search query)        — search web and return [{:title :url :snippet} ...]
  (web/fetch url)           — fetch URL and return markdown/text
  (web/config)              — inspect active web config

Recommended usage pattern: Search, then fetch the most relevant result.

1. Search and peek the results.
  ...|'(!peek-now results (web/search "clojure transducers"))

2. Next turn: results is available. Pick the best URL and fetch it.
  ...(def results {:ok [{:title "Transducers - Clojure" :url "https://clojure.org/reference/transducers" :snippet "..."} ...]})
  (rethink 2 "After persisting what you need, rethink 2 to drop the prior !peek-now call and binding.")
  |(persist best-url (get (first (:ok results)) :url))
  '(!peek-now page (web/fetch best-url))
\end{lstlisting}

\paragraph{\texorpdfstring{\texttt{react} namespace guide}{react namespace guide}}\label{react-namespace-guide}

\begin{lstlisting}
REACT - Hidden ReAct loop (effect namespace).

  (react/run prompt-or-opts) - run a plain-text command loop while hiding Spell from the inner model

Use react/run from an :init program whose trailing expression calls the
react namespace:

  (eval (do
    (def prompt "Inspect the repo and summarize the failing test.")
    '(react/run prompt)))

Map form:

  (eval (do
    '(react/run {:task "Inspect the repo and summarize the failing test."
                 :max-steps 20})))

react/run uses leaf-llm internally, and the inner model sees only a genuine plain-text
ReAct transcript: task text, prior thoughts/actions/observations, and the
required output contract Action: Command[...] or Action: Finish[...].

   Requires an agent profile that exposes react/ plus shell execution
   capability (via io/sh).
\end{lstlisting}

\subsection{Error recovery prompts}\label{app:language:error-recovery-prompts}

\Spell{} has two LLM-facing recovery prompt templates plus a tool-call retry reminder. In the templates below, \passthrough{\lstinline!<error message>!} is replaced with the cleaned reader or evaluation error.

\paragraph{Inert recovery prompt}\label{inert-recovery-prompt}

\begin{lstlisting}
The previous Spell program threw an error. Please recover from this error by writing a new program that fulfills the intent of the previous program while avoiding the error. The previous program is inert text; it will not be reevaluated, and none of its bindings are live unless you redefine them. Do not rely on bindings or data from the previous program remaining available in later prompts. If you need to keep any useful context, carry it forward now by emitting fresh summaries or literal values, or by repeating tool calls that recover the relevant files, definitions, tests, or other evidence. Reminder: Emit Spell code only, not prose. Error message: <error message>
\end{lstlisting}

\paragraph{Trailing-expression recovery
	prompt}\label{trailing-expression-recovery-prompt}

\begin{lstlisting}
The previous Spell program threw an error in its trailing expression. Please recover by continuing the same `(eval (do ...))` block with a new trailing expression that fulfills the original intent while avoiding the error. Earlier expressions in this same do block will be reevaluated first. Their bindings remain available if reevaluation still succeeds. The previous failing trailing expression is now inert earlier context because your new expression will be appended after it. Use the existing live bindings when helpful, but do not repeat the failing trailing expression unchanged. Reminder: Emit Spell code only, not prose. Error message: <error message>
\end{lstlisting}

\paragraph{Missing tool-call retry
	prompt}\label{missing-tool-call-retry-prompt}

For tool-call-based API providers, this prompt is injected by the \passthrough{\lstinline!call-llm!} function when retrying an API call after a previous attempt failed to produce a tool call.

\begin{lstlisting}
;; system: retrying because the previous response did not call the required spell_suffix tool.
;; Respond with exactly one spell_suffix tool call. Do not send assistant text, markdown, or a thinking-only response.
;; The raw Spell suffix must be in the spell_suffix tool input.
\end{lstlisting}

\subsubsection{Comparison With Codex CLI and Claude
	Code}\label{comparison-with-codex-cli-and-claude-code}

The amount of prompting in \Spell{} is comparable to contemporary coding-agent harnesses, although it is allocated differently. The \Spell{} system prompt variant contains about 3.0k words, and the coding prompt adds another 1.0k words when loaded. The open-source Codex CLI prompt surface is of the same order: model-specific prompt files contain 1.1-1.2k words, while the combined prompt with apply-patch instructions contains about 3.9k words before adding tool schemas. Claude Code does not publish a single stable system-prompt word count, but its documented prompt surface is similarly modular: project memory files, scoped rules, permission modes, plan mode, subagents, tool permissions, skills, and subagent memory are injected according to session state. The central comparison is therefore not whether \Spell{} is unusually verbose, but where the prompt budget is spent.

All three prescribe a workflow, either in the system prompt or through mode-specific instructions. \Spell{}'s coding reminder prescribes a research, plan, implement, verify, iterate loop: inspect the real environment, persist only evidence needed across turns, patch from that evidence, and treat failed verification as feedback. Codex CLI, similarly, provides guidance to plan nontrivial work, validate progressively, and continue until the task is resolved; Claude Code's plan-mode prompt provides guidance for codebase exploration before implementation.

All three also provide rules for coding-agent best practices. They overlap on core practices such as searching with fast repository tools, grounding changes in observed files and test output, keeping edits scoped, and reporting validation honestly. A typical shared rule is that the agent must not overwrite or revert unrelated user work. The non-shared rules reflect each harness's distinctive surface: \Spell{} warns about suffix-only output, the single trailing expression, recovery prompts, and unproductive \passthrough{\lstinline"!extend"} loops; Codex CLI emphasizes patch discipline, sandbox and approval behavior, and final-answer formatting; Claude Code emphasizes permission modes, \passthrough{\lstinline!CLAUDE.md!} loading, subagent permissions, memory, and plan-mode boundaries.

The most important difference is that \Spell{} prompts explain the semantics of \Spell{} and teach how it should and should not be used. A large fraction of the \Spell{} base prompt describes suffix completion, the completion wrapper, quines, double evaluation, namespace availability, self-calls, effect calls, materialized tool results, context pruning, and recovery. Codex CLI and Claude Code also devote substantial prompt and tool-schema budget to tool use, but those descriptions mostly tell the model how to call externally implemented tools and obey harness policy. \Spell{} has to teach a programming interface: the model's output is executable control flow that the runtime will parse and evaluate. This additional semantic instruction is the price of making orchestration part of the generated program rather than a fixed behavior of the surrounding harness.

\subsection{Transport}\label{app:language:transport}

\subsubsection{Transport variants}\label{transport-variants}

\Spell{}'s provider transport has a single semantic contract: each LM call receives a prefix of a \Spell{} program, and the visible model output must be exactly the suffix that completes it. The runtime parses and evaluates the completion. However, LM APIs may not expose full control over the prefix that is passed to an LM; instead, they often impose mandatory fencing to distinguish system prompts, user messages, model responses, and tool calls. \Spell{} supports three transport mechanisms.

First, prefill transport sends the \Spell{} prefix as assistant prefill or as a completion-style prompt prefix. The model is instructed via a user message (\passthrough{\lstinline!"Continue this Spell program."!}) to continue the prefix directly. In this transport mechanism, the program appears as one contiguous program in the context window of the LM, and this is the first-choice transport mode for providers that support it. The Fireworks completion provider supports prefill transport; in this paper, it was used for diagnostic open-weight transport comparisons rather than for the main open-weight benchmark results.

Second, message transport sends the \Spell{} prefix as the user message and asks for a raw assistant-message suffix. This transport is an option for APIs or model settings that do not support assistant prefill. Because some models echo all or part of the prefix when asked to continue text from a user message, \Spell{} strips an exact or whitespace-normalized prefix echo before appending the response. This differs from prefill transport because the prefix and suffix are separated within the model's context window by fencing mandated by the API.

Third, tool-call transport sends the \Spell{} prefix as ordinary user/input content but requires the response to arrive in a custom tool call named \passthrough{\lstinline!spell\_suffix!}. Similar to message transport, the prefix and suffix are separated in the model's context window by fencing. This paper uses tool-call transport for Anthropic, OpenAI, and Fireworks tool-call runs. Preliminary analyses did not show a difference between message and tool-call transport.

\subsection{Agent and provider
	configuration}\label{agent-and-provider-configuration}

\Spell{} separates LM configuration into two declarative EDN files. An \passthrough{\lstinline!.agent.edn!} file describes the \Spell{}-facing execution environment: which system prompt is used, which namespaces are visible, which subagents can be called through \passthrough{\lstinline!llms/!}, and which run-level defaults apply. A \passthrough{\lstinline!.provider.edn!} file describes the API-facing provider: which implementation is used, where credentials come from, which model and endpoint are selected, and how usage should be priced. The transport section describes how those providers map API calls onto \Spell{}'s prefix-suffix contract. The configuration layer records the chosen pairing between that transport and a \Spell{} language profile.

\subsubsection{Agent files}\label{agent-files}

Agent files live under \passthrough{\lstinline!config/agents/!} and are loaded by \passthrough{\lstinline!spell.agent/load-agent-spec!}. The three base agents correspond to the three transport-specific system prompts:

\begin{lstlisting}
;; config/agents/base-tc.agent.edn
{:name base-tc
 :doc "Base agent for tool-call providers"
 :system {:file "../prompts/sysprompt-toolcall.txt"}
 :llms []}
\end{lstlisting}

The base files intentionally expose no effect namespaces. Specialized agents inherit from a base and add capabilities. For example, \passthrough{\lstinline!io-tc.agent.edn!} pairs the tool-call system prompt with filesystem and shell access (\lstinline[language=spell]!io/!), inter-agent communication (\lstinline[language=spell]!agents/!), and shared-global state (\lstinline[language=spell]!globals/!):

\begin{lstlisting}
{:base base-tc.agent.edn
 :name io-tc
 :doc "Tool-call transport benchmark agent with io namespace (web disabled by default)"
 :llms {explore explore.agent.edn}

 :namespaces
 {io      stdlib/io
  agents  stdlib/agents
  globals stdlib/globals}}
\end{lstlisting}

Inheritance is file-based. \passthrough{\lstinline!:base!} is resolved relative to the child file, then the child definition is merged onto the parent. Scalar fields such as \passthrough{\lstinline!:name!}, \passthrough{\lstinline!:doc!}, \passthrough{\lstinline!:system!}, \passthrough{\lstinline!:model!}, \passthrough{\lstinline!:budget!}, \passthrough{\lstinline!:provider!}, \passthrough{\lstinline!:thinking!}, \passthrough{\lstinline!:reasoning-effort!}, \passthrough{\lstinline!:verbosity!}, \passthrough{\lstinline!:recover!}, \passthrough{\lstinline!:format!}, and retry or grammar settings are replaced by the child when present. \passthrough{\lstinline!:namespaces!} is merged as a map, so a child can add or override individual namespace entries. The \passthrough{\lstinline!:llms!} field is scalar for this purpose: a child either replaces it, omits it and gets auto-discovery when no inherited value is present, or sets \passthrough{\lstinline![]!} to opt out.

Namespace values are resolved by \passthrough{\lstinline!spell.agent/resolve-namespace-value!}. The common case is a \passthrough{\lstinline!stdlib/...!} symbol, but the loader also supports several extension forms:

\begin{lstlisting}
{:namespaces
 {io       stdlib/io                  ; stdlib namespace map
  io-lite  [stdlib/io-read stdlib/io-write] ; merge namespace maps
  helper   local_helpers.clj/helper-ns ; Clojure var from a file
  prompt   {:file "extra_prompt.txt"}  ; file contents as a string
  tools    {:file "tools.clj"
            :items {foo foo-ns
                    bar bar-ns}}
  worker   worker.agent.edn}}          ; compiled subagent function
\end{lstlisting}

After inheritance, \passthrough{\lstinline!compile-agent-spec!} resolves the system prompt, provider, namespaces, and \passthrough{\lstinline!llms/!} namespace, validates declared pattern dependencies, and passes the resulting data to \passthrough{\lstinline!spell.llm/compile-agent!}. The compiled value is a spawn function. It can be invoked as the root agent or passed into \passthrough{\lstinline!agents/spawn!} / \passthrough{\lstinline"agents/!spawn-ask"} for subagent execution.

The \passthrough{\lstinline!:llms!} field is a convenience mechanism for named subagents. A map such as \passthrough{\lstinline!\{:llms \{explore explore.agent.edn\}\}!} exposes \passthrough{\lstinline!(llms/explore prompt)!} inside the agent. Each entry compiles a normal agent spec, inheriting the parent provider and model unless the subagent spec overrides them. The loader uses lazy atoms while building this namespace, so mutually referential agent profiles can be compiled without ordering problems.

\subsubsection{Capability boundaries}\label{app:language:capability-boundaries}

SPE is not inherently more dangerous than other ways of letting a language model act through tools: a self-programmed agent can only exercise the capabilities exposed by its runtime. However, if dangerous capabilities are exposed, SPE can exacerbate misuse or accident risk because the user delegates more of the agent's internal state management to the model-written program. In particular, the model can decide what to retain, prune, summarize, or reveal in its own context window. This flexibility is useful for long-running work, but it also means that the user has less direct control over the exact internal context from which later actions are chosen.

\Spell{}'s primary capability boundary is the namespace surface. Pure core namespaces are always available, while effect namespaces are opt-in through the agent configuration. For example, a read-only exploration agent can be given \lstinline[language=spell]!io-read/! without \lstinline[language=spell]!io-write/! or \lstinline[language=spell]!io-exec/!, while an editing agent can be given write helpers only when mutation is intended. This design makes capabilities inspectable in the same place as model, provider, prompt, and subagent configuration.

\subsubsection{Provider files}\label{provider-files}

Provider files live under \passthrough{\lstinline!config/providers/!} and are loaded by \passthrough{\lstinline!spell.provider/load-provider!}. They instantiate an implementation of the \passthrough{\lstinline!LLMProvider!} protocol. A typical OpenAI tool-call provider is:

\begin{lstlisting}
{:type :openai
 :api-key-env "OPENAI_API_KEY"
 :model "gpt-5.4"
 :use-responses-api true
 :force-tool-call true
 :max-tokens 32768
 :default-agent "../agents/base-tc.agent.edn"}
\end{lstlisting}

The provider \passthrough{\lstinline!:type!} selects the implementation. The supported file-backed types in the v0.1.0 reference implementation are \passthrough{\lstinline!:anthropic-pf!}, \passthrough{\lstinline!:anthropic-tc!}, \passthrough{\lstinline!:openai!}, \passthrough{\lstinline!:codex-tc!}, \passthrough{\lstinline!:fireworks!}, \passthrough{\lstinline!:ollama!}, and \passthrough{\lstinline!:test!} (described below). These types imply the low-level API adaptation described in Section~\ref{app:language:transport}; the remaining keys are implementation options. Common options include \passthrough{\lstinline!:api-key-env!}, \passthrough{\lstinline!:base-url!}, \passthrough{\lstinline!:model!}, \passthrough{\lstinline!:max-tokens!}, \passthrough{\lstinline!:costs!}, \passthrough{\lstinline!:cache-read-ratio!}, \passthrough{\lstinline!:prompt-cache-key!}, and \passthrough{\lstinline!:request-timeout-sec!}. Provider-specific options include \passthrough{\lstinline!:use-responses-api!} and \passthrough{\lstinline!:force-tool-call!} for OpenAI Responses custom-tool mode, \passthrough{\lstinline!:auth-file!} and \passthrough{\lstinline!:account-id!} for the ChatGPT-backed Codex provider, \passthrough{\lstinline!:chat-template!} and \passthrough{\lstinline!:convert-think?!} for Fireworks-style completion models, and \passthrough{\lstinline!:responses!}, \passthrough{\lstinline!:response!}, \passthrough{\lstinline!:response-rules!}, and \passthrough{\lstinline!:prefill?!} for the test provider.

The \passthrough{\lstinline!:default-agent!} key is metadata rather than a constructor argument. It is read by \passthrough{\lstinline!provider-edn-default-agent!} so wrappers around \passthrough{\lstinline!spell.api/run!} can choose a base agent whose system prompt matches the provider transport. The provider instance itself only knows how to call the model; the agent file supplies the \Spell{} prompt and namespace surface.

Provider specs can also appear inline inside an agent file:

\begin{lstlisting}
{:base base-tc.agent.edn
 :name local-openai-toolcall
 :provider {:file "../providers/openai-tc.provider.edn"}
 :namespaces {io stdlib/io}}
\end{lstlisting}

or, equivalently, as a complete inline map with the same keys as a \passthrough{\lstinline!.provider.edn!} file. Programmatic callers may also pass an already-created provider object to \passthrough{\lstinline!spell.api/run!}; in that case the object is injected into the loaded agent spec before compilation.

\subsubsection{Testing and interactive
	providers}\label{testing-and-interactive-providers}

The \passthrough{\lstinline!:test!} provider is the provider used by \Spell{}'s unit tests and by lightweight smoke runs. It implements the same \passthrough{\lstinline!LLMProvider!} protocol as real model transports, but replaces the external LM call with a deterministic response lookup. A test provider can be constructed directly:

\begin{lstlisting}
(provider/test-provider
  {:responses {"(do " "42)"}
   :response-rules [{:includes ["classify" "sentiment"]
                     :excludes ["retry"]
                     :response "{:label :positive})"}]
   :prefill? false})
\end{lstlisting}

or through a provider config:

\begin{lstlisting}
{:type :test
 :response-rules [{:includes ["hello"]
                   :response "\"world\""}]
 :prefill? true}
\end{lstlisting}

Lookup proceeds in a fixed order. First, \passthrough{\lstinline!:responses!} is checked for an exact prompt-string match. Second, a programmatic \passthrough{\lstinline!:response-fn!}, when the provider is constructed directly, may return a response for prompts that contain nondeterministic fragments such as generated handles; a static \passthrough{\lstinline!:response!} uses this same slot as a catch-all response. Third, if no exact or function response is available, \passthrough{\lstinline!:response-rules!} are scanned in order. Each rule requires all \passthrough{\lstinline!:includes!} strings to be present and all optional \passthrough{\lstinline!:excludes!} strings to be absent. Entries may be strings or maps of the form \passthrough{\lstinline!\{:response string :latency ms\}!}, allowing tests to exercise asynchronous behavior without calling a network service. If no entry matches, the provider throws an exception containing the full prompt, which makes exact prompt fixtures easy to create by copying the observed prompt into \passthrough{\lstinline!:responses!}.

The \passthrough{\lstinline!:prefill?!} flag controls whether \passthrough{\lstinline!supports-prefill!} returns true. This is important because \Spell{} has distinct prompt assembly paths for prefill, message-style, and tool-call providers. The test provider therefore tests the compiled agent and prompt-as-prefix machinery, not only pure evaluator functions. In practice, most tests use helpers such as \passthrough{\lstinline!make-test-agent!}, which compile a normal agent spec with a test provider injected, then run through \passthrough{\lstinline"!llm-self"} just as a real model-backed agent would.

\Spell{} also has an interactive user transport for debugging and demonstrations. Selecting model \passthrough{\lstinline!user!} in the CLI constructs \passthrough{\lstinline!spell.provider/user-provider!} instead of a network provider. This provider prints the system prompt, user message, and any prefill prefix to stderr, then reads the completion from stdin. The human can therefore write the \Spell{} continuation manually, exactly where the LM response would have appeared. This is distinct from the runtime \passthrough{\lstinline!:user!} agent handle. In interactive terminal sessions, \passthrough{\lstinline!spell.api/run!} can also register \passthrough{\lstinline!:user!} as an agent, so model agents may send messages to the human with \passthrough{\lstinline!agents/send!} or \passthrough{\lstinline"agents/!ask"}. The former path makes the human stand in for the LM; the latter makes the human one participant in the agent runtime.

\subsubsection{Resolution at runtime}\label{resolution-at-runtime}

\passthrough{\lstinline!spell.api/run!} is the public runtime entry
point. It requires an \passthrough{\lstinline!:agent!} path and exactly
one of \passthrough{\lstinline!:prompt!} or
\passthrough{\lstinline!:init!}. The provider can come from the agent
file or from the call itself:

\begin{lstlisting}
(spell.api/run
  {:agent "config/agents/io-tc.agent.edn"
   :provider (spell.provider/openai-provider {:model "gpt-5.4"})
   :prompt "Find and fix the failing test."})
\end{lstlisting}

If a provider is supplied to \passthrough{\lstinline!spell.api/run!}, it overrides the provider from the agent file before compilation. The CLI defaults to \passthrough{\lstinline!config/agents/cli.agent.edn!} and constructs a provider from its \passthrough{\lstinline!-m!}/\passthrough{\lstinline!--model!} flags. Benchmark wrappers can infer an agent from a provider file's \passthrough{\lstinline!:default-agent!}. The runtime core itself has no hidden global provider or agent default.

This split is the main design point of the config layer. Agent files answer ``what kind of \Spell{} agent is this?'' Provider files answer ``how does this agent call an LM?'' Testing and interactive modes preserve that split by swapping only the provider side of the boundary.

\subsection{\Spell{} runtime}\label{app:language:spell-runtime}

This section describes the implementation of the \Spell{} runtime.

\subsubsection{\texorpdfstring{\texttt{spell-eval}}{spell-eval}}\label{spell-eval}

\passthrough{\lstinline!spell-eval!} is the core evaluator. It takes an
expression and an environment and returns a result map:

\begin{lstlisting}
{:ok value :env env'}
{:err msg :env env :expr expr :trace [...]}
\end{lstlisting}

\passthrough{\lstinline!spell-eval!} is common across agents.

\subsubsection{\texorpdfstring{\texttt{eval}}{eval}}\label{eval}

Each agent gets its own \passthrough{\lstinline!eval!} builtin. This builtin provides the outer, effectful evaluation step used by the completion wrapper. It merges pure builtins with the effectful namespace set that the agent is allowed to access. In this way, all agents share the same core evaluator but differ in the effect surface exposed at the trailing-expression boundary.

\subsubsection{\texorpdfstring{\texttt{box} and the inside function}{box and the inside function}}\label{box-and-the-inside-function}

\passthrough{\lstinline!box!} is the execution wrapper that processes
completions. It is invoked by any self-call. It takes three inputs: an
agent handle, a raw completion source (often a
\passthrough{\lstinline!future!}, but also a promise or raw string),
and an inside function that acts on the completion. In particular, the inside function
might do one of three things:
\begin{itemize}
	\tightlist
	\item evaluate the completion (the ``awake function'');
	\item await some signal, then evaluate the completion (the ``asleep function'');
	\item evaluate a root completion, deliver its result, and install a sleeping continuation for the next wakeup cycle (the ``root function'').
\end{itemize}

This allows a \Spell{} program to request a wakeup signal, sleep until that signal is delivered, install a new program which might depend on the signal or some other global state (for example, an incoming message), and resume execution of the modified program. The waiting step is dictated by the model's program, but the actual waiting occurs within \passthrough{\lstinline!box!} instead of \passthrough{\lstinline!spell-eval!}. If the waiting were to occur within \passthrough{\lstinline!spell-eval!}, then it would necessitate some mechanism to replace the program currently being evaluated with a new, global state-dependent program mid-evaluation.

\subsubsection{\texorpdfstring{\texttt{call-llm}}{call-llm}}\label{call-llm}

\passthrough{\lstinline!call-llm!} makes the actual LM API call. It is
responsible for model routing, token usage tracking, system prompt
injection, API error handling and retry logic, prompt caching
configuration, and more generally any logic that is API-provider
specific. It is configured by a \passthrough{\lstinline!provider.edn!}
file.

\subsubsection{End-to-end flow}\label{end-to-end-flow}

A typical root execution looks like this:

\begin{enumerate}
	\def\labelenumi{\arabic{enumi}.}
	\tightlist
	\item
	      For each agent, create an \passthrough{\lstinline!eval!} function and
	      install it within agent-specific inside functions.
	\item
	      Construct an initial program from a user prompt.
	\item
	      For the root inside function of the main agent, run
	      \passthrough{\lstinline!(box :main init-program root-inside-fn)!}.
	\item
	      All subsequent execution occurs inside of this function call; for
	      example, the initial program usually makes a self-call, which triggers
	      the creation of a new \passthrough{\lstinline!box!}.
\end{enumerate}


\clearpage
\section{Benchmarking methods and results}
\label{app:bench}

\subsection{Shared evaluation configuration}\label{app:bench:config}

\subsubsection{Compared Agents}\label{compared-agents}

\paragraph{\Spell{} agent.} The \Spell{} agent was configured with the tool-call transport agent profile \texttt{config/agents/io-tc.agent.edn}. This profile exposes \texttt{io/} for filesystem and shell operations, \texttt{agents/} for background agents and inter-agent communication, and \texttt{globals/} for shared state. The benchmark profile intentionally omits the \texttt{web/} namespace, so \Spell{} did not have live web access in the retained benchmark analyses. The core namespaces \texttt{strings/}, \texttt{math/}, \texttt{builtins/}, and \texttt{reminders/} are always available as documented language functions and prompt reminders.

\Spell{} runs are initialized by running the following \Spell{} program:

\begin{verbatim}
(quine completion
  (eval (do
    (quine prompt "<benchmark prompt>")
    '(!extend))))
\end{verbatim}

The trailing expression \texttt{\textquotesingle{}(!extend)} produces a self-call with the initial program (minus trailing parentheses) as its prefix. An exception is that in a secondary Terminal-Bench analysis, a coding-task prompt (Appendix~\ref{app:language:reminders-coding-prompt}) was given via the trailing expression \texttt{\textquotesingle{}(!describe reminders :coding)}. This expression likewise produces a self-call, appending a prompt as a string literal to the prefix.

\paragraph{Codex CLI.} Codex CLI v0.120.0 was invoked as a baseline harness with the same underlying OpenAI model where applicable. The benchmark adapter passed the benchmark prompt directly to \texttt{codex\ exec}, set the requested model, disabled live web search with \texttt{-c\ web\_search="disabled"}, emitted JSON logs with \texttt{-\/-json}, skipped repository checks with \texttt{-\/-skip-git-repo-check}, used \texttt{-\/-sandbox\ danger-full-access}, and set \texttt{model\_reasoning\_effort} when the run specified low, medium, or high effort. It received no additional benchmark-specific or user-specific prompt beyond the task prompt.

\paragraph{mini-swe-agent.} mini-swe-agent v2.4.6 was invoked as a second baseline harness with the same underlying OpenAI model. The benchmark adapter passed the benchmark prompt to the \texttt{mini} CLI, selected \texttt{openai/gpt-5.4} through mini-swe-agent's LiteLLM Responses API model class, disabled interactive confirmations and post-run prompts, recorded JSON trajectories, and set \texttt{reasoning.effort="medium"}. It received no additional benchmark-specific or user-specific prompt beyond the task prompt.

\paragraph{Claude Code.} Claude Code v2.1.107 was invoked with the benchmark prompt through \texttt{claude\ -p}, stream JSON output, verbose mode, an explicit allowed-tools list, the Opus 4.6 model, medium effort, and a run budget. It received no additional benchmark-specific or user-specific prompt beyond the task prompt.

\subsubsection{Transport, Caching, and Effort}\label{transport-caching-budgets-and-effort}

\paragraph{GPT-5.4.} When configured with GPT-5.4, the \Spell{} agent used tool-call transport and the OpenAI Responses API. The API request requires the model to emit a \texttt{spell\_suffix} tool call whose input is the raw \Spell{} suffix (see Appendix~\ref{app:language:system-prompts}). OpenAI prompt caching used a stable \texttt{prompt\_cache\_key} after the first self-call, so repeated calls from one \Spell{} run shared a cache partition.

\paragraph{Claude Opus 4.6.} When configured with Claude Opus 4.6, the \Spell{} agent used tool-call transport and the Anthropic Messages API. The API request likewise requires the model to emit a \texttt{spell\_suffix} tool use whose input contains the raw \Spell{} suffix. Anthropic prompt caching used ephemeral \texttt{cache\_control} markers on cacheable system text and on shared user-prompt prefixes after the first self-call. Medium effort was used, equating to a 10,000-token thinking budget.

\paragraph{Open-weight models.} GLM-5.1, Kimi-K2.6, and Qwen3.7 Plus were run through Fireworks serverless inference using tool-call transport. The API request requires the model to emit a \texttt{spell\_suffix} tool call whose input is the raw \Spell{} suffix, matching the transport pattern used for the closed-source tool-call runs. GLM-5.1 and Qwen3.7 Plus used \texttt{reasoning\_effort="high"}. Kimi-K2.6 was run without a reasoning parameter because the deployment used does not expose one. Earlier completion-prefill diagnostics used the then-current Qwen3.6 Plus endpoint (Appendix~\ref{app:bench:obs1}).

\subsection{Orchestration games}\label{app:bench:orchestration-games}

I created three ``orchestration games'' mocking plausibly-useful orchestration patterns. The three game prompts were:

\begin{lstlisting}[basicstyle=\ttfamily\footnotesize,breaklines=true]
Auction:
Return a sealed-bid auction report for a painting. Ask three independent !llm-self child bidders for one integer bid from 100 to 1000. Then return all bids and the winner or tie.

Telephone:
Return an 8-step telephone relay report. Start with: "The museum closes at five because the winter storm is approaching." Write a deterministic game loop over relay numbers 1 through 8. At iteration k, pass the message from relay k-1 to a fresh !llm-self relay k with a prompt asking it to rewrite the message while preserving meaning. Bind the returned wording as the message for the next iteration. When the loop completes, return the initial wording, each relay wording, and the final wording. Do not rewrite relay messages yourself and do not use agents/.

Twenty questions:
Play an 8-turn yes/no guessing game. The host should choose one ordinary animal as the hidden answer and keep it out of all worker prompts. On each turn, call a fresh !llm-self worker with only the public transcript and a prompt asking for one yes/no question or GUESS: animal. The host answers truthfully from the hidden answer, appends Worker: and Host: lines to the public transcript, and passes only that public transcript to the next fresh worker. Stop after a correct guess or after 8 worker turns. Return the hidden answer and the full transcript. Do not use agents/ or leaf-llm.
\end{lstlisting}

\begin{longtable}[]{@{}
	>{\raggedright\arraybackslash}p{(\linewidth - 6\tabcolsep) * \real{0.1800}}
	>{\raggedright\arraybackslash}p{(\linewidth - 6\tabcolsep) * \real{0.4600}}
	>{\raggedright\arraybackslash}p{(\linewidth - 6\tabcolsep) * \real{0.1600}}
	>{\raggedright\arraybackslash}p{(\linewidth - 6\tabcolsep) * \real{0.2000}}@{}}
	\caption{Prompt-only orchestration-game results with GPT-5.4 medium and tool-call transport. Completed reports counts instances in which the \Spell{} agent returned a plausible report; the trace audit, performed by GPT-5.5, scored whether the generated \Spell{} programs implemented the intended orchestration pattern and all players followed the game rules.}\label{tab:orchestration-games-gpt54} \\
	\toprule\noalign{}
	Game               & Intended orchestration                                       & Completed reports & Trace audit                                                                                                                                                                                                                                                                                                         \\
	\midrule\noalign{}
	\endhead
	\bottomrule\noalign{}
	\endlastfoot
	Sealed-bid auction & Three independent \texttt{!llm-self} bidders                 & 8/8               & 8/8 pass                                                                                                                                                                                                                                                                                                            \\
	Telephone relay    & Deterministic relay loop over fresh \texttt{!llm-self} calls & 7/8               & 4/8 approximate pass                                                                                                                                                                                                                                                                                                \\
	Twenty questions   & Hidden-answer host with fresh public-transcript workers      & 7/8               & 7/8 pass                                                                                                                                                                                                                                                                                                            \\
\end{longtable}

In the auction task, the orchestrator made independent self-calls for three bidders and collected them back before declaring a winner on a subsequent turn. A representative fragment is:

\begin{lstlisting}[language=spell]
'(!call-now bid1 (!llm-self (wrap-cat bidder1-prompt))
            bid2 (!llm-self (wrap-cat bidder2-prompt))
            bid3 (!llm-self (wrap-cat bidder3-prompt)))
\end{lstlisting}

On the telephone task, four of eight trials approximately used the intended deterministic loop pattern; one of these was weaker because it performed an initial setup relay outside the loop before looping over the remaining relays. The relevant fragment of a representative clean example was:

\begin{lstlisting}[language=spell]
(defn relay-prompt [k wording]
  (str "You are telephone relay " k " of 8.\n"
       "Rewrite the message while preserving meaning exactly.\n"
       "Return exactly one Spell string literal containing only the rewritten message.\n"
       "Do not add commentary, markdown, or any extra text.\n"
       "Message: " wording))

(defn relay-program [k wording]
  (wrap-cat
    (str "(quine prompt " (pr-str (relay-prompt k wording)) ")")))

'(let [relays
       (loop [k 1
              current initial-wording
              acc []]
         (if (> k 8)
           acc
           (let [next-wording (!llm-self (relay-program k current))]
             (recur (+ k 1)
                    next-wording
                    (conj acc {:relay k :wording next-wording})))))
       final-wording (if (empty? relays) initial-wording (:wording (last relays)))]
   {:initial initial-wording
    :relays relays
    :final final-wording})
\end{lstlisting}

The twenty-questions task tested whether the model could separate private host state from public worker context. The prompt did not name the hidden animal; successful programs chose an ordinary animal internally, sent only the public transcript to each fresh worker call, answered the worker truthfully as the host, and stopped after a correct guess or eight turns. A representative beginning was:

\begin{lstlisting}[language=spell]
(def hidden-answer "cat")
(def public-transcript "")
(def max-turns 8)
(quine worker-rules "You are a fresh worker in an 8-turn yes/no animal guessing game. You know only the public transcript provided as the next expression. Output exactly one Spell string literal and nothing else. The string content must be either one yes/no question about the hidden ordinary animal or GUESS: animal.")
'(!call-now worker-1 (!llm-self (wrap-cat worker-rules public-transcript)))
\end{lstlisting}

\subsubsection{Asynchronous multi-agent variant}\label{app:bench:orchestration-games-async}

The games above use synchronous self-call orchestration. In a second variant, the agent instead uses asynchronous multi-agent orchestration: its programs spawn concurrent child agents and coordinate them through message passing in the \texttt{agents/} namespace (Appendix~\ref{app:language:concurrent-agents-and-inter-agent-communication}). The asynchronous games used \Spell{} commit \texttt{6ad7d40} with a correction to an API-handling bug in v0.1.0, ensuring that the program prefix is included in the user message when the model API does not support assistant prefill. The three asynchronous game prompts were:

\begin{lstlisting}[basicstyle=\ttfamily\footnotesize,breaklines=true]
Auction:
Run a sealed-bid auction for a painting. You must orchestrate this with subagents, not local simulation. Requirements: Spawn exactly 3 bidder agents with handles :bidder-a, :bidder-b, :bidder-c. Each bidder independently chooses one integer bid (for example 100-1000). Each bidder sends its bid to the parent via agents/send. Main agent must collect all three bids via agents/!ask on the bidder handles, then announce the winner. Do not fabricate bids in the main agent; bids are valid only if returned by spawned bidders. Bidder agents must not communicate with each other.

Telephone:
Play a game of telephone. Create an initial message which gets sent to agent 1, then 2, then 3, then back to you. Each agent should receive the message and should be told to pass it on to the next agent, but re-worded. Your goal is to orchestrate this system such that the message goes around successfully. Report what the initial and final messages were.

Twenty questions:
Play '20 questions'. You are the orchestrator who picks a secret animal. Create a WORKER who does NOT know the secret and must guess by asking questions. Start by telling it that the answer is an animal and wait for its first guess. Then, communicate using !ask or !reply-ask. Loop until the worker guesses correctly or runs out of guesses. Use a guess limit of 8, not 20.
\end{lstlisting}

The auction and twenty-questions prompts name the \texttt{agents/} API explicitly, whereas the telephone prompt requests a three-agent relay without naming it; all telephone trials of both models nevertheless implemented the relay with \texttt{agents/} operations rather than synchronous self-calls. GPT-5.4 and GPT-5.6 Sol were each evaluated at medium reasoning effort with eight replicates per game. Scoring was performed by GPT-5.6 Sol as follows: full credit means correct orchestration with no \Spell{} errors and with all participants following the game rules; partial credit means valid orchestration completed after recovered errors or participant deviations; no credit otherwise. Table~\ref{tab:orchestration-games-async} summarizes the results.

\begin{longtable}[]{@{}llrrrr@{}}
	\caption{Asynchronous multi-agent orchestration-game results at medium reasoning effort, eight replicates per game. Valid orchestration counts full- or partial-credit trials. All 24 GPT-5.6 Sol trials invoked \texttt{agents/} operations; eight GPT-5.4 trials were terminated at the 900-second time limit.}\label{tab:orchestration-games-async} \\
	\toprule\noalign{}
	Model       & Game               & Full & Partial & None & Valid orchestration \\
	\midrule\noalign{}
	\endhead
	\bottomrule\noalign{}
	\endlastfoot
	GPT-5.4     & Sealed-bid auction & 1    & 2       & 5    & 3/8                 \\
	GPT-5.4     & Telephone relay    & 2    & 3       & 3    & 5/8                 \\
	GPT-5.4     & Twenty questions   & 0    & 1       & 7    & 1/8                 \\
	GPT-5.4     & Total              & 3    & 6       & 15   & 9/24                \\
	\midrule\noalign{}
	GPT-5.6 Sol & Sealed-bid auction & 3    & 5       & 0    & 8/8                 \\
	GPT-5.6 Sol & Telephone relay    & 3    & 3       & 2    & 6/8                 \\
	GPT-5.6 Sol & Twenty questions   & 3    & 5       & 0    & 8/8                 \\
	GPT-5.6 Sol & Total              & 9    & 13      & 2    & 22/24               \\
\end{longtable}

Eight GPT-5.4 trials were terminated at the 900-second time limit; one had completed a valid relay before termination and received partial credit. GPT-5.4 failure modes included malformed sends and replies, invalid or duplicate agent handles, nil-message wakeups, unproductive recovery loops, fabricating auction bids with \texttt{!llm-self} instead of collecting them from spawned bidders, and one deadlock involving \texttt{globals/wait-until} (this function is outside of the \texttt{agents/} system, which guarantees non-deadlock; Appendix~\ref{why-this-avoids-deadlock}).

GPT-5.6 Sol's partial-credit trials completed valid orchestration after recoverable protocol errors, most often malformed child sends, invalid agent handles (such as \texttt{:orchestrator} or \texttt{:parent}), or replaced messages. Its two no-credit trials were both telephone runs: one exhausted the recovery limit, and one drifted beyond the required three-agent chain and reached the recursion limit.

Trials ran locally in two batches of 12 per model, with a recursion-depth limit of 80 and an intended 900-second wall-clock cap. The cap was ineffective for the first 12 GPT-5.6 Sol trials; the no-credit telephone trial that reached the recursion limit ran for 2,226 seconds. The cap was enforced for the remaining GPT-5.6 Sol trials and all GPT-5.4 trials. GPT-5.6 Sol recorded 609 model calls costing \texttt{\$9.89}. The 16 GPT-5.4 trials with saved results recorded 286 calls costing \texttt{\$7.74}; usage was unavailable for the eight terminated trials.

A representative partial-credit recovery comes from a twenty-questions trial: on its first turn, the worker addressed its host by a hallucinated handle, then corrected itself on the first recovery attempt (comments are added):

\begin{lstlisting}[language=spell,basicstyle=\ttfamily\footnotesize,breaklines=true]
'(agents/!ask :orchestrator "Is it a mammal?")
; error: agents/!ask: Handle not registered {:handle :orchestrator}
(think "Recovery: the game controller's registered handle is :main, as shown by the incoming message. Ask exactly one first-turn question there.")
'(agents/!ask :main "Is it a mammal?")
\end{lstlisting}

The trial then ran without further errors to a correct guess on the final turn; under the rubric, the recovered error excludes full credit.

A representative full-credit auction implementation separated fan-out, blocking collection, and winner selection across turns. The orchestrator first spawned three bidder agents and yielded (the second and third bidder prompts, elided below, are analogous to the first):

\begin{lstlisting}[language=spell,basicstyle=\ttfamily\footnotesize,breaklines=true]
'(do
  (agents/spawn
    "You are bidder A in a sealed-bid auction for a painting. Independently choose exactly one integer bid from 100 through 1000. Do not communicate with any other bidder. Do not ask the parent to choose for you. Send your chosen integer to your parent exactly once using this trailing Spell expression: '(agents/send (agents/parent-handle) YOUR-INTEGER)."
    :bidder-a)
  (agents/spawn "..." :bidder-b)
  (agents/spawn "..." :bidder-c)
  (!extend))
\end{lstlisting}

On its next turn, the orchestrator issued a single blocking multi-target ask, \lstinline[language=spell]|'(agents/!ask [:bidder-a :bidder-b :bidder-c])|. Each bidder's \texttt{agents/send} then woke the orchestrator, with the runtime binding the incoming messages in scope as \texttt{msg-N} maps; the final turn computed the winner from those bindings:

\begin{lstlisting}[language=spell,basicstyle=\ttfamily\footnotesize,breaklines=true]
(def bid-a (:body msg-3687))
(def bid-b (:body msg-3688))
(def bid-c (:body msg-3686))
(def winning-bid (max bid-a bid-b bid-c))
(def winner (cond
  (= bid-a winning-bid) :bidder-a
  (= bid-b winning-bid) :bidder-b
  :else :bidder-c))
'(str "Sealed bids: bidder A = " bid-a ", bidder B = " bid-b ", bidder C = " bid-c ". Winner: " (name winner) " with a bid of " winning-bid ".")
\end{lstlisting}

The other games elicited different topologies. Full-credit telephone trials formed a peer-to-peer chain, spawning the three relay agents in reverse order so that each downstream agent was ready before its upstream sender, then passing the message by sibling-to-sibling \texttt{agents/send}; full-credit twenty-questions trials used a blocking \texttt{agents/!reply-ask} loop between the host and a single worker, switching to a terminal \texttt{agents/reply} on the correct guess.

\subsection{Benchmark protocols}\label{app:bench:protocols}

Benchmark task evaluations were performed on GCP virtual machines; the asynchronous orchestration games were run locally. Table~\ref{tab:compute-resources} summarizes the compute resources used by the benchmark task evaluations. API spend is model-provider API spend reconstructed from benchmark usage logs, not GCP VM cost. Pilot runs, failed dispatches, and aborted diagnostic reruns are excluded. Figure~\ref{fig:obs5} reuses the medium-effort Terminal-Bench and SWE-bench Lite rows from the GPT-5.4 frontier analyses. The model comparisons in Figure~\ref{fig:obs1} and Table~\ref{tab:obs1-toolcall-model-values} instead use dedicated subset runs for every model.

\footnotesize
\begin{longtable}[]{@{}
	>{\raggedright\arraybackslash}p{(\linewidth - 10\tabcolsep) * \real{0.2200}}
	>{\raggedright\arraybackslash}p{(\linewidth - 10\tabcolsep) * \real{0.1800}}
	>{\raggedright\arraybackslash}p{(\linewidth - 10\tabcolsep) * \real{0.1600}}
	>{\raggedright\arraybackslash}p{(\linewidth - 10\tabcolsep) * \real{0.2200}}
	>{\raggedright\arraybackslash}p{(\linewidth - 10\tabcolsep) * \real{0.1000}}
	>{\raggedleft\arraybackslash}p{(\linewidth - 10\tabcolsep) * \real{0.1200}}@{}}
	\caption{Compute resources for retained benchmark analyses. Wall-clock is elapsed source-run time when run metadata was retained; for sharded JSONL-only runs it is the approximate span from the first to last completed item record.}\label{tab:compute-resources}                                                                                                                                                                                     \\
	\toprule\noalign{}
	\begin{minipage}[b]{\linewidth}\raggedright
		Analysis
	\end{minipage}         & \begin{minipage}[b]{\linewidth}\raggedright
		                         VMs and machine type
	                         \end{minipage}                                                                               & \begin{minipage}[b]{\linewidth}\raggedright
		                                                                                                                        Tasks
	                                                                                                                        \end{minipage} & \begin{minipage}[b]{\linewidth}\raggedright
		                                                                                                                                         Concurrency and caps
	                                                                                                                                         \end{minipage}                                                                                   & \begin{minipage}[b]{\linewidth}\raggedright
		                                                                                                                                                                                                                                            Wall-clock
	                                                                                                                                                                                                                                            \end{minipage}                          & \begin{minipage}[b]{\linewidth}\raggedleft
		                                                                                                                                                                                                                                                                                      API spend
	                                                                                                                                                                                                                                                                                      \end{minipage}                                                                                                                          \\
	\midrule\noalign{}
	\endfirsthead
	\toprule\noalign{}
	\begin{minipage}[b]{\linewidth}\raggedright
		Analysis
	\end{minipage}         & \begin{minipage}[b]{\linewidth}\raggedright
		                         VMs and machine type
	                         \end{minipage}                                                                               & \begin{minipage}[b]{\linewidth}\raggedright
		                                                                                                                        Tasks
	                                                                                                                        \end{minipage} & \begin{minipage}[b]{\linewidth}\raggedright
		                                                                                                                                         Concurrency and caps
	                                                                                                                                         \end{minipage}                                                                                   & \begin{minipage}[b]{\linewidth}\raggedright
		                                                                                                                                                                                                                                            Wall-clock
	                                                                                                                                                                                                                                            \end{minipage}                          & \begin{minipage}[b]{\linewidth}\raggedleft
		                                                                                                                                                                                                                                                                                      API spend
	                                                                                                                                                                                                                                                                                      \end{minipage}                                                                                                                          \\
	\midrule\noalign{}
	\endhead
	\bottomrule\noalign{}
	\endlastfoot
	Observation 1 model diagnostics, Terminal-Bench 1.1 & One \texttt{e2-standard-16} VM per model, 300 GB disk; additional Fireworks retry VMs & 7 models $\times$ 32 tasks & Four concurrent tasks; 3600 s agent cap; 600 s verifier cap; \texttt{\$5} per task & 1.3 h for GPT-5.4; other runs not summarized & \texttt{\$86.64} \\
	Observation 1 model diagnostics, SWE-bench Lite & One \texttt{e2-standard-16} VM per model, 300 GB disk & 7 models $\times$ 32 sampled tasks; 31 evaluable per model & Four concurrent tasks and four build workers; 1800 s per item; \texttt{\$10} per item & Not summarized & \texttt{\$127.61} \\
	GPT-5.4 Terminal-Bench frontier                     & 6 VMs, \texttt{e2-standard-8} (\Spell{} and Codex CLI at low/medium/high effort)                                             & 6 $\times$ 80 tasks                         & \texttt{-\/-n-concurrent} 4; 3600 s agent cap; 600 s verifier cap                                                             & 2.5--5.0 h                                                           & \texttt{\$248.95} \\
	GPT-5.4 Terminal-Bench coding intervention          & Same Terminal-Bench harness and caps as the frontier \Spell{} rows                                                           & 3 $\times$ 80 tasks                         & \texttt{-\/-n-concurrent} 4; coding reminder injected in the initial \Spell{} program                                         & Not separately retained                                              & \texttt{\$84.35}  \\
	Opus 4.6 Terminal-Bench comparison                  & 2 VMs, \texttt{e2-standard-8} (\Spell{} and Claude Code)                                                                     & 2 $\times$ 80 tasks                         & \texttt{-\/-n-concurrent} 4; 3600 s agent cap; 600 s verifier cap; \texttt{\$5} \Spell{}/Claude-Code budget cap                  & 3.5--5.1 h                                                           & \texttt{\$108.94} \\
	GPT-5.4 SWE-bench Lite frontier                     & 18 VMs, \texttt{e2-standard-16}, 300 GB boot disk (3 shards $\times$ 2 systems $\times$ 3 efforts)                        & 6 $\times$ 300 tasks                        & \texttt{-\/-parallel} 4; \texttt{-\/-timeout} 3600; \texttt{-\/-prewarm-envs}; \texttt{-\/-paper-compliant}                   & 7.3 h                                                                & \texttt{\$883.78} \\
	LongBench v2 comparison                             & 4 VMs, \texttt{e2-standard-4}, 100 GB boot disk (pilot/rest for each system)                                              & 2 $\times$ 200 items                        & \texttt{-\/-parallel} 4; \Spell{} used 1200 s timeout and \texttt{\$3} budget cap; Codex used the general-runner default timeout & 0.9--1.9 h per system                                                & \texttt{\$53.41}  \\
	AppWorld dev comparison                             & 2 retained full-run VMs, \texttt{e2-standard-16}, 200 GB boot disk                                                        & 2 $\times$ 57 tasks                         & \texttt{-\/-n-concurrent} 4; 3600 s agent cap; 600 s verifier cap; \texttt{\$5} \Spell{} budget cap                              & 0.8--2.5 h per system                                                & \texttt{\$43.97}  \\
	mini-swe-agent Terminal-Bench comparison & One GCP VM & 80 tasks & Four concurrent tasks; 3600 s agent cap; 600 s verifier cap; \texttt{\$5} per task & Not summarized & \texttt{\$30.71} \\
	mini-swe-agent SWE-bench Lite comparison & Four GCP VMs; 75-item shards & 300 tasks & Four concurrent tasks per VM; 3600 s per item; \texttt{\$10} per item; prewarmed environments & Not summarized & \texttt{\$86.57} \\
\end{longtable}
\normalsize

The model-diagnostic rows cover the seven models in Table~\ref{tab:obs1-toolcall-model-values}; Figure~\ref{fig:obs1} plots six of them. The reported API spend is the sum of recorded costs: 15 Terminal-Bench records and 19 SWE-bench records lack cost data, so these totals are lower bounds. The latter include the shared container-build failure for each of the seven models. The GPT-5.4 Terminal-Bench duration is from run metadata; the other diagnostic durations and the mini-swe-agent durations were not summarized. The retained mini-swe-agent records identify one Terminal-Bench VM and four SWE-bench shards, but do not specify their machine types.

\subsubsection{Terminal-Bench 1.1}\label{terminal-bench-1.1}

Terminal-Bench 1.1 is a set of containerized command-line tasks with hidden tests over the final filesystem state or program behavior. The GPT-5.4 \Spell{}-Codex comparison used the full 80-task \texttt{terminal-bench-core==0.1.1} old-core task set, one trial per task, Docker execution, disabled live web access, and the Terminal-Bench verifier as the source of truth. For Figure~\ref{fig:obs1}, we used a fixed 32-task subset whose exact identifiers are provided in the benchmarking repository at \texttt{paper-repro/figure2/terminalbench-32.txt}. No items were excluded based on errors or outcomes. The benchmark harness passed each task instruction as the task prompt.

The same runtime settings were used across compared systems: four-way within-VM concurrency, a 3600 s per-task agent wall-clock cap, and a 600 s global verifier/test timeout. \Spell{} and Claude Code received a \texttt{\$5} per-task dollar cap from their respective harness adapters; Codex CLI was run with the same wall-clock and verifier caps but without a dollar budget cap. The comparison between \Spell{} and Codex CLI remained fair because no \Spell{} item actually hit the \texttt{\$5} cap (Appendix~\ref{app:bench:cost}).

For the Observation 1 analysis, the same 32 items were run for every model. Terminal run errors include timeout statuses and records with a non-null \texttt{error\_type}; errors remain in the accuracy denominator. Provider refusals were audited separately from canonical agent logs.

\subsubsection{SWE-bench Lite}\label{swe-bench-lite}

SWE-bench Lite contains 300 real GitHub issues from 11 Python repositories. Each task gives the agent an issue description and a repository checkout, and the agent must produce a patch that passes the SWE-bench tests. The full-run comparisons in Figures~\ref{fig:obs2_gpt54_frontiers} and~\ref{fig:obs5} used all 300 items in the public \texttt{princeton-nlp/SWE-bench\_Lite} split, official per-instance SWE-bench images, the repository path \texttt{/testbed}, and one trial per item. For Figure~\ref{fig:obs1}, 32 tasks were sampled from official dataset order with seed 20260725. The same sample was used for every model. One task, \texttt{matplotlib\_\_matplotlib-23913}, had the same container-build failure for all models and was excluded as a shared infrastructure failure, leaving 31 evaluable tasks. The SWE-bench Lite prompt includes the issue description and omits \texttt{hints\_text}. It instructs the agent to make minimal non-test changes, explore the repository, reproduce the error when practical, edit source code, rerun focused validation, consider edge cases, and run selected tests. Live web access and task-specific solution lookup were disabled by policy for all compared agents. Items were scored by applying patches against the official SWE-bench evaluation harness.

The same runtime settings were used across compared systems: \texttt{-\/-parallel\ 4}, \texttt{-\/-timeout\ 3600} per item, \texttt{-\/-paper-compliant}, and \texttt{-\/-prewarm-envs}. \Spell{} received the SWE-bench adapter's \texttt{\$10} per-item dollar cap; Codex CLI received no dollar budget cap. The comparison between \Spell{} and Codex CLI remained fair because no item from either method actually hit \texttt{\$10} (Appendix~\ref{app:bench:cost}). The 32-item model-diagnostic runs used a per-item timeout of 1800 s.

\subsubsection{LongBench v2}\label{longbench-v2}

LongBench v2 is a 503-item long-context multiple-choice benchmark spanning single-document question answering, multi-document question answering, long-dialogue history understanding, long in-context learning, and structured-data understanding. The retained Figure~\ref{fig:obs5} comparison evaluated the first 200 records in the benchmark dataset's published order, using the same records for both systems. Unlike the Terminal-Bench subset in Figure~\ref{fig:obs1}, these records were not re-sorted by task identifier. No items were excluded based on errors or outcomes. Each document was supplied on disk for tool-using agents, and the prompt included the document path, document size, question, answer options, and instruction to return only the option letter.

Runs used four-way parallelism and GPT-5.4 medium effort for both \Spell{} and Codex CLI, but the runtime and budget caps were not identical. The retained \Spell{} rerun used \texttt{-\/-timeout\ 1200} and a \texttt{\$3} per-item \Spell{} budget cap. The retained Codex CLI rerun used the general-benchmark default 300 s per-item subprocess timeout and no dollar budget cap. The comparison between \Spell{} and Codex CLI remained fair because no \Spell{} item hit the \texttt{\$3} cap and no Codex item exceeded either \texttt{\$3} cost or 300 s latency (Appendix~\ref{app:bench:cost}).

\subsubsection{AppWorld Dev}\label{appworld-dev}

AppWorld evaluates computer-use tasks over a simulated application state. The retained Figure~\ref{fig:obs5} comparison used the full 57-item \texttt{appworld-dev} split through the Terminal-Bench adapter, GPT-5.4 medium effort, and four-way within-VM parallelism. The \Spell{} condition used \texttt{openai-tc:gpt-5.4}, \texttt{config/agents/io-tc.agent.edn}, and the plain trailing \texttt{\textquotesingle{}(!extend)}, matching the LongBench, Terminal-Bench, and SWE-bench \Spell{} rows. The Codex condition used Codex CLI with \texttt{gpt-5.4}.

The AppWorld dev \Spell{} and Codex CLI runs used the same Terminal-Bench-adapter caps: \texttt{-\/-n-concurrent\ 4}, \texttt{-\/-agent-timeout-sec\ 3600}, and \texttt{-\/-test-timeout-sec\ 600}. \Spell{} received the Terminal-Bench adapter's default \texttt{\$5} per-task dollar cap; Codex CLI used the same wall-clock and verifier caps but no dollar budget cap. The comparison between \Spell{} and Codex CLI remained fair with respect to the budget cap because no \Spell{} item hit the \texttt{\$5} cap and no Codex item exceeded \texttt{\$5} (Appendix~\ref{app:bench:cost}).

\subsubsection{Pairwise Statistics}\label{pairwise-statistics}

For paired comparisons, the analysis counts decisive items only: an item where method A solved and method B did not, or vice versa. Reported pairwise \texttt{p} values use an exact two-sided binomial test with null probability \texttt{p=0.5} over decisive items.

\subsection{Comparison between models}\label{app:bench:obs1}

\small
\begin{longtable}[]{@{}lrrrr@{}}
	\caption{Tool-call transport accuracy and terminal run-error rates. Terminal-Bench used 32 tasks; SWE-bench Lite had 31 evaluable tasks after one shared container-build failure. Fable 5 produced explicit provider refusals on 7/32 Terminal-Bench tasks (21.9\%); no other tabulated model produced a canonical Terminal-Bench refusal.}\label{tab:obs1-toolcall-model-values} \\
	\toprule\noalign{}
	Model          & TB resolved    & TB errors       & SWE resolved   & SWE errors                                                                  \\
	\midrule\noalign{}
	\endhead
	\bottomrule\noalign{}
	\endlastfoot
	GPT-5.4        & 17/32 (53.1\%) & 6/32 (18.8\%)  & 16/31 (51.6\%) & 0/31 (0.0\%)                                                                \\
	Opus 4.6       & 12/32 (37.5\%) & 7/32 (21.9\%)  & 16/31 (51.6\%) & 6/31 (19.4\%)                                                               \\
	GLM-5.1        & 12/32 (37.5\%) & 10/32 (31.2\%) & 12/31 (38.7\%) & 7/31 (22.6\%)                                                               \\
	Kimi-K2.6      & 12/32 (37.5\%) & 12/32 (37.5\%) & 12/31 (38.7\%) & 11/31 (35.5\%)                                                              \\
	Qwen3.7 Plus   & 13/32 (40.6\%) & 7/32 (21.9\%)  & 5/31 (16.1\%)  & 7/31 (22.6\%)                                                               \\
	GPT-5.6 Sol    & 17/32 (53.1\%) & 5/32 (15.6\%)  & 17/31 (54.8\%) & 0/31 (0.0\%)                                                                \\
	Fable 5        & 13/32 (40.6\%) & 13/32 (40.6\%) & 18/31 (58.1\%) & 8/31 (25.8\%)                                                               \\
\end{longtable}
\normalsize

Earlier diagnostic traces showed that some open-weight failures were invalid \Spell{} continuations of the kind warned against in the system prompt. For example, in Kimi-K2.6 on SWE-bench Lite \texttt{astropy\_\_astropy-7746}, the model attempted to use \texttt{io/read-lines} directly in the evaluated program rather than inside a quoted trailing \texttt{!call-now} form, producing the unrecovered error \texttt{io/read-lines: io/ is an effect namespace - use it in the trailing expression via eval}. In Kimi-K2.6 on \texttt{django\_\_django-11422}, the model emitted an unbound \texttt{final-diff} symbol. In the then-current Qwen3.6 Plus endpoint on \texttt{django\_\_django-10914}, the model emitted \texttt{println}, which is not a \Spell{} binding, producing \texttt{Unbound symbol: println}. Thus, tool-call transport reduced wrapper and chat-template leakage, but it did not eliminate invalid \Spell{} programs.

Models also differed in persistence. In the original sorted-identifier subset runs at matched medium effort, Opus 4.6 used roughly twice as many turns and tokens per task as GPT-5.4, with the difference concentrated in failures: on SWE-bench Lite, failed Opus 4.6 runs averaged 26.9 turns versus 10.1 for successful runs, whereas the corresponding GPT-5.4 averages were 10.4 and 7.1. The retained data do not show that this added persistence was productive.

\subsubsection{Prefill transport mode}

Because the open-weight Fireworks endpoints support assistant prefill, I tested a prefill transport mode in which the model prefix is literally the prefix of the \Spell{} program. This transport is natural for \Spell{} because the model's visible completion is the program to be evaluated, but for the closed-source models it was unavailable because the corresponding APIs impose mandatory fencing that separates user inputs from model responses and tool calls. In these earlier diagnostic runs, prefill transport did not improve the open-weight results and instead made them substantially worse (Table~\ref{tab:obs1-prefill-diagnostic-values}). The likely reason is that tool-call transport imposes a useful output constraint: the model must place the raw \Spell{} suffix in the \texttt{spell\_suffix} argument, which separates executable \Spell{} code from ordinary assistant text, hidden-reasoning markers, wrapper echoes, and chat-template artifacts. For example, a common error in prefill mode was that models would emit thinking tags (``<think>'') inside the \Spell{} program, causing reader failure.

\begin{longtable}[]{@{}lrrrr@{}}
	\caption{Prefill-transport diagnostic results for open-weight models on the original 32-item diagnostic subsets.}\label{tab:obs1-prefill-diagnostic-values} \\
	\toprule\noalign{}
	Model        & TB resolved  & TB fatal errors & SWE resolved & SWE fatal errors                                                              \\
	\midrule\noalign{}
	\endhead
	\bottomrule\noalign{}
	\endlastfoot
	GLM-5.1      & 1/32 (3.1\%) & 16/32 (50.0\%)  & 0/32 (0.0\%) & 17/32 (53.1\%)                                                                \\
	Kimi-K2.6    & 0/32 (0.0\%) & 17/32 (53.1\%)  & 0/32 (0.0\%) & 12/32 (37.5\%)                                                                \\
	Qwen3.6 Plus & 0/32 (0.0\%) & 10/32 (31.2\%)  & 0/32 (0.0\%) & 14/32 (43.8\%)                                                                \\
\end{longtable}

\subsection{Harness comparison on coding benchmarks}\label{app:bench:harness}

\subsubsection{Terminal-Bench 1.1 Cost-Accuracy Frontier}\label{terminal-bench-1.1-cost-accuracy-frontier}

Two secondary analyses were performed. First, a coding intervention prescribed a research/plan/implement/verify workflow. Its effect varied across reasoning levels: at low effort, accuracy was unchanged while cost increased by 56\%; at medium effort, the intervention won 8 tasks and lost 4 while cost increased by 6\%; and at high effort, it won 5 tasks and lost 2 while cost decreased by 7\% (Table~\ref{tab:tb-gpt54-source-values}). At high effort, 38 tasks were solved by both conditions and 35 by neither. The reminder remained model-accessible in the default configuration, but none of the 1,125 retained main-experiment traces invoked it through \texttt{(!describe reminders :coding)}. Second, a \Spell{} agent running Opus 4.6 was compared against Claude Code running the same model, both with medium effort. Both agents had the same accuracy and similar cost (Table~\ref{tab:tb-opus46-source-values}).

\begin{longtable}[]{@{}
	>{\raggedright\arraybackslash}p{(\linewidth - 10\tabcolsep) * \real{0.2200}}
	>{\raggedright\arraybackslash}p{(\linewidth - 10\tabcolsep) * \real{0.1400}}
	>{\raggedright\arraybackslash}p{(\linewidth - 10\tabcolsep) * \real{0.1200}}
	>{\raggedleft\arraybackslash}p{(\linewidth - 10\tabcolsep) * \real{0.1800}}
	>{\raggedleft\arraybackslash}p{(\linewidth - 10\tabcolsep) * \real{0.1600}}
	>{\raggedleft\arraybackslash}p{(\linewidth - 10\tabcolsep) * \real{0.1800}}@{}}
	\caption{Terminal-Bench 1.1 GPT-5.4 source values for Figure~\ref{fig:obs2_gpt54_frontiers}. The default \Spell{} configuration uses \texttt{\textquotesingle{}(!extend)}; \texttt{:coding} additionally injects the coding-task reminder from Appendix~\ref{app:language:reminders-coding-prompt}.}\label{tab:tb-gpt54-source-values} \\
	\toprule\noalign{}
	\begin{minipage}[b]{\linewidth}\raggedright
		System
	\end{minipage} & \begin{minipage}[b]{\linewidth}\raggedright
		                 Config
	                 \end{minipage} & \begin{minipage}[b]{\linewidth}\raggedright
		                                  Effort
	                                  \end{minipage} & \begin{minipage}[b]{\linewidth}\raggedleft
		                                                   Resolved
	                                                   \end{minipage} & \begin{minipage}[b]{\linewidth}\raggedleft
		                                                                    Cost
	                                                                    \end{minipage} & \begin{minipage}[b]{\linewidth}\raggedleft
		                                                                                     Total tokens
	                                                                                     \end{minipage}                                                                                                                                                                                                      \\
	\midrule\noalign{}
	\endfirsthead
	\toprule\noalign{}
	\begin{minipage}[b]{\linewidth}\raggedright
		System
	\end{minipage} & \begin{minipage}[b]{\linewidth}\raggedright
		                 Config
	                 \end{minipage} & \begin{minipage}[b]{\linewidth}\raggedright
		                                  Effort
	                                  \end{minipage} & \begin{minipage}[b]{\linewidth}\raggedleft
		                                                   Resolved
	                                                   \end{minipage} & \begin{minipage}[b]{\linewidth}\raggedleft
		                                                                    Cost
	                                                                    \end{minipage} & \begin{minipage}[b]{\linewidth}\raggedleft
		                                                                                     Total tokens
	                                                                                     \end{minipage}                                                                                                                                                                                                      \\
	\midrule\noalign{}
	\endhead
	\bottomrule\noalign{}
	\endlastfoot
	\Spell{} GPT-5.4                               & default                                     & low                                         & 35/80 (43.8\%)                             & \texttt{\$12.26}                           & 10.11M                                                                                          \\
	\Spell{} GPT-5.4                               & default                                     & medium                                      & 36/80 (45.0\%)                             & \texttt{\$25.72}                           & 6.96M                                                                                           \\
	\Spell{} GPT-5.4                               & default                                     & high                                        & 40/80 (50.0\%)                             & \texttt{\$40.72}                           & 6.52M                                                                                           \\
	\Spell{} GPT-5.4                               & \texttt{:coding}                            & low                                         & 35/80 (43.8\%)                             & \texttt{\$19.11}                           & 22.51M                                                                                          \\
	\Spell{} GPT-5.4                               & \texttt{:coding}                            & medium                                      & 40/80 (50.0\%)                             & \texttt{\$27.36}                           & 10.21M                                                                                          \\
	\Spell{} GPT-5.4                               & \texttt{:coding}                            & high                                        & 43/80 (53.8\%)                             & \texttt{\$37.88}                           & 6.48M                                                                                           \\
	Codex CLI GPT-5.4                           & default                                     & low                                         & 38/80 (47.5\%)                             & \texttt{\$39.43}                           & 39.93M                                                                                          \\
	Codex CLI GPT-5.4                           & default                                     & medium                                      & 39/80 (48.8\%)                             & \texttt{\$46.96}                           & 51.14M                                                                                          \\
	Codex CLI GPT-5.4                           & default                                     & high                                        & 43/80 (53.8\%)                             & \texttt{\$83.86}                           & 94.65M                                                                                          \\
\end{longtable}

\begin{longtable}[]{@{}
	>{\raggedright\arraybackslash}p{(\linewidth - 10\tabcolsep) * \real{0.2200}}
	>{\raggedright\arraybackslash}p{(\linewidth - 10\tabcolsep) * \real{0.1400}}
	>{\raggedright\arraybackslash}p{(\linewidth - 10\tabcolsep) * \real{0.1200}}
	>{\raggedleft\arraybackslash}p{(\linewidth - 10\tabcolsep) * \real{0.1800}}
	>{\raggedleft\arraybackslash}p{(\linewidth - 10\tabcolsep) * \real{0.1600}}
	>{\raggedleft\arraybackslash}p{(\linewidth - 10\tabcolsep) * \real{0.1800}}@{}}
	\caption{Terminal-Bench 1.1 Claude Opus 4.6 source values. Both rows use medium effort on the full 80-task old-core set. The paired comparison had 6 \Spell{}-only successes, 6 Claude-Code-only successes, 36 shared successes, 32 shared failures, and exact sign-test \(p=1.0\).}\label{tab:tb-opus46-source-values} \\
	\toprule\noalign{}
	\begin{minipage}[b]{\linewidth}\raggedright
		System
	\end{minipage} & \begin{minipage}[b]{\linewidth}\raggedright
		                 Config
	                 \end{minipage} & \begin{minipage}[b]{\linewidth}\raggedright
		                                  Effort
	                                  \end{minipage} & \begin{minipage}[b]{\linewidth}\raggedleft
		                                                   Resolved
	                                                   \end{minipage} & \begin{minipage}[b]{\linewidth}\raggedleft
		                                                                    Cost
	                                                                    \end{minipage} & \begin{minipage}[b]{\linewidth}\raggedleft
		                                                                                     Total tokens
	                                                                                     \end{minipage}                                                                                                                                                                                       \\
	\midrule\noalign{}
	\endhead
	\bottomrule\noalign{}
	\endlastfoot
	\Spell{} Opus 4.6                              & default                                     & medium                                      & 42/80 (52.5\%)                             & \texttt{\$54.98}                           & 13.38M                                                                           \\
	Claude Code Opus 4.6                        & default                                     & medium                                      & 42/80 (52.5\%)                             & \texttt{\$53.96}                           & 33.26M                                                                           \\
\end{longtable}

\subsubsection{SWE-bench Lite Cost-Accuracy Frontier}\label{swe-bench-lite-cost-accuracy-frontier}

\begin{longtable}[]{@{}
	>{\raggedright\arraybackslash}p{(\linewidth - 10\tabcolsep) * \real{0.2200}}
	>{\raggedright\arraybackslash}p{(\linewidth - 10\tabcolsep) * \real{0.1400}}
	>{\raggedright\arraybackslash}p{(\linewidth - 10\tabcolsep) * \real{0.1200}}
	>{\raggedleft\arraybackslash}p{(\linewidth - 10\tabcolsep) * \real{0.1800}}
	>{\raggedleft\arraybackslash}p{(\linewidth - 10\tabcolsep) * \real{0.1600}}
	>{\raggedleft\arraybackslash}p{(\linewidth - 10\tabcolsep) * \real{0.1800}}@{}}
	\caption{SWE-bench Lite GPT-5.4 source values for Figure~\ref{fig:obs2_gpt54_frontiers}. All rows use the default initialization; no \texttt{:coding} prompt-intervention rows are included for SWE-bench Lite.}                            \\
	\toprule\noalign{}
	\begin{minipage}[b]{\linewidth}\raggedright
		System
	\end{minipage} & \begin{minipage}[b]{\linewidth}\raggedright
		                 Config
	                 \end{minipage} & \begin{minipage}[b]{\linewidth}\raggedright
		                                  Effort
	                                  \end{minipage} & \begin{minipage}[b]{\linewidth}\raggedleft
		                                                   Resolved
	                                                   \end{minipage} & \begin{minipage}[b]{\linewidth}\raggedleft
		                                                                    Cost
	                                                                    \end{minipage} & \begin{minipage}[b]{\linewidth}\raggedleft
		                                                                                     Total tokens
	                                                                                     \end{minipage}                                                                                                              \\
	\midrule\noalign{}
	\endhead
	\bottomrule\noalign{}
	\endlastfoot
	\Spell{} GPT-5.4                               & default                                     & low                                         & 153/300 (51.0\%)                           & \texttt{\$68.94}                           & 74.30M  \\
	\Spell{} GPT-5.4                               & default                                     & medium                                      & 171/300 (57.0\%)                           & \texttt{\$102.12}                          & 51.37M  \\
	\Spell{} GPT-5.4                               & default                                     & high                                        & 171/300 (57.0\%)                           & \texttt{\$177.38}                          & 49.26M  \\
	Codex CLI GPT-5.4                           & default                                     & low                                         & 170/300 (56.7\%)                           & \texttt{\$136.79}                          & 134.09M \\
	Codex CLI GPT-5.4                           & default                                     & medium                                      & 172/300 (57.3\%)                           & \texttt{\$161.82}                          & 155.96M \\
	Codex CLI GPT-5.4                           & default                                     & high                                        & 185/300 (61.7\%)                           & \texttt{\$236.73}                          & 220.67M \\
\end{longtable}

\subsubsection{Simple traditional-harness comparison}\label{app:bench:mini-swe-agent}

I additionally evaluated mini-swe-agent v2.4.6~\citep{yang2024sweagent}, a strong academic open-source harness with a fixed outer agent loop and a single shell tool. mini-swe-agent used GPT-5.4 at medium reasoning effort, the same complete benchmark item sets and task prompts, and the same per-item budget and timeout settings as \Spell{}. Its task template prescribed a workflow tailored to tasks involving a codebase issue to be fixed.

\begin{longtable}[]{@{}lccc@{}}
	\caption{GPT-5.4 medium-effort comparison with mini-swe-agent. Each cell reports tasks resolved and total model cost. Codex CLI is included for context. Terminal-Bench costs sum the 79 cost-bearing rows; the shared \texttt{intrusion-detection} setup failure remained in the 80-task accuracy denominator but incurred no model cost.}\label{tab:mini-swe-agent-source-values} \\
	\toprule\noalign{}
	Benchmark & \Spell{} & mini-swe-agent & Codex CLI \\
	\midrule\noalign{}
	\endhead
	\bottomrule\noalign{}
	\endlastfoot
	Terminal-Bench 1.1 & \begin{tabular}[c]{@{}c@{}}36/80 (45.0\%)\\\$25.72\end{tabular} & \begin{tabular}[c]{@{}c@{}}44/80 (55.0\%)\\\$30.71\end{tabular} & \begin{tabular}[c]{@{}c@{}}39/80 (48.8\%)\\\$46.96\end{tabular} \\
	SWE-bench Lite & \begin{tabular}[c]{@{}c@{}}171/300 (57.0\%)\\\$102.12\end{tabular} & \begin{tabular}[c]{@{}c@{}}168/300 (56.0\%)\\\$86.57\end{tabular} & \begin{tabular}[c]{@{}c@{}}172/300 (57.3\%)\\\$161.82\end{tabular} \\
\end{longtable}

On Terminal-Bench, mini-swe-agent and \Spell{} solved 35 tasks in common; mini-swe-agent alone solved 9 and \Spell{} alone solved 1 (exact paired \(p=0.0215\)). On SWE-bench Lite, they solved 155 tasks in common; mini-swe-agent alone solved 13 and \Spell{} alone solved 16 (\(p=0.711\)). Thus, mini-swe-agent performed better on Terminal-Bench at 19.4\% higher total cost, while on SWE-bench Lite its accuracy was statistically indistinguishable from \Spell{} and its total cost was 15.2\% lower. These results remain consistent with the conclusion that \Spell{} is competitive on coding benchmarks, but not uniformly.

\subsubsection{Error and Scoring Anomalies}\label{error-and-scoring-anomalies}

The following table lists benchmark items that failed with an error code across Terminal-Bench and SWE-bench Lite. No results were excluded on the basis of these errors.

\begin{longtable}[]{@{}
	>{\raggedright\arraybackslash}p{(\linewidth - 6\tabcolsep) * \real{0.1800}}
	>{\raggedright\arraybackslash}p{(\linewidth - 6\tabcolsep) * \real{0.1800}}
	>{\raggedleft\arraybackslash}p{(\linewidth - 6\tabcolsep) * \real{0.2000}}
	>{\raggedright\arraybackslash}p{(\linewidth - 6\tabcolsep) * \real{0.4400}}@{}}
	\caption{Error and scoring anomalies retained in the Figure~\ref{fig:obs2_gpt54_frontiers} denominators.}                                                                                                                                                                                                                   \\
	\toprule\noalign{}
	\begin{minipage}[b]{\linewidth}\raggedright
		Benchmark
	\end{minipage} & \begin{minipage}[b]{\linewidth}\raggedright
		                 Condition
	                 \end{minipage} & \begin{minipage}[b]{\linewidth}\raggedleft
		                                  Affected rows
	                                  \end{minipage}         & \begin{minipage}[b]{\linewidth}\raggedright
		                                                           Notes
	                                                           \end{minipage}                                                                                                                                                                                                                       \\
	\midrule\noalign{}
	\endhead
	\bottomrule\noalign{}
	\endlastfoot
	Terminal-Bench 1.1                          & \Spell{} default, all efforts                  & 5 test-timeout rows per effort                     & Same five items at every effort: \path{build-initramfs-qemu}, \path{swe-bench-astropy-1}, \path{swe-bench-astropy-2}, \path{swe-bench-fsspec}, \path{swe-bench-langcodes}. \\
	Terminal-Bench 1.1                          & Codex CLI, all efforts                      & 5 test-timeout rows per effort                     & Same five timeout items as \Spell{}, so \Spell{}-Codex differences are unaffected.                                                                                               \\
	Terminal-Bench 1.1                          & Codex CLI, all efforts                      & 1 parse error and 1 unknown-agent error per effort & \path{cron-broken-network} produced a parse error; \path{intrusion-detection} produced an unknown-agent error.                                                             \\
	Terminal-Bench 1.1                          & \Spell{} default                               & 1 low, 1 medium, and 3 high agent-timeout rows     & \Spell{}-specific agent timeouts occurred in addition to the shared test-timeout items.                                                                                       \\
	Terminal-Bench 1.1                          & \Spell{} default                               & 1 low agent-installation failure                   & The low-effort row \texttt{qemu-startup} had an agent-installation failure.                                                                                                \\
	SWE-bench Lite                              & \Spell{} low                                   & 3 depth-exceeded and 2 recovery-exhausted rows     & These are \Spell{} runtime errors.                                                                                                                                            \\
	SWE-bench Lite                              & \Spell{} medium                                & 0 generation errors                                & All rows generated patches.                                                                                                                                                \\
	SWE-bench Lite                              & \Spell{} high                                  & 2 install errors and 1 container error             & These were marked as failures.                                                                                                                                             \\
	SWE-bench Lite                              & Codex CLI low                               & 1 install error                                    & This was marked as a failure.                                                                                                                                              \\
	SWE-bench Lite                              & Codex CLI medium/high                       & 0 generation errors                                & All rows generated patches.                                                                                                                                                \\
\end{longtable}

Follow-up diagnosed the five Terminal-Bench rows that errored with \texttt{test\_timeout} in all conditions tested. Rerunning the five items (Codex low patches) with \texttt{-\/-test-timeout-sec\ 3600} and \texttt{-\/-agent-timeout-sec\ 3600} still ended as \texttt{test\_timeout}. However, a custom empty-patch/no-op probe of \texttt{swe-bench-langcodes}, which installed the Codex runtime but did not produce a patch, completed the test phase under the original 600 s budget and failed normally; for this reason, these items were retained in all analyses.

\subsection{Token and cost accounting}\label{app:bench:cost}

Token usage was tracked per API call and summed across calls, stratified by token type (cached/uncached input, output, and output cache writes if priced differently). Token totals in tables below count the sum of all token types.

\begin{longtable}[]{@{}
	>{\raggedright\arraybackslash}p{(\linewidth - 8\tabcolsep) * \real{0.1579}}
	>{\raggedleft\arraybackslash}p{(\linewidth - 8\tabcolsep) * \real{0.2105}}
	>{\raggedleft\arraybackslash}p{(\linewidth - 8\tabcolsep) * \real{0.2105}}
	>{\raggedleft\arraybackslash}p{(\linewidth - 8\tabcolsep) * \real{0.2105}}
	>{\raggedleft\arraybackslash}p{(\linewidth - 8\tabcolsep) * \real{0.2105}}@{}}
	\caption{Per-million-token prices used for benchmark cost reconstruction.}                                                                                                                            \\
	\toprule\noalign{}
	\begin{minipage}[b]{\linewidth}\raggedright
		Model family
	\end{minipage} & \begin{minipage}[b]{\linewidth}\raggedleft
		                 Uncached input / 1M
	                 \end{minipage} & \begin{minipage}[b]{\linewidth}\raggedleft
		                                  Cached input / 1M
	                                  \end{minipage} & \begin{minipage}[b]{\linewidth}\raggedleft
		                                                   Cache write / 1M
	                                                   \end{minipage} & \begin{minipage}[b]{\linewidth}\raggedleft
		                                                                    Output / 1M
	                                                                    \end{minipage}                                                                                         \\
	\midrule\noalign{}
	\endhead
	\bottomrule\noalign{}
	\endlastfoot
	GPT-5.4                                     & \texttt{\$2.50}                            & \texttt{\$0.25}                            & \texttt{\$3.125}                           & \texttt{\$15.00} \\
	Opus 4.6                                    & \texttt{\$5.00}                            & \texttt{\$0.50}                            & \texttt{\$6.25}                            & \texttt{\$25.00} \\
\end{longtable}

These numbers do not account for the GPT-5.4 long-context pricing tier, which applies higher cost to API calls above a token threshold, and this was applied to both \Spell{} and Codex CLI. Because Codex CLI has much larger average context length than \Spell{} (see below), omitting the higher long-context tier is more likely to favor Codex.

\subsubsection{Token utilization}\label{token-utilization}

\begingroup
\footnotesize
\begin{longtable}[]{@{}
	>{\raggedright\arraybackslash}p{(\linewidth - 14\tabcolsep) * \real{0.1700}}
	>{\raggedright\arraybackslash}p{(\linewidth - 14\tabcolsep) * \real{0.1200}}
	>{\raggedleft\arraybackslash}p{(\linewidth - 14\tabcolsep) * \real{0.1300}}
	>{\raggedleft\arraybackslash}p{(\linewidth - 14\tabcolsep) * \real{0.1100}}
	>{\raggedleft\arraybackslash}p{(\linewidth - 14\tabcolsep) * \real{0.1400}}
	>{\raggedleft\arraybackslash}p{(\linewidth - 14\tabcolsep) * \real{0.1100}}
	>{\raggedleft\arraybackslash}p{(\linewidth - 14\tabcolsep) * \real{0.1200}}
	>{\raggedleft\arraybackslash}p{(\linewidth - 14\tabcolsep) * \real{0.1000}}@{}}
	\caption{Medium-effort GPT-5.4 token and cost breakdown for the two Figure~\ref{fig:obs2_gpt54_frontiers} coding benchmarks. Token columns report per-item averages rounded to the nearest 1k tokens.}                                                                                                                            \\
	\toprule\noalign{}
	\begin{minipage}[b]{\linewidth}\raggedright
		Benchmark
	\end{minipage} & \begin{minipage}[b]{\linewidth}\raggedright
		                 System
	                 \end{minipage} & \begin{minipage}[b]{\linewidth}\raggedleft
		                                  Resolved / items
	                                  \end{minipage} & \begin{minipage}[b]{\linewidth}\raggedleft
		                                                   Cost
	                                                   \end{minipage} & \begin{minipage}[b]{\linewidth}\raggedleft
		                                                                    Uncached\\k/item
	                                                                    \end{minipage} & \begin{minipage}[b]{\linewidth}\raggedleft
		                                                                                     Cached\\k/item
	                                                                                     \end{minipage} & \begin{minipage}[b]{\linewidth}\raggedleft
		                                                                                                      Output\\k/item
	                                                                                                      \end{minipage} & \begin{minipage}[b]{\linewidth}\raggedleft
		                                                                                                                       Total\\k/item
	                                                                                                                       \end{minipage}                                                                                                                                                                  \\
	\midrule\noalign{}
	\endhead
	\bottomrule\noalign{}
	\endlastfoot
	Terminal-Bench 1.1                          & \Spell{}                                       & 36/80                                      & \$25.72                                    & 20k                                        & 49k                                        & 17k                                        & 87k  \\
	Terminal-Bench 1.1                          & Codex CLI                                   & 39/80                                      & \$46.96                                    & 48k                                        & 586k                                       & 5k                                         & 639k \\
	SWE-bench Lite                              & \Spell{}                                       & 171/300                                    & \$102.12                                   & 39k                                        & 118k                                       & 14k                                        & 171k \\
	SWE-bench Lite                              & Codex CLI                                   & 172/300                                    & \$161.82                                   & 45k                                        & 469k                                       & 5k                                         & 520k \\
\end{longtable}
\endgroup

\begin{figure}
	\centering
	\pandocbounded{\includegraphics[width=\linewidth,keepaspectratio,alt={Medium-effort token breakdown}]{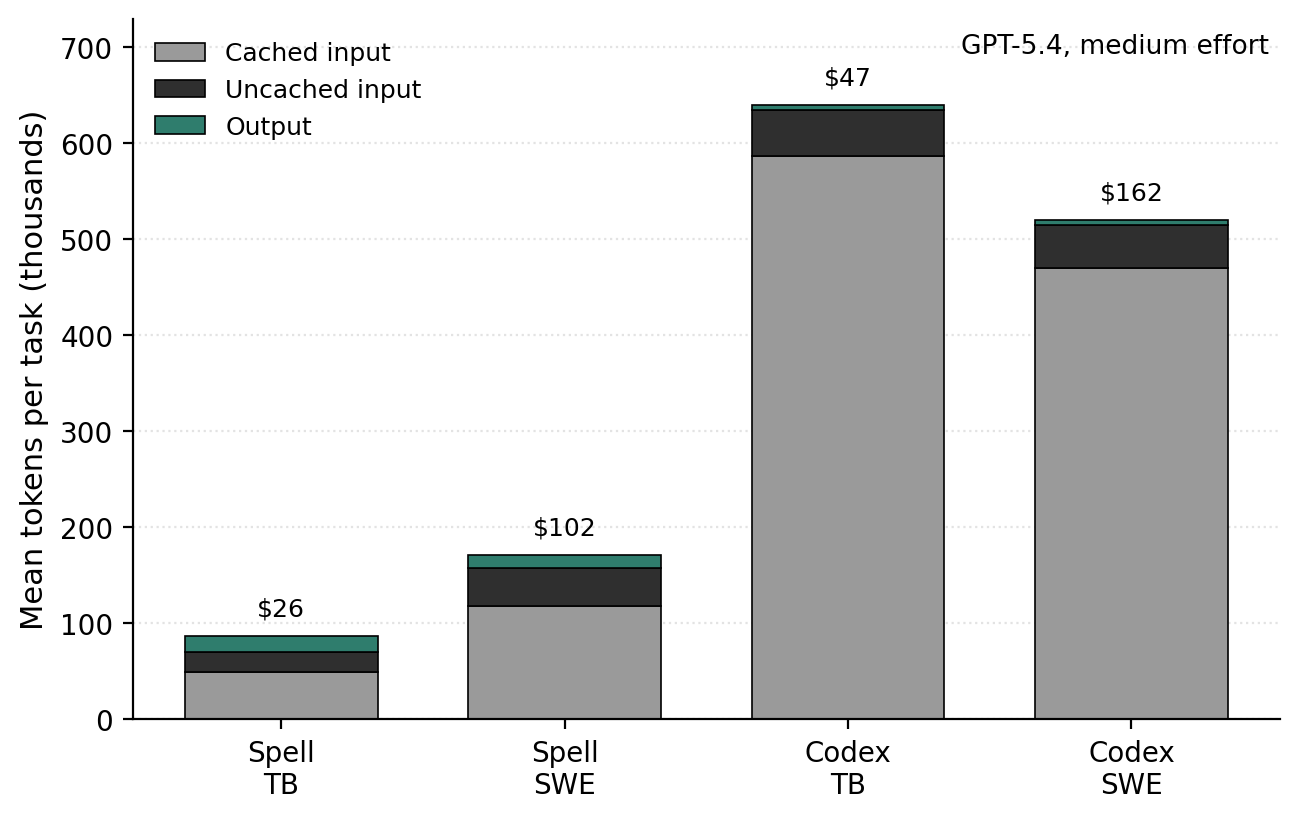}}
	\caption{Mean cached input, uncached input, and output tokens per task in medium-effort GPT-5.4 coding runs.}
\end{figure}

\subsubsection{Budget-Cap Sensitivity}\label{budget-cap-sensitivity}

Per-item budget caps were applied to \Spell{} but not Codex CLI, motivating a sensitivity analysis. In Terminal-Bench 1.1, where a \texttt{\$5} per-task cap was imposed for \Spell{} but not Codex, no \Spell{} item reached the cap at any effort. In SWE-bench Lite, no \Spell{} item reached the \texttt{\$10} per-item cap, and no Codex CLI item exceeded the same \texttt{\$10} threshold. In LongBench v2, no \Spell{} item reached the \texttt{\$3} per-item cap, and no Codex CLI item exceeded the same \texttt{\$3} threshold. In AppWorld dev, no \Spell{} item reached the \texttt{\$5} per-item cap, and no Codex CLI item exceeded the same \texttt{\$5} threshold.

\begin{longtable}[]{@{}llrrrr@{}}
	\caption{Observed cap hits in retained GPT-5.4 \Spell{}-Codex runs.}                      \\
	\toprule\noalign{}
	Benchmark      & Measure                    & Low   & Med.  & High  & Max              \\
	\midrule\noalign{}
	\endfirsthead
	\toprule\noalign{}
	Benchmark      & Measure                    & Low   & Med.  & High  & Max              \\
	\midrule\noalign{}
	\endhead
	\bottomrule\noalign{}
	\endlastfoot
	Terminal-Bench & \Spell{} at \texttt{\$5} cap  & 0/80  & 0/80  & 0/80  & \texttt{\$2.66}  \\
	Terminal-Bench & Codex \(>\)\texttt{\$5}    & 1/80  & 1/80  & 2/80  & \texttt{\$12.66} \\
	SWE-bench      & \Spell{} at \texttt{\$10} cap & 0/300 & 0/300 & 0/300 & \texttt{\$2.57}  \\
	SWE-bench      & Codex \(>\)\texttt{\$10}   & 0/300 & 0/300 & 0/300 & \texttt{\$4.16}  \\
	LongBench      & \Spell{} at \texttt{\$3} cap  & --    & 0/200 & --    & \texttt{\$0.70}  \\
	LongBench      & Codex \(>\)\texttt{\$3}    & --    & 0/200 & --    & \texttt{\$0.71}  \\
	AppWorld       & \Spell{} at \texttt{\$5} cap  & --    & 0/57  & --    & \texttt{\$2.40}  \\
	AppWorld       & Codex \(>\)\texttt{\$5}    & --    & 0/57  & --    & \texttt{\$0.55}  \\
\end{longtable}

Repricing Terminal-Bench 1.1 GPT-5.4 Codex CLI rows whose observed cost exceeded \texttt{\$5} as exactly \texttt{\$5} has the following effect:

\begin{longtable}[]{@{}lrrrrr@{}}
	\caption{Terminal-Bench 1.1 Codex CLI cost sensitivity under a \texttt{\$5} per-task cap.}           \\
	\toprule\noalign{}
	Effort & Rows \(>\)\texttt{\$5} & Observed         & Repriced         & Change           & Reduction \\
	\midrule\noalign{}
	\endfirsthead
	\toprule\noalign{}
	Effort & Rows \(>\)\texttt{\$5} & Observed         & Repriced         & Change           & Reduction \\
	\midrule\noalign{}
	\endhead
	\bottomrule\noalign{}
	\endlastfoot
	Low    & 1                      & \texttt{\$39.43} & \texttt{\$37.83} & \texttt{\$1.60}  & 4.1\%     \\
	Medium & 1                      & \texttt{\$46.96} & \texttt{\$40.67} & \texttt{\$6.29}  & 13.4\%    \\
	High   & 2                      & \texttt{\$83.86} & \texttt{\$69.22} & \texttt{\$14.64} & 17.5\%    \\
\end{longtable}

This adjustment was not made in Figure~\ref{fig:obs2_gpt54_frontiers} because no \Spell{} item actually hit the budget cap. Making this adjustment would not change the ordering by cost of any pair of comparators.

\subsection{Context management and pruning}\label{app:bench:context}

Feature-usage and pruning analyses are based on program traces emitted during each \Spell{} run. To estimate how many tokens were removed from context, the analysis replays the structural effect of pruning forms on the parsed program. It traverses the AST and maintains a stack of ordered sibling forms under the current parent expression. A \texttt{prune} or \texttt{rethink} form with argument \(k\) is interpreted as removing the \(k\) immediately preceding sibling forms; these are popped, and their source is recorded for tokenization. Removed source text and full turn context are tokenized with the same tokenizer used for GPT-5.4 accounting. For turn \(t\), \(D_t\) is the number of tokens removed by pruning events attributed to that turn, \(A_t = \sum_{i<t} D_i\) is the cumulative amount already absent from the prompt at the start of the turn, and \(C_t\) is the full token count for the turn, combining the prefix and response. The summary table reports means over successfully parsed traced turns; malformed traces are excluded from these aggregates.

\pagebreak
\begingroup
\footnotesize
\begin{longtable}[]{@{}
	>{\raggedright\arraybackslash}p{(\linewidth - 10\tabcolsep) * \real{0.3000}}
	>{\raggedleft\arraybackslash}p{(\linewidth - 10\tabcolsep) * \real{0.1200}}
	>{\raggedleft\arraybackslash}p{(\linewidth - 10\tabcolsep) * \real{0.1500}}
	>{\raggedleft\arraybackslash}p{(\linewidth - 10\tabcolsep) * \real{0.1500}}
	>{\raggedleft\arraybackslash}p{(\linewidth - 10\tabcolsep) * \real{0.1500}}
	>{\raggedleft\arraybackslash}p{(\linewidth - 10\tabcolsep) * \real{0.1300}}@{}}
	\caption{Per-turn context-pruning token accounting in GPT-5.4 \Spell{} traces.}                                                                                                                                                             \\
	\toprule\noalign{}
	\begin{minipage}[b]{\linewidth}\raggedright
		Condition
	\end{minipage} & \begin{minipage}[b]{\linewidth}\raggedleft
		                 Mean turns
	                 \end{minipage} & \begin{minipage}[b]{\linewidth}\raggedleft
		                                  Mean pruned\\ / turn
	                                  \end{minipage} & \begin{minipage}[b]{\linewidth}\raggedleft
		                                                   Mean completion
	                                                   \end{minipage} & \begin{minipage}[b]{\linewidth}\raggedleft
		                                                                    Mean pruned\\ / completion
	                                                                    \end{minipage} & \begin{minipage}[b]{\linewidth}\raggedleft
		                                                                                     Mean cumulative\\ pruned
	                                                                                     \end{minipage}                                                                                                           \\
	\midrule\noalign{}
	\endfirsthead
	\toprule\noalign{}
	\begin{minipage}[b]{\linewidth}\raggedright
		Condition
	\end{minipage} & \begin{minipage}[b]{\linewidth}\raggedleft
		                 Mean turns
	                 \end{minipage} & \begin{minipage}[b]{\linewidth}\raggedleft
		                                  Mean pruned\\ / turn
	                                  \end{minipage} & \begin{minipage}[b]{\linewidth}\raggedleft
		                                                   Mean completion
	                                                   \end{minipage} & \begin{minipage}[b]{\linewidth}\raggedleft
		                                                                    Mean pruned\\ / completion
	                                                                    \end{minipage} & \begin{minipage}[b]{\linewidth}\raggedleft
		                                                                                     Mean cumulative\\ pruned
	                                                                                     \end{minipage}                                                                                                           \\
	\midrule\noalign{}
	\endhead
	\bottomrule\noalign{}
	\endlastfoot
	SWE-bench Lite low                          & 9.7                                        & 507                                        & 17,912                                     & 2.8\%                                      & 2,891  \\
	SWE-bench Lite medium                       & 9.3                                        & 1,198                                      & 11,886                                     & 10.1\%                                     & 12,124 \\
	SWE-bench Lite high                         & 9.0                                        & 1,730                                      & 9,894                                      & 17.5\%                                     & 13,282 \\
	Terminal-Bench 1.1 low                      & 7.9                                        & 369                                        & 11,113                                     & 3.3\%                                      & 3,350  \\
	Terminal-Bench 1.1 medium                   & 7.1                                        & 883                                        & 5,428                                      & 16.3\%                                     & 7,664  \\
	Terminal-Bench 1.1 high                     & 5.9                                        & 748                                        & 4,358                                      & 17.2\%                                     & 3,342  \\
\end{longtable}
\endgroup

Higher reasoning effort shifts compute into internal reasoning rather than extending the rollout. In the GPT-5.4 \Spell{} runs, mean reasoning tokens per item rise from 1.8k (low) to 11.5k (medium) to 28.6k (high) on SWE-bench Lite, and from 1.9k to 13.9k to 27.5k on Terminal-Bench 1.1, while mean turns per task (table above) and total tokens (Appendix~\ref{app:bench:harness}) fall. API cost nevertheless rises with effort because reasoning output is more expensive per token than the cached input context it replaces.

\subsection{Feature utilization}\label{app:bench:features}

\Spell{} feature utilization was measured from the same model-response suffixes used in the context-management analysis. Prefixes were excluded; only newly emitted \Spell{} forms were counted. Table~\ref{tab:feature-utilization} summarizes the retained GPT-5.4 \Spell{} traces used for the Terminal-Bench 1.1 and SWE-bench Lite analyses, pooling low-, medium-, and high-effort runs within each benchmark. Context-management and ordinary tool-call features were common. Self-call primitives, auxiliary model calls, multi-agent delegation, shared-state access, web access, and ordinary higher-order/control forms were rare or absent.

\begingroup
\footnotesize
\begin{longtable}[]{@{}lrrrr@{}}
	\caption{Feature utilization in retained GPT-5.4 \Spell{} traces. ``Uses'' counts emitted forms; ``traces'' counts traces with at least one use. Function rows count direct calls. Namespace rows, such as \texttt{agents/*}, count any call to a function within that namespace. The \texttt{web} namespace was unavailable in the configuration.}
	\label{tab:feature-utilization}                                              \\
	\toprule\noalign{}
	Feature            & SWE uses & SWE traces & Terminal uses & Terminal traces \\
	\midrule\noalign{}
	\endfirsthead
	\toprule\noalign{}
	Feature            & SWE uses & SWE traces & Terminal uses & Terminal traces \\
	\midrule\noalign{}
	\endhead
	\bottomrule\noalign{}
	\endlastfoot
	\texttt{!peek}     & 1,980    & 452/898    & 288           & 89/227          \\
	\texttt{!call-now} & 5,015    & 872/898    & 876           & 207/227         \\
	\texttt{!llm-self} & 0        & 0/898      & 0             & 0/227           \\
	\texttt{leaf-llm}  & 0        & 0/898      & 9             & 3/227           \\
	\texttt{rethink}   & 425      & 331/898    & 91            & 60/227          \\
	\texttt{persist}   & 53       & 29/898     & 24            & 12/227          \\
	\texttt{let}       & 0        & 0/898      & 6             & 2/227           \\
	\texttt{map}       & 32       & 10/898     & 19            & 3/227           \\
	\texttt{get}       & 7        & 3/898      & 7             & 3/227           \\
	\texttt{apply}     & 0        & 0/898      & 1             & 1/227           \\
	\texttt{if}        & 6        & 4/898      & 8             & 8/227           \\
	\texttt{agents/*}  & 0        & 0/898      & 1             & 1/227           \\
	\texttt{globals/*} & 0        & 0/898      & 0             & 0/227           \\
	\texttt{io/*}      & 25,995   & 898/898    & 3,624         & 222/227         \\
	\texttt{strings/*} & 308      & 93/898     & 151           & 44/227          \\
	\texttt{web/*}     & 0        & 0/898      & 0             & 0/227           \\
\end{longtable}
\endgroup

In Terminal-Bench, the only \texttt{agents/} hit was low-effort \texttt{path-tracing-reverse}, where the model attempted \texttt{(agents/!ask :explore ...)}. This did not spawn a helper; it asked an unregistered handle and failed with \texttt{agents/!ask: Handle not registered}, after which the agent recovered.

The most concentrated auxiliary-model use was high-effort Terminal-Bench \texttt{play-zork}, which made six \texttt{leaf-llm} calls. The agent requested Zork I walkthroughs, treasure lists, and final-ending text from a plain-text model, then tried to validate the generated walkthrough through \texttt{dfrotz}. One \texttt{leaf-llm} call timed out, triggering recovery; the harness still scored the task unresolved, with both end-of-game and maximum-score checks failed.

To test whether non-use of multi-agent orchestration was merely a prompting artifact, a SWE-bench Lite intervention was run on the 32-item model-comparison subset. The \Spell{} agent used GPT-5.4 at medium reasoning effort, the ordinary \texttt{io-tc} coding setup, and the initial expression:

\begin{Shaded}
	\begin{Highlighting}[]
		\NormalTok{\textquotesingle{}(!describe io reminders :coding reminders :orchestrator)}
	\end{Highlighting}
\end{Shaded}

Feature usage was counted over completed traces, again restricted to model-response forms. The intervention solved 20/32 items, produced 3 error rows, cost \texttt{\$9.48}, and had 29 completed traces. It did not elicit multi-agent or recursive self-call orchestration: the completed traces contained zero instances of \texttt{agents/spawn} or \texttt{agents/!spawn-ask}.

Control flow features like \texttt{map} were used sparsely, mostly in support of programmatic tool calling. For example, \texttt{heterogeneous-dates} used \texttt{let}, \texttt{map}, \texttt{get}, and \texttt{apply} to parse two CSV files and compute an average; several other traces used simple \texttt{if} or \texttt{let} forms around shell-command success and structured status outputs. These examples show that the model sometimes used \Spell{} as a dataflow language around tool results.

Tool batching was much more common than higher-order control flow. I counted \texttt{io/*} forms in each model-response suffix, again excluding inherited prefix text. On the GPT-5.4 \Spell{} trace sets used for the paper, most turns emitted more than one tool call.

\begin{longtable}[]{@{}
	>{\raggedright\arraybackslash}p{(\linewidth - 10\tabcolsep) * \real{0.2000}}
	>{\raggedright\arraybackslash}p{(\linewidth - 10\tabcolsep) * \real{0.1200}}
	>{\raggedleft\arraybackslash}p{(\linewidth - 10\tabcolsep) * \real{0.1600}}
	>{\raggedleft\arraybackslash}p{(\linewidth - 10\tabcolsep) * \real{0.1800}}
	>{\raggedleft\arraybackslash}p{(\linewidth - 10\tabcolsep) * \real{0.2200}}
	>{\raggedleft\arraybackslash}p{(\linewidth - 10\tabcolsep) * \real{0.1200}}@{}}
	\caption{Per-turn \texttt{io/*} batching in GPT-5.4 \Spell{} traces.}                                                                                                                                                                    \\
	\toprule\noalign{}
	\begin{minipage}[b]{\linewidth}\raggedright
		Benchmark
	\end{minipage} & \begin{minipage}[b]{\linewidth}\raggedright
		                 Effort
	                 \end{minipage} & \begin{minipage}[b]{\linewidth}\raggedleft
		                                  Traced turns
	                                  \end{minipage} & \begin{minipage}[b]{\linewidth}\raggedleft
		                                                   Mean \texttt{io/*} calls / turn
	                                                   \end{minipage} & \begin{minipage}[b]{\linewidth}\raggedleft
		                                                                    Turns with \(\geq 2\) \texttt{io/*} calls
	                                                                    \end{minipage} & \begin{minipage}[b]{\linewidth}\raggedleft
		                                                                                     Max \texttt{io/*} calls / turn
	                                                                                     \end{minipage}                                                                                                        \\
	\midrule\noalign{}
	\endhead
	\bottomrule\noalign{}
	\endlastfoot
	SWE-bench Lite                              & low                                         & 3,122                                      & 2.85                                       & 2,406/3,122 (77.1\%)                       & 13 \\
	SWE-bench Lite                              & medium                                      & 2,776                                      & 2.89                                       & 2,166/2,776 (78.0\%)                       & 11 \\
	SWE-bench Lite                              & high                                        & 2,712                                      & 3.34                                       & 2,176/2,712 (80.2\%)                       & 14 \\
	Terminal-Bench 1.1                          & low                                         & 605                                        & 2.12                                       & 343/605 (56.7\%)                           & 13 \\
	Terminal-Bench 1.1                          & medium                                      & 542                                        & 2.46                                       & 372/542 (68.6\%)                           & 13 \\
	Terminal-Bench 1.1                          & high                                        & 439                                        & 2.26                                       & 262/439 (59.7\%)                           & 11 \\
\end{longtable}

\subsection{LongBench and AppWorld benchmarks}\label{app:bench:cross}

The source values for Figure~\ref{fig:obs5} compare \Spell{} and Codex CLI across the four GPT-5.4 medium-effort benchmark settings emphasized in the main text. Accuracy is the fraction of tasks resolved under each benchmark's native scoring rule. Cost and token totals are summed over all evaluated tasks, including tasks that were not resolved. The paired \(p\) values use the exact two-sided sign test described above.

\begin{longtable}[]{@{}
	>{\raggedright\arraybackslash}p{(\linewidth - 10\tabcolsep) * \real{0.22}}
	>{\raggedright\arraybackslash}p{(\linewidth - 10\tabcolsep) * \real{0.13}}
	>{\raggedleft\arraybackslash}p{(\linewidth - 10\tabcolsep) * \real{0.20}}
	>{\raggedleft\arraybackslash}p{(\linewidth - 10\tabcolsep) * \real{0.13}}
	>{\raggedleft\arraybackslash}p{(\linewidth - 10\tabcolsep) * \real{0.10}}
	>{\raggedleft\arraybackslash}p{(\linewidth - 10\tabcolsep) * \real{0.12}}@{}}
	\caption{Figure~\ref{fig:obs5} source table. \texttt{Total tokens} is the sum of input plus visible and reasoning output tokens.} \\
	\toprule\noalign{}
	Benchmark          & System    & Resolved         & Cost              & \(p\) & Total tokens                                      \\
	\midrule\noalign{}
	\endhead
	\bottomrule\noalign{}
	\endlastfoot
	Terminal-Bench 1.1 & \Spell{}     & 36/80 (45.0\%)   & \texttt{\$25.72}  & 0.549 & 6.96M                                             \\
	Terminal-Bench 1.1 & Codex CLI & 39/80 (48.8\%)   & \texttt{\$46.96}  & 0.549 & 51.14M                                            \\
	SWE-bench Lite     & \Spell{}     & 171/300 (57.0\%) & \texttt{\$102.12} & 1.000 & 51.37M                                            \\
	SWE-bench Lite     & Codex CLI & 172/300 (57.3\%) & \texttt{\$161.82} & 1.000 & 155.96M                                           \\
	LongBench v2       & \Spell{}     & 122/200 (61.0\%) & \texttt{\$27.83}  & 0.041 & 10.36M                                            \\
	LongBench v2       & Codex CLI & 135/200 (67.5\%) & \texttt{\$25.58}  & 0.041 & 35.19M                                            \\
	AppWorld dev       & \Spell{}     & 24/57 (42.1\%)   & \texttt{\$32.99}  & 0.002 & 10.73M                                            \\
	AppWorld dev       & Codex CLI & 36/57 (63.2\%)   & \texttt{\$10.98}  & 0.002 & 16.43M                                            \\
\end{longtable}

\paragraph{AppWorld trace analysis.} Investigation into causes of the AppWorld gap revealed an unclear mix of behaviors and issues. Among the 13 Codex-only items, approximately 8--9 involved \Spell{}-side query formulation, data-source selection, action-selection, or final-submission mistakes. One item was a clear \Spell{}-side refusal to perform a simulated Venmo transfer (\texttt{37a8675\_1}). Three items appeared to be task-completion or evaluator mismatches after \Spell{} had performed the requested action, and one was an \texttt{unknown\_agent\_error}. In rows missed by both agents, the most common shared pattern was temporal anchoring to the real run date (May 1, 2026) rather than the 2022--2023 fixture dates used by the AppWorld state, producing false-zero or no-op answers.

\clearpage
\normalsize

\section*{NeurIPS Paper Checklist}

\begin{enumerate}

	\item {\bf Claims}
	\item[] Question: Do the main claims made in the abstract and introduction accurately reflect the paper's contributions and scope?
	\item[] Answer: \answerYes{}
	\item[] Justification: The abstract and Section~\ref{sec:intro} identify SPE, its formalization, the \Spell{} language, and an evaluation of existing models without SPE-specific training as the contributions. Section~\ref{sec:theory} and Appendix~\ref{app:theory} establish the theoretical results; Section~\ref{sec:eval} reports task performance and observed orchestration behavior, including the limits of the harness comparisons.

	\item {\bf Limitations}
	\item[] Question: Does the paper discuss the limitations of the work performed by the authors?
	\item[] Answer: \answerYes{}
	\item[] Justification: Section~\ref{sec:eval} reports limited use of advanced self-orchestration on ordinary benchmarks and performance differences across harnesses and tasks. Section~\ref{sec:discussion} explains that SPE changes who specifies orchestration policy, that current models mainly use simple policies, and that the benefits of SPE-specific training remain unknown.

	\item {\bf Theory assumptions and proofs}
	\item[] Question: For each theoretical result, does the paper provide the full set of assumptions and a complete (and correct) proof?
	\item[] Answer: \answerYes{}
	\item[] Justification: Section~\ref{sec:theory} defines agentic machines, embeddings, SPE states, and agentic evaluators. Appendix~\ref{app:theory} specifies the agentic evaluator and its boundary states, gives the proof of Theorem~\ref{thm:main}, and states and proves the realizability and universality results, including Corollary~\ref{cor:universal-seed}.

	\item {\bf Experimental result reproducibility}
	\item[] Question: Does the paper fully disclose all the information needed to reproduce the main experimental results?
	\item[] Answer: \answerYes{}
	\item[] Justification: The language and runtime are described in Section~\ref{sec:spell} and Appendix~\ref{app:language}, with the \texttt{v0.1.0} reference implementation available from the public repository. Appendices~\ref{app:bench:config} and~\ref{app:bench:protocols} specify agent configurations, prompts, model settings, task selection, resource limits, and scoring; Appendix~\ref{app:bench} supplies numerical result tables. Replication requires access to the commercial model services used in the experiments.

	\item {\bf Open access to data and code}
	\item[] Question: Does the paper provide open access to the data and code, with sufficient instructions to faithfully reproduce the main experimental results?
	\item[] Answer: \answerYes{}
	\item[] Justification: The \Spell{} implementation is public at \url{https://github.com/lukejoconnor/spell}, and the benchmark datasets are publicly available. The evaluation runners and instructions for reproducing the main experiments are available at \url{https://github.com/lukejoconnor/spell-benchmarking}; Appendix~\ref{app:bench} provides evaluation settings, scoring procedures, and numerical result tables.

	\item {\bf Experimental setting/details}
	\item[] Question: Does the paper specify all the training and test details necessary to understand the results?
	\item[] Answer: \answerYes{}
	\item[] Justification: Section~\ref{sec:eval} describes the experiments with existing models, which were not trained for SPE or \Spell{}. Appendices~\ref{app:bench:config}, \ref{app:bench:orchestration-games}, and~\ref{app:bench:protocols} give the harness and model settings, prompts, task sets, sampling rules, budgets, and scoring procedures.

	\item {\bf Experiment statistical significance}
	\item[] Question: Does the paper report error bars suitably and correctly defined?
	\item[] Answer: \answerYes{}
	\item[] Justification: Figures~\ref{fig:obs1} and~\ref{fig:obs5} show 95\% Wilson score intervals for marginal task-level proportions; these describe uncertainty across evaluated tasks rather than variation across repeated runs of each task. Figure~\ref{fig:obs5} also reports paired exact two-sided binomial tests on discordant task outcomes, with the null probability of 0.5 specified in Appendix~\ref{pairwise-statistics}; denominators and exclusions are described in Appendix~\ref{app:bench:protocols}.

	\item {\bf Experiments compute resources}
	\item[] Question: For each experiment, does the paper provide sufficient information on the computer resources needed to reproduce the experiments?
	\item[] Answer: \answerYes{}
	\item[] Justification: Table~\ref{tab:compute-resources} in Appendix~\ref{app:bench:protocols} reports GCP VM types and counts, wall-clock durations, concurrency, resource limits, and model API spend for the retained benchmark analyses. The table and accompanying text distinguish model API spend from VM cost and describe the exclusion of pilot and aborted runs.

	\item {\bf Code of ethics}
	\item[] Question: Does the research conducted in the paper conform, in every respect, with the NeurIPS Code of Ethics?
	\item[] Answer: \answerYes{}
	\item[] Justification: The experiments evaluate existing models on public benchmarks and synthetic orchestration games, without recruiting human participants or training new models. Appendix~\ref{app:language:capability-boundaries} discusses capability restrictions and misuse or accident risks associated with model-written orchestration.

	\item {\bf Broader impacts}
	\item[] Question: Does the paper discuss both potential positive societal impacts and negative societal impacts of the work performed?
	\item[] Answer: \answerYes{}
	\item[] Justification: Appendix~\ref{app:language:capability-boundaries} describes the value of flexible context management for long-running work and the associated risks of misuse, accidents, and reduced user control over the context informing later actions. It also describes configurable capability boundaries.

	\item {\bf Safeguards}
	\item[] Question: Does the paper describe safeguards that have been put in place for responsible release of data or models that have a high risk for misuse?
	\item[] Answer: \answerNA{}
	\item[] Justification: The released asset is language-runtime code; no new trained model or dataset is released. Configurable effect namespaces and their role in limiting an agent\textquotesingle{}s capabilities are described in Appendix~\ref{app:language:capability-boundaries}.

	\item {\bf Licenses for existing assets}
	\item[] Question: Are the creators or original owners of assets used in the paper, properly credited and are the license and terms of use explicitly mentioned and respected?
	\item[] Answer: \answerYes{}
	\item[] Justification: The benchmark sources are SWE-bench Lite \citep{jimenez2024swebench}, Terminal-Bench \citep{merrill2026terminalbench}, LongBench~v2 \citep{bai2025longbenchv2}, and AppWorld \citep{trivedi2024appworld}; versions and task sets are specified in Appendix~\ref{app:bench:protocols}. SWE-bench code is MIT-licensed, Terminal-Bench is Apache-2.0-licensed, LongBench~v2 data are Apache-2.0-licensed and its evaluation code is MIT-licensed, and AppWorld is Apache-2.0-licensed with an additional encrypted-redistribution condition for its protected components. Codex CLI is Apache-2.0-licensed and mini-swe-agent is MIT-licensed; Claude Code and commercial model APIs were used under their provider terms.

	\item {\bf New assets}
	\item[] Question: Are new assets introduced in the paper well documented?
	\item[] Answer: \answerYes{}
	\item[] Justification: The \Spell{} implementation is released under the MIT license at \url{https://github.com/lukejoconnor/spell}, with installation and usage documentation alongside the code. Appendix~\ref{app:language} documents the reference language, runtime, prompts, and capability boundaries.

	\item {\bf Crowdsourcing and research with human subjects}
	\item[] Question: Does the paper provide participant instructions, applicable screenshots, and compensation details for crowdsourcing or human-subject experiments?
	\item[] Answer: \answerNA{}
	\item[] Justification: The study does not involve crowdsourcing or research with human participants.

	\newpage

	\item {\bf IRB approvals (or equivalent) for research with human subjects}
	\item[] Question: Does the paper explain participant risks, disclosure of those risks, and approval by an IRB or equivalent institutional review process?
	\item[] Answer: \answerNA{}
	\item[] Justification: No human participants were recruited or studied, so human-subject review is not applicable.

	\item {\bf LLM usage}
	\item[] Question: Does the paper describe the usage of LLMs if it is an important, original, or non-standard component of the research methodology?
	\item[] Answer: \answerYes{}
	\item[] Justification: LLM calls are the central mechanism of SPE and \Spell{} (Sections~\ref{sec:theory} and~\ref{sec:spell}); model comparisons and self-orchestration experiments are described in Section~\ref{sec:eval}. Appendices~\ref{app:bench:config} and~\ref{app:bench:orchestration-games} specify model settings and prompts, including LLM-based scoring of the asynchronous orchestration games.

\end{enumerate}

\end{document}